\documentclass[11pt,twoside]{article}
\usepackage{fullpage}

\usepackage{epsf}
\usepackage{fancyhdr}
\usepackage{graphics}
\usepackage{graphicx}
\usepackage{psfrag}
\usepackage{microtype}
\usepackage{subfigure}
\usepackage{algorithmic}
\usepackage{color,xcolor}
\usepackage{caption}
\usepackage{subcaption}
\usepackage{subfigure}
\usepackage{verbatim}

\usepackage[linesnumbered,ruled]{algorithm2e}
\DontPrintSemicolon

\usepackage{color}

\usepackage{amsthm}
\usepackage{amsfonts}
\usepackage{amsmath}
\usepackage{amssymb,bbm}

\usepackage{mathrsfs}

\usepackage{url}
\usepackage[colorlinks,linkcolor=magenta,citecolor=blue, pagebackref=true,backref=true]{hyperref}
\renewcommand*{\backrefalt}[4]{%
    \ifcase #1 \footnotesize{(Not cited.)}%
    \or        \footnotesize{(Cited on page~#2.)}%
    \else      \footnotesize{(Cited on pages~#2.)}%
    \fi}

\usepackage{nicefrac}
\usepackage{comment}

\usepackage{chngpage}

 \usepackage{tabularx}%
\usepackage{enumitem}
\usepackage{booktabs}
\usepackage{caption}

\usepackage{bm,bbm}
\usepackage{mathtools}
\usepackage{cleveref}

\usepackage{eucal}

\usepackage{scalerel}

\usepackage{final_macros}

\newcommand{\state}{x}

\newcommand{\MyState}{X}

\newtheorem{example}{Example}

\def\l{\left}
\def\r{\right}
\def\la{\left\langle}
\def\ra{\right\rangle}
\def\ll{\lVert}
\def\rl{\rVert}
\def\lv{\left\lvert}
\def\rv{\right\rvert}
\def\pt{\partial}
\def\tcr{\textcolor{red}}
\newtheorem{assumption}{Assumption}
\newcommand{\numobs}{\ensuremath{N}}

\newcommand{\usedim}{\ensuremath{d}}

\newcommand{\Reward}{R}
\newcommand{\reward}{r}

\newcommand{\lspace}{\LinSpace}

\newcommand{\projecttolin}{\Pi_{\lspace}}

\newcommand{\projectto}[1]{\Pi_{#1}}

\newcommand{\ltwospace}{\mathbb{L}^2}
\newcommand{\statespace}{\mathbb{X}}

\newcommand{\totaltime}{T}

\newcommand{\transition}{\mathcal{P}}

\newcommand{\stationary}{\xi}

\newcommand{\sobospace}{\mathbb{H}^1}
\newcommand{\soboavgnorm}[2]{\vecnorm{#1}{\mathbb{H}^1, #2} }
\newcommand{\soboavginprod}[3]{\inprod{#1}{#2}_{\mathbb{H}^1, #3}}

\newcommand{\fakerefassumelip}[1]{\hyperref[assume:smooth-high-order]{{\color{magenta} {\upshape\textbf (}{\upshape{\textbf{Lip}}#1}{\upshape\textbf )}}} }

\newcommand{\discount}{\beta}

\newcommand{\valuefunc}{f}

\newcommand{\Bel}{\mathrm{Bel}}

\newcommand{\gen}{\mathrm{gen}}

\newcommand{\valuestar}{\valuefunc^*}

\newcommand{\ValueFunc}{\valuefunc}

\newcommand{\ValFun}{\valuefunc}
\newcommand{\ValFunc}{\valuefunc}

\DeclareFontFamily{U}{mathx}{}
\DeclareFontShape{U}{mathx}{m}{n}{<-> mathx10}{}
\DeclareSymbolFont{mathx}{U}{mathx}{m}{n}

\DeclareMathAccent{\widecheck}{0}{mathx}{"71}
\newcommand{\BellOp}{\ensuremath{\mathcal{T}}}

\newcommand{\valuebar}{\widebar{\valuefunc}}

\newcommand{\valuehat}{\widehat{\valuefunc}}

\newcommand{\lammin}{\lambda_{\min}}
\newcommand{\lammax}{\lambda_{\max}}

\usepackage{mathbbol}

\newcommand{\mbasis}{m}

\newcommand{\constScaryReg}{\constScary_{\mathrm{reg}}}

\newcommand{\statnorm}[1]{\|#1\|_{\stationary}}
\newcommand{\statinprod}[2]{\inprod{#1}{#2}_{\stationary}}

\newcommand{\ValPol}{\ensuremath{\ValFun^\star}}
\newcommand{\ValTrue}{\ValPol}

\newcommand{\ValHat}{\ensuremath{\widehat{\ValFun}}}

\newcommand{\StateSpace}{\statespace}

\newcommand{\torus}{\mathbb{T}}

\usepackage[framemethod=TikZ]{mdframed}

\newcommand{\myframe}[1]{
\begin{mdframed}[backgroundcolor=black!1, roundcorner=5pt]
  #1
\end{mdframed}
}

\newenvironment{narrowpara}
  {\par\addvspace{\smallskipamount}\narrower\noindent\ignorespaces}
  {\par\addvspace{\smallskipamount}}

\newcommand{\coordinate}{e}

\newcommand{\myassumption}[4]{
  \setlist[enumerate,1]{leftmargin=#4}
\myframe{
	\begin{enumerate}[label={ \bf{{{(#1)}}}}]
		\item \label{#2} {#3}
	\end{enumerate}
        }
}

\renewcommand{\hat}{\widehat}
\renewcommand{\tilde}{\widetilde}

\newcommand{\LinSpace}{\mathbb{K}}

\newcommand{\IdMat}{\mathcal{I}}

\newcommand{\occumsr}{\mu}

\newcommand{\drift}{b}
\newcommand{\covMat}{\Lambda}

\newcommand{\BM}{B}

\newcommand{\goodendex}{\ensuremath{\clubsuit}}

\newcommand{\ValInter}[1]{\ValueFunc_{\mathrm{Bel}}^{(#1)}}
\newcommand{\valhatinter}[1]{\ValueFunc_{\mathrm{gen}}^{(#1)}}
\newcommand{\valtdinter}[1]{\valuebar^{(#1)}_{\mathrm{gen}}}
\newcommand{\geninter}[1]{\generator^{(#1)}}
\newcommand{\semigroup}{\mathcal{P}}
\newcommand{\generator}{\mathcal{A}}
\newcommand{\constScary}{\widebar{C}}
\newcommand{\constdiff}{\hat{C}}
\newcommand{\constgen}{\tilde{C}}
\newcommand{\constbd}{\hat{D}}
\newcommand{\numerOrder}{n}
\newcommand{\sobonorm}[1]{\vecnorm{#1}{\mathbb{H}^1}}

\newcommand{\coef}[1]{a^{(#1)}}
\newcommand{\coefc}[1]{c^{(#1)}}
\newcommand{\coefd}[1]{d^{(#1)}}
\newcommand{\dt}{\stepsize}

\begin{document}

\begin{center}
{\bf{\LARGE{On Bellman equations for continuous-time policy evaluation I: discretization and approximation}}}

\vspace*{.2in}
{\large{
 \begin{tabular}{cc}
  Wenlong Mou$^{ \diamond, \star}$ & Yuhua Zhu$^{ \dagger, \star}$ 
 \end{tabular}

}

\vspace*{.2in}

 \begin{tabular}{c}
 Department of Statistical Sciences, University of Toronto$^{\diamond}$
 \end{tabular}
 
  \begin{tabular}{c}
  Department of Mathematics, University of California, San Diego$^{ \dagger}$
  \end{tabular}

}

\begin{abstract}
  We study the problem of computing the value function from a discretely-observed trajectory of a continuous-time diffusion process. We develop a new class of algorithms based on easily implementable numerical schemes that are compatible with discrete-time reinforcement learning (RL) with function approximation. We establish high-order numerical accuracy as well as the approximation error guarantees for the proposed approach. In contrast to discrete-time RL problems where the approximation factor depends on the effective horizon, we obtain a bounded approximation factor using the underlying elliptic structures, even if the effective horizon diverges to infinity. 
  \let\thefootnote\relax\footnote{$^{\star}$ WM and YZ contributed equally to this work.}
\end{abstract}
\end{center}

\section{Introduction}
Over recent years, data-driven and learning-based technologies are playing increasingly important roles in the estimation and control of dynamical systems. The workhorse of modern control and dynamic programming is reinforcement learning (RL), powered by flexible function approximation techniques. Model-free RL approaches allow us to learn the target function of interest directly, while circumventing the high computational and statistical costs for estimating the underlying dynamics. When combined with recent developments of deep neural networks, model-free RL has achieved great success in various applications, ranging from the game of go~\cite{silver2017mastering} to real-world robotics~\cite{ibarz2021train}.

Most of existing theoretical studies in RL focus on discrete-time Markov decision processes (MDPs), with a large class of algorithms have been developed and analyzed; see the monographs~\cite{bertsekas1996neuro,szepesvari2022algorithms}. However, since most of real-world systems run continuously in time, in order to apply existing algorithms, issues with time discretization must be taken into account. In this paper, we study the continuous-time policy evaluation with discrete-time data. 

Concretely, let $\drift: \StateSpace \rightarrow \real^\usedim$ be a vector field and $\covMat: \StateSpace \rightarrow \mathbb{S}_+^{\usedim \times \usedim}$ be a matrix-valued function. We consider a stochastic differential equation (SDE)
\begin{align}
  d \MyState_t = \drift (\MyState_t) dt + \covMat (\MyState_t)^{1/2} d \BM_t,\label{eq:cts-time-process}
\end{align}
where $(\BM_t \in \R^d: t \geq 0)$ is a standard Brownian motion in the Euclidean space $\StateSpace$.
Under certain regularity conditions, Eq~\eqref{eq:cts-time-process} admits a strong solution $(\MyState_t)_{t \in [0, + \infty)}$, which forms a (time-homogeneous) Markovian diffusion process. For a bounded reward function $\reward : \StateSpace \rightarrow \real$ and a known discount rate $\discount > 0$, we can define the \emph{value function}
\begin{align}
  \valuestar (x) \mydefn \int_0^{+ \infty}  e^{- \discount t} \Exs \big[ \reward (\MyState_t) \mid \MyState_0 = x \big] dt, \quad \mbox{for $x \in \StateSpace$}.\label{eq:defn-value-func}
\end{align}
Our goal is to compute the value function using (discretely-)observed trajectory of the continuous-time diffusion process~\eqref{eq:cts-time-process}.

At a superficial glance, such a task does not seem particularly challenging -- from the computational perspective, under mild regularity assumptions, the value function $\valuestar$ is known to satisfy a second-order linear elliptic equation.
\begin{align}
  \discount \ValFunc - \Big( \inprod{\nabla \ValFunc}{\drift} + \frac{1}{2} \mathrm{Tr} \big( \covMat \cdot \nabla^2 \ValFunc \big) \Big) = \reward.\label{eq:elliptic-eq-for-valfunc}
\end{align}
Abundant numerical methods exist for solving such equations (see e.g.~\cite{larsson2008partial}). From a learning perspective, note that an equi-spaced observed trajectory $(\MyState_0, \MyState_\stepsize, \MyState_{2\stepsize}, \cdots)$ of the process~\eqref{eq:cts-time-process} forms a time-homogeneous discrete-time Markov chain. To solve the policy evaluation problem for such a process, a straightforward approach is via the Bellman equation
\begin{align}
  \ValFunc (\state) = \stepsize \reward (\state) + e^{- \discount \stepsize} \Exs \big[ \ValFunc (\MyState_\stepsize) \mid \state \big].\label{eq:naive-discretized-bellman}
\end{align}
In RL literature, several algorithms have been developed and analyzed for solving such equations (see e.g.~\cite{szepesvari2022algorithms}). The key features of the two approaches are however in tension, creating new challenges -- direct methods for solving~\eqref{eq:elliptic-eq-for-valfunc} require the functions $(\drift, \covMat)$ to be either known or explicitly estimated through data. This is not compatible with the flexible model-free approaches for value learning in RL, as we are forced to model the underlying dynamics, instead of just the function of interest. On the other hand, the discrete-time Bellman equation~\eqref{eq:naive-discretized-bellman} could yield an inaccurate approximation to the true value function $\ValTrue$, and since the discount factor $e^{- \discount \stepsize}$ is close to unity, the resulting RL problem is ill-conditioned. This motivates a central question:
\begin{quote}
  Can we achieve sharp discretization error guarantees while retaining the flexibility of model-free RL methods?
\end{quote}
In this paper, we give this question an affirmative answer. Our main contribution is a new class of numerical methods for policy evaluation that are compatible with RL using function approximation.  The new method provides a model-free method that approximates the continuous-time value function more accurately with the same amount of discrete-time data with the following theoretical guarantees.
\begin{itemize}
  \item Under suitable smoothness assumptions, we construct a class of high-order Bellman operators, the fixed-point of which enjoys high-order error guarantees compared to the true value function $\ValTrue$. Notably, the high-order Bellman operators can be directly estimated using trajectory data, without constructing any estimators for the coefficients.
  \item By taking an orthonormal projection of the high-order Bellman operator to a given linear subspace, we can bound the approximation error for the projected discretized fixed point using the best approximation error within the subspace. While an approximation error blow-up is known to be unavoidable for worst-case MDPs (see e.g.~\cite{mou2023optimal}), the elliptic structure in~\eqref{eq:cts-time-process} allows us to establish uniformly bounded approximation factor upper bounds.
  \item Based on the population-level equations, we develop a class of data-driven algorithms that use trajectory data to compute the value function. We test their performances through numerical simulation studies.
\end{itemize}
A natural question is regarding the statistical error guarantees for the data-driven approach, as well as the parameter selection thereof. We defer a comprehensive study on this question to our companion paper.

The rest of this paper is organized as follows: we first introduce notations and discuss related work. Section~\ref{sec:problem-setup} describes the basic problem setups. In Sections~\ref{sec:time-discretization} and~\ref{sec:projected-fixed-pt}, we present the two main steps of the algorithm -- time discretization and function approximation -- as well as their theoretical guarantees. We present some numerical simulation results in Section~\ref{sec:simulation}. The proof of the main results are collected in Section~\ref{sec:proofs}. Finally, we conclude the paper in Section~\ref{sec:discussion} with a discussion of the results and future work.

\paragraph{Notation.} Here we summarize some notation used throughout the paper. For a
positive integer $m$, we define the set $[m] \defn \{1,2, \cdots,
m\}$. We use $\vecnorm{\cdot}{2}$ to denote the standard Euclidean norm in finite-dimensional spaces. For a probability distribution $\mu$, we use $\inprod{\cdot}{\cdot}_{\mu}$ as a shorthand notation for the inner product structure $(f, g) \mapsto \int f g d \mu$, and use $\vecnorm{\cdot}{\mu}$ to denote the induced norm. Furthermore, for a vector-valued or matrix-valued function $f$, we use $\vecnorm{f}{\mu}$ to denote the integrals $\int \vecnorm{f}{2}^2 d \mu$ and $\int \matsnorm{f}{F}^2 d \mu$. Given an Euclidean space $\real^\usedim$, we call a vector in $\mathbb{N}^\usedim$ a multi-index. For a multi-index $\alpha = (\alpha_1, \alpha_2, \cdots, \alpha_\usedim)$ and a function $f$, we use the shorthand notation $\partial^\alpha f (x)$ to denote $\frac{\partial^{\vecnorm{\alpha}{1}}}{\partial x_1^{\alpha_1} \partial x_2^{\alpha_2} \cdots \partial x_\usedim^{\alpha_\usedim}} f (x)$.
For a random object $X$, we use $\law(X)$ to denote its probability
law. Given a vector $\mu \in \real^d$ and a positive semi-definite
matrix $\Sigma \in \real^{d \times d}$, we use $\mathcal{N} (\mu,
\Sigma)$ to denote the Gaussian distribution with mean $\mu$ and
covariance $\Sigma$. We use $\mathcal{U} (\Omega)$ to
denote the uniform distribution over a set $\Omega$.
We use $\{e_j\}_{j=1}^\usedim$ to denote the standard basis vectors in
the Euclidean space $\real^\usedim$, i.e., $e_i$ is a vector with a
$1$ in the $i$-th coordinate and zeros elsewhere. For symmetric matrices $A, B \in \real^{d \times
  d}$, we use $A \preceq B$ to denote the fact $B - A$ is a positive
semi-definite matrix, and denote by $A \prec B$ when $B - A$ is
positive definite.

\subsection{Related work}
The paper involves a synergy of computational differential equations and data-driven reinforcement learning methods. It is useful to briefly discuss existing works in both domains, and their intersection, as well as connections to the current paper.

\paragraph{Policy evaluation and (projected) fixed-points:} Estimating the value function of a given policy is a canonical problem in reinforcement learning. The temporal difference (TD) methods~\cite{sutton1988learning,boyan2002technical} solve Bellman fixed-point equations and their orthogonal projections using empirical data. Since proposed, several research efforts have been devoted into their theoretical analysis, including approximation properties~\cite{tsitsiklis1997analysis,tsitsiklis1999optimal,mou2023optimal}, asymptotic convergence~\cite{bradtke1996linear,tsitsiklis1997analysis}, and non-asymptotic optimal statistical guarantees~\cite{bhandari2018finite,khamaru2021temporal,mou2023optimal,mou2024optimal}. Though applicable to the discretely-observed continuous-time trajectory settings studied in this paper, they do not take into account discretization error and structures of the underlying dynamics, which lead to sub-optimal guarantees.

Closely related to this work is a line of research on approximation factors of the projected fixed-point methods. Tsitsiklis and Van Roy~\cite{tsitsiklis1997analysis} showed that the exact projected fixed-point has an error bounded by the projection error in the approximating class multiplied by the effective horizon of the RL problem. Such a guarantee is shown to be worst-case optimal~\cite{mou2023optimal}, and improved instance-dependent guarantees were derived~\cite{yu2010error,mou2023optimal}. When applied to the problems we consider, even the instance-dependent bounds lead to dependence on the effective horizon (and therefore the stepsize of the time discretization; see Section~\ref{subsec:general-worst-case-approx} to follow). On the other hand, we obtain an improved bound by utilizing elliptic structures that are not present in general discrete-time Markov chains.

\paragraph{Computational methods in differential equations and control:}
Back in 1915, Galerkin~\cite{galerkin1915series} proposed solving a partial differential equation (PDE) via orthogonal projection onto a finite-dimensional linear subspace. This marked the beginning of a rich literature on numerical methods using projected equations, and lays the foundations of modern finite-element methods. When applied to second-order elliptic equations, error bounds were established by C\'{e}a~\cite{cea1964approximation}, and Babu\v{s}ka~\cite{babuvvska1978error,babuvvska1978error}, among other works. We refer the interested readers to the monograph~\cite{brenner2007mathematical} for a comprehensive survey.

The concepts of value function and differential equations for continuous-time control problems date back to Bellman~\cite{bellman1954dynamic} and Kalman~\cite{kalman1963theory}. See the book~\cite{fleming2006controlled} for a comprehensive survey. The study of numerical methods in controlled diffusion processes dates back to Kushner~\cite{kushner1990numerical}, with several recent developments~\cite{tsitsiklis1995efficient,beard1997galerkin,yin2012continuous}. A major difference between the RL problems and classical control problems is that the underlying system in RL is unknown, and needs to be learned from data.

\paragraph{Machine learning for control and differential equations:} Recent years have witnessed the increasing popularity of data-driven and machine learning methods applied to computational problems in science and engineering. Such methods include solving PDEs through neural networks \cite{raissi2019physics, cai2021physics}, and learning the solution mapping for PDEs \cite{kovachki2023neural, lu2021machine, lu2022sobolev}. Additionally, various methods have been developed for solving the continuous-time optimal control problem with known dynamics, i.e., the Hamilton-Jacobi-Bellman equation \cite{zhou2021actor, li2022neural, ruthotto2020machine}, and specifically for linear quadratic control problems \cite{bhan2023neural, hwang2022solving}.
Though our policy evaluation problem can be characterized as a special case of elliptic PDEs, there is a major difference: the coefficients in their PDEs are either known or directly observed with $\mathrm{i.i.d.}$ unbiased samples, and their methods rely on explicit estimation of the coefficients; by way of contrast, our method requires only access to trajectory data, and can be incorporated into any model-free RL approach. 

In a concurrent work by one of the authors~\cite{zhu2024phibe}, another method is proposed to address the mismatch between the discrete-time data in continuous-time policy evaluation problems. 
There are two major differences. Firstly, \cite{zhu2024phibe} uses the trajectory data to approximate the dynamics instead of the value function, which results in a PDE as the new Bellman equation. Secondly, \cite{zhu2024phibe} only applies to the continuous dynamics with the standard SDE \eqref{eq:cts-time-process}, while the proposed high-order algorithms in this paper do not explicitly depend on underlying dynamics. 


The continuous-time version of TD methods dates back to 1995 by Doya~\cite{doya1995temporal}, where asymptotic convergence is established for a least-square algorithm under deterministic dynamics. A recent series of research~\cite{jia2023q,jia2022policyeval,jia2022policygrad,wang2020reinforcement} studied the asymptotic limits of various RL algorithms in continuous-time dynamics, under the na\"{i}ve first-order discretization schemes. Closely related to our paper is the work~\cite{jia2022policyeval}, which showed inconsistency of a na\"{i}ve error minimization approach, and proposed a new class of martingale-based Monte Carlo approach as remedies. In contrast, our methods could employ existing TD algorithms in a plug-and-play fashion, by solving the projected fixed-points instead of minimizing losses. The paper~\cite{kobeissi2023temporal} studies the convergence rate of TD stochastic approximation methods to the projected fixed-point under an $\mathrm{i.i.d.}$ observation model. They did not address the discretization error and approximation errors, which are non-trivial in the SDE settings. We will provide a detailed comparison for their method with our methods in our companion paper.


\section{Problem setup}\label{sec:problem-setup}

Let $(\semigroup_t : t \geq 0)$ be the semi-group associated to the diffusion~\eqref{eq:cts-time-process}, i.e.,
\begin{align*}
  \semigroup_t \ValFunc (\state) \mydefn \Exs [\ValFunc (\MyState_t) | \MyState_0 = \state], \quad \mbox{for any } \state \in \StateSpace,
\end{align*}
for any continuous function $\ValFunc$ with bounded support.

We can further define the generator operator
\begin{align}
  \generator \ValFunc \mydefn \frac{d}{dt} \semigroup_t \ValFunc \mid_{t = 0} = \inprod{\nabla \ValFunc}{\drift} + \frac{1}{2} \mathrm{Tr} \big( \covMat \nabla^2 \ValFunc \big),
\end{align}
for $C^\infty$ function $f$ with bounded support. By density arguments, it is easy to extend the operator $\generator$ to Sobolev spaces.

Given a step size $\stepsize > 0$, we observe a trajectory $(\MyState_{k \stepsize})_{0 \leq k \leq T / \stepsize}$.

Consider the idealized discrete-time Bellman operator
\begin{align}
  \BellOp^* (\valuefunc) (\state) \mydefn \int_0^\stepsize e^{- \discount s} \Exs \big[ \reward (\MyState_s) \mid \MyState_0 = \state \big] ds + e^{- \discount \stepsize} \Exs \big[ \ValueFunc (\MyState_\stepsize) \mid \MyState_0 = \state  \big].
\end{align}
It is easy to see that $\valuestar = \BellOp^* (\valuestar)$. Furthermore, since the operator $\BellOp^*$ is an $e^{- \discount \stepsize}$-contraction under the $L^{\infty}$-norm, the value function $\valuestar$ is the unique fixed point of $\BellOp^*$.

Let us consider some illustrative examples of the diffusion process~\eqref{eq:cts-time-process}.
\begin{example}[Linear quadratic systems]\upshape
  A classical example in continuous-time control is linear diffusion processes with quadratic reward. We consider the underlying diffusion process~\eqref{eq:cts-time-process} given by an Ornstein--Uhlenbeck process
\begin{align*}
  d X_t = (A X_t + z) dt + \covMat dB_t.
\end{align*}
  for a fixed pair of matrices $A, \covMat \in \real^{\usedim\times \usedim}$ and a vector $z \in \real^\usedim$. The time marginals for such a process admit closed-form expression, and the stationary distribution exists if and only if the eigenvalues of $A$ have negative real parts. If the reward function is quadratic, it is known (see e.g.~\cite{kumar1986stochastic}, Chapter 7) that the value function $\ValTrue$ is quadratic in the variable $x$. Therefore, in this example, zero approximation error can be achieved when we use the basis functions
  \begin{align*}
    \big\{1, x_1, x_2, \cdots, x_\usedim, x_1^2, x_1x_2, \cdots, x_1 x_\usedim, x^2 \cdots, x_{\usedim - 1} x_{\usedim},  x_{\usedim}^2 \big\}
  \end{align*}
to represent the value function.
    \hfill \goodendex
\end{example}

\begin{example}[Langevin diffusion]\upshape\label{example:langevin-diffusion}
  Let $U: \real^\usedim \rightarrow \real$ be a potential function, we consider the Langevin diffusion process
  \begin{align*}
     d X_t = - \frac{1}{2} \nabla U (X_t) + dB_t.
   \end{align*}
  Assuming that $\inprod{\nabla U (x)}{x} \geq c_1 \vecnorm{x}{2} - c_2$ for some $c_1, c_2 > 0$, the process is known to have a unique stationary distribution $\stationary \propto e^{- U}$, and to converge geometrically~\cite{bakry2008simple}. The mixing time depends on the spectral gap of the diffusion generator $\generator$. When we solve the Bellman fixed-point equation using empirical data, a trajectory length longer than the mixing time is necessary~\footnote{We will present the corresponding upper and lower bounds in the companion paper}; yet, for the population-level arguments in the present paper, while some results rely on existence of stationary measures, the convergence is not necessary.
\end{example}

\section{Time discretization for the value function}\label{sec:time-discretization}
We first consider time discretization schemes applied to the value function $\valuestar$, which is defined through a continuous-time integration in Eq~\eqref{eq:defn-value-func}. The goal is to use finite number of observations in the trajectory to form an approximate representation for such a function \emph{without} using the coefficients $(\drift, \covMat)$, and therefore facilitating flexible RL methods via function approximation. To start with, we note that the function $\valuestar$ exactly satisfies the equations
\begin{subequations}\label{eq:idealized-eq-for-valtrue}
\begin{align}
  \ValTrue (\state) &= \int_0^t e^{- \discount s} \Exs \big[ \reward (\MyState_s) \mid \MyState_0 = \state \big] ds + e^{- \discount t} \Exs \big[ \ValTrue (\MyState_\stepsize) \mid \MyState_0 = \state  \big], \quad \mbox{for any $t > 0$},\label{eq:idealized-integral-eq-for-valtrue} \\
  \reward (\state) &= \generator \ValTrue (\state) + \discount \ValTrue (\state) = \frac{d}{dt} \Big( \Exs \big[ \ValTrue (\MyState_t) \mid \MyState_0 = \state \big]\Big) \mid_{t = 0} + \discount \ValTrue (\state).\label{eq:idealized-differential-eq-for-valtrue}
\end{align}
\end{subequations}
The integration in Eq~\eqref{eq:idealized-integral-eq-for-valtrue} and the differentiation in Eq~\eqref{eq:idealized-differential-eq-for-valtrue} could not be solved exactly. A na\"{i}ve solution is to replace them with their respective first-order numerical approximations. In particular, given a discretization step size $\stepsize > 0$, we define the first-order Bellman operator and the first-order diffusion generator
\begin{subequations}\label{eqs:naive-eq-for-valtrue}
\begin{align}
  \BellOp^{(1)} (\valuefunc) (\state) &\mydefn \stepsize \reward (\state) + e^{- \discount \stepsize} \Exs \big[ \ValueFunc (\MyState_\stepsize) \mid \MyState_0 = \state  \big], \quad\mbox{and} \label{eq:naive-integral-eq-for-valtrue} \\
  \generator^{(1)} (\ValFunc) (\state) &\mydefn \frac{1}{\stepsize} \Big\{ \Exs \big[ \ValueFunc (\MyState_\stepsize) \mid \MyState_0 = \state \big] - \ValFunc (\state) \Big\}.\label{eq:naive-differential-eq-for-valtrue}
\end{align}
\end{subequations}
Approximate value functions can be obtained by solving the operator equations $ \BellOp^{(1)} (\valuefunc) = \ValueFunc$ and $ \discount \ValFunc  - \generator^{(1)} (\ValFunc) = \reward$. Indeed, many existing RL-based approaches in literature~\cite{jia2022policyeval,jia2023q} can be seen as solving first-order approximate equations of Eqs~\eqref{eqs:naive-eq-for-valtrue} type.

Despite its simplicity, Eqs~\eqref{eqs:naive-eq-for-valtrue} may lead to large numerical errors. A natural idea, therefore, is to construct more sophisticated high-order approximations to the integration and differentiation in Eqs~\eqref{eq:idealized-eq-for-valtrue}. In the next two subsections, we describe high-order numerical schemes that are compatible with RL, as well as their approximation error guarantees.

\subsection{High-order approximation to the Bellman operator}\label{subsec:higher-order-discre}
By generalizing the idea in Eq~\eqref{eq:naive-integral-eq-for-valtrue}, we are ready to construct a class of higher-order numerical approximations to the Bellman equation~\eqref{eq:idealized-integral-eq-for-valtrue}. For any integer $\numerOrder \geq 2$, given a bounded function $\ValFunc$ and a state $\state \in \StateSpace$, we define the operator
\begin{subequations}\label{eq:defn-high-order-bellman-operator}
\begin{align}
  \BellOp^{(\numerOrder)} (\valuefunc) (\state) &\mydefn \sum_{i = 0}^{\numerOrder - 1} \int_0^{(\numerOrder - 1) \stepsize} e^{- \discount s} W_i (s) \Exs \big[ \reward (\MyState_{i \stepsize}) \mid \MyState_0 = \state \big] ds + e^{- \discount (\numerOrder - 1) \stepsize} \Exs \big[ \ValueFunc (\MyState_{(\numerOrder - 1) \stepsize}) \mid \MyState_0 = \state  \big],\\
  \mbox{where }& W_i (s) \mydefn \Big\{ \prod_{j \neq i} \big( \stepsize j - \stepsize i \big) \Big\}^{-1}  \prod_{j \neq i} \big( s - \stepsize i\big), \quad \mbox{for } i = 1,2, \cdots \numerOrder.
\end{align}
In words, we interpolates the reward functions at discrete points $(\MyState_0 = \state, \MyState_\stepsize, \MyState_{2 \stepsize}, \cdots, \MyState_{(\numerOrder - 1) \stepsize})$ via Lagrangian polynomial, and use them to replace the true reward function in Eq~\eqref{eq:idealized-integral-eq-for-valtrue}. In order to implement the numerical scheme in Eq~\eqref{eq:defn-high-order-bellman-operator} the integrals in Eq~\eqref{eq:defn-high-order-bellman-operator}, we need to compute integrals of the form $\int_0^{(\numerOrder - 1) \stepsize} e^{- \discount s} W_i (s) ds$, which admits close-form solutions, as $W_i$ is a uni-variate polynomial. 

A natural question is whether the operator $\BellOp^{(\numerOrder)}$ possess the desirable structures as in discrete-time RL problems, such as existence and uniqueness of the fixed points. This is easy to verify -- indeed, for any pair $\ValFun_1, \ValFun_2$ of bounded functions, we have
\begin{align*}
  \vecnorm{\BellOp^{(\numerOrder)} \ValFun_1 - \BellOp^{(\numerOrder)} \ValFun_2}{\infty} = e^{- \discount \stepsize (\numerOrder - 1)} \vecnorm{\semigroup_{(\numerOrder - 1) \stepsize} \ValFun_1 - \semigroup_{(\numerOrder - 1) \stepsize} \ValFun_2}{\infty} \leq e^{- \discount \stepsize (\numerOrder - 1)} \vecnorm{ \ValFun_1 -  \ValFun_2}{\infty}.
\end{align*}
By contraction mapping theorem, a unique fixed point exists, i.e., the operator defines an approximate value function $\ValInter{\numerOrder}$ that satisfies the following equation
\begin{align}
    \ValInter{\numerOrder} =& \BellOp^{(\numerOrder)} (\ValInter{\numerOrder}) (\state) .\label{eq:defn-high-order-bellman-eqn}
\end{align}
\end{subequations}

In order to analyze such a fixed-point, we impose the following class of Lipschitz assumption.
\myassumption{Lip$(\numerOrder)$}{assume:smooth-high-order}{
    There exists positive constants $\big\{\smoothness_i^{\drift}\big\}_{i = 0}^{2\numerOrder - 2}$, $\big\{\smoothness_i^{\covMat}\big\}_{i = 0}^{2\numerOrder - 2}$, and $\big\{\smoothness_i^{\reward}\big\}_{i = 0}^{2 \numerOrder}$, such that
    \begin{subequations}
    \begin{align}
      \abss{\partial^\alpha \drift_k (x)} \leq \smoothness_i^{\drift},& \quad \mbox{for any $k \in [d]$ and multi-index $\alpha ~ \mathrm{s.t.}~|\alpha| \leq i$},\\
      \abss{\partial^\alpha \covMat_{k, \ell} (x)} \leq \smoothness_i^{\covMat}, &\quad \mbox{for any $k, \ell \in [d]$ and multi-index $\alpha ~ \mathrm{s.t.}~|\alpha| \leq i$},\\
      \abss{\partial^\alpha \reward (x)} \leq \smoothness_i^{\reward}, &\quad \mbox{for any multi-index $\alpha ~ \mathrm{s.t.}~|\alpha| \leq i$}.
    \end{align}
    \end{subequations}
}{2cm}
Under Assumption~\ref{assume:smooth-high-order}, we can obtain a general class of discretization error guarantees, as stated in the following theorem.
\begin{theorem}\label{thm:discretization-high-order}
  If Assumption~\ref{assume:smooth-high-order} holds true for some integer $\numerOrder > 0$, we have
  \begin{align*}
    \vecnorm{\ValInter{\numerOrder} - \valuestar}{\infty} \leq \discount^{-1} \constScary_\numerOrder \stepsize^\numerOrder,
  \end{align*}
  for a constant $\constScary_\numerOrder$ depending on $\big\{\smoothness_i^{\drift}\big\}_{i = 0}^{2 \numerOrder - 2}$, $\big\{\smoothness_i^{\covMat}\big\}_{i = 0}^{2 \numerOrder - 2}$, $\big\{\smoothness_i^{\reward}\big\}_{i = 0}^{2 \numerOrder}$ and problem dimension $\usedim$.
\end{theorem}
\noindent See Section~\ref{subsec:proof-thm-high-order} for its proof.

A few remarks are in order. First, note that the fixed-point equation $\ValFunc = \BellOp^{(\numerOrder)} (\valuefunc) $ is equivalent to a discrete-time policy evaluation problem with Markovian transition kernel $\semigroup_{(\numerOrder - 1) \stepsize}$, discount factor $e^{- \stepsize (\numerOrder - 1)}$, and a reward function given by Lagrangian polynomial interpolations. This fact facilitates the use of various RL algorithms on a plug-and-play basis. In the next subsection, we will discuss concretely how temporal difference (TD) methods apply to this problem, and in particular, how the underlying diffusion process improves its theoretical guarantees.

Second, compared to the na\"{i}ve first-order scheme~\eqref{eq:naive-integral-eq-for-valtrue}, the high-order discretization technique in Eq~\eqref{eq:defn-high-order-bellman-operator} allows us to construct approximations to the value function $\ValTrue$ with a $\numerOrder$-th order numerical accuracy, where the exponent $\numerOrder$ depends on the smoothness assumptions satisfied by the coefficients $(\drift, \covMat, \reward)$. Indeed, they are needed to ensure the existence and boundedness of the $\numerOrder$-th order time derivative of the uni-variate function $t \mapsto \semigroup_t \reward (\state)$. From the proof, it is easy to see that the constant $\constScary_\numerOrder$ depends polynomially on $\big\{\smoothness_i^{\drift}\big\}_{i = 0}^{2 \numerOrder - 2}$, $\big\{\smoothness_i^{\covMat}\big\}_{i = 0}^{2 \numerOrder - 2}$, $\big\{\smoothness_i^{\reward}\big\}_{i = 0}^{2 \numerOrder}$ and problem dimension $\usedim$. We omit its specific form for simplicity.

It is useful to discuss some special cases of Theorem~\ref{thm:discretization-high-order}. For $\numerOrder = 2$, the discretized Bellman operator~\eqref{eq:defn-high-order-bellman-operator} uses a one-step linear interpolation between the current state and the state after the time period $\stepsize$. In particular, the discretized operator can be explicitly computed as
\begin{multline*}
  \BellOp^{(2)} (\valuefunc) (\state) = \frac{1 - e^{- \discount \stepsize}}{\discount} \reward (\state) \\
  + \frac{1 - (1 + \discount \stepsize) e^{- \discount \stepsize}}{\discount^2 \stepsize} \Big\{ \Exs \big[ \reward (\MyState_\stepsize) \mid \MyState_0 = \state \big] - \reward (\state) \Big\} + e^{- \discount \stepsize} \Exs \big[ \ValueFunc (\MyState_\stepsize) \mid \MyState_0 = \state  \big]
\end{multline*}
Under Assumption~\fakerefassumelip{(2)}, Theorem~\ref{thm:discretization-high-order} implies the error bound
  \begin{align*}
    \vecnorm{\ValInter{2} - \valuestar}{\infty} \leq \frac{\constScary_2}{\discount} \stepsize^2.
  \end{align*}
In~\Cref{subsubsec:details-second-order-bellman}, we give an explicit expression for the form of the constant $\constScary_2$. It can be seen that $\constScary_2$ depends on the smoothness parameters polynomially.

\subsection{High-order approximation to the diffusion generator}
Another natural approach towards the construction of high-order Bellman equation is by generalizing the first-order approximate generator in Eq~\eqref{eq:naive-differential-eq-for-valtrue}. For any bounded continuous function $f$ and state $x \in \StateSpace$, we define
\begin{subequations}\label{eq:defn-high-order-generator}
\begin{align}
\generator^{(\numerOrder)} (\valuefunc)(\state) \mydefn \frac{1}{\stepsize} \sum_{j=0}^\numerOrder\coef{\numerOrder}_j \Exs \big[ \valuefunc(\MyState_{j\stepsize}) \mid \MyState_0 = \state \big],
\end{align}
where the coefficients are given by $\coef{\numerOrder} = \big[ \coef{\numerOrder}_0,\cdots, \coef{\numerOrder}_i \big]^\top = \l(A^{(\numerOrder)}\r)^{-1}b^{(\numerOrder)}$, with
\begin{align}\label{def of A b}
A^{(\numerOrder)}_{kj} = j^k, \quad 0\leq k,j \leq i; \quad b^{(\numerOrder)}_k = 
\begin{cases}
0, & k \neq 1, ~0\leq k \leq \numerOrder \\
1, & k = 1
\end{cases}
\end{align}
The $\numerOrder$-th order approximate value function $\valhatinter{\numerOrder}$ can be further defined as the solution to the integral equation 
\begin{align}
    \discount \ValFunc - \generator^{(\numerOrder)} \ValFunc = \reward.\label{eq:defn-discretized-gen-eq-high-order}
\end{align}
Intuitively, the operator $\generator^{(\numerOrder)}$ approximates the diffusion generator $\generator$ by a $\numerOrder$-th order numerical scheme, following a polynomial-based approach similar to the backward differentiation formula~\cite{curtiss1952integration}.
\end{subequations}

In the following proposition, we show that the distance between $\valhatinter{\numerOrder}(\state)$ for with the true value function $\valuestar$ is indeed $\numerOrder$-th order, for $\numerOrder = 1, 2$.

\begin{proposition}\label{thm:diffusion-second-order}
If Assumption~\ref{assume:smooth-high-order} holds true for integer $\numerOrder = 1,2$, we have
\[
\vecnorm{ \valhatinter{\numerOrder} - \valuestar }{\infty} \leq \constdiff_\numerOrder\stepsize^\numerOrder
\]
holds true for $\forall \stepsize > 0$ when $\numerOrder=1$, and $\forall \stepsize \in (0, \tfrac{1}{3 \discount})$ when $\numerOrder=2$. The constant $\constdiff_\numerOrder$ depends on $\big\{\smoothness_i^{\drift}\big\}_{i = 0}^{2 \numerOrder - 2}$, $\big\{\smoothness_i^{\covMat}\big\}_{i = 0}^{2 \numerOrder - 2}$, $\big\{\smoothness_i^{\reward}\big\}_{i = 0}^{2 \numerOrder}$, discount factor $\discount$ and problem dimension $\usedim$.
\end{proposition}
\noindent See Section \ref{proof:diffusion-second-order} for its proof.

A few remarks are in order. First, unlike~\Cref{thm:discretization-high-order}, in~\Cref{thm:diffusion-second-order}, we only establish approximation error bounds for $\numerOrder \in \{1,2\}$. Indeed, the backward difference formulae for ODEs are not zero-stable for $\numerOrder \geq 7$. Since the proof of~\Cref{thm:diffusion-second-order} relies on the stability of a similar recursive structure, we suspect that such type of guarantees cannot hold true for $\numerOrder \geq 7$. As for the case of $3 \leq \numerOrder \leq 6$, an analogous proof could be potentially carried through, but we omit them for simplicity, and only present the simple case for illustration purposes. Nevertheless, as we show in~\Cref{subsec:approx-gen}, in the case of $\numerOrder \geq 7$, even if~\Cref{eq:defn-discretized-gen-eq-high-order} does not admit a desirable solution, its projected version still satisfies $\numerOrder$-th order guarantees, making it suitable for practical algorithms.

When $\numerOrder = 2$, the second-order approximation $\valhatinter{2}$ satisfies,
\[
\frac{1}{\stepsize}\Exs\l[-\frac12\l(\valhatinter{2}(\MyState_{2\stepsize}) - \valhatinter{2}(\MyState_0)\r) + 2\l(\valhatinter{2}(\MyState_{\stepsize}) - \valhatinter{2}(\MyState_0)\r)|\MyState_0 = \state\r] = -\reward(\state) + \discount \valhatinter{2}(\state). 
\]
which can be equivalently written as
\[
\valhatinter{2}(\state) = \frac{\stepsize}{\discount\stepsize+3/2} \reward(\state) + \Exs\l[-\frac{1/2}{\discount\stepsize+3/2} \valhatinter{2}(\MyState_{2\stepsize}) + \frac{2}{\discount\stepsize+3/2}\valhatinter{2}(\MyState_\stepsize)|\MyState_0 = \state\r].
\]
Compare to the second-order Bellman operator, the diffusion generator applies high-order approximation schemes to the look-ahead value function instead of the observed rewards. It is an interesting direction of future research to study and compare the effect of the two different approaches on the approximation and statistical errors.

\section{Practical algorithms via function approximation}\label{sec:projected-fixed-pt}
In~\Cref{sec:time-discretization}, we have converted the continuous-time policy evaluation problem into a discrete-time one. Yet the discrete-time operators are still infinite-dimensional kernel integral operators. In reinforcement learning, a practical approach is via function approximation, i.e., by projecting the Bellman fixed-point equations onto low-complexity structures, and solving the projected fixed-points. In this section, we discuss the impact of function approximation on the continuous-time policy evaluation problems. We start by presenting approximation guarantees derived directly from RL literature, which could lead to potential error blow-ups and worse numerical order. Then, by exploiting elliptic structures in the diffusion process, we obtain improved guarantees for both the high-order Bellman operator and the high-order generator. We also discuss practical estimators for the population-level projected fixed-points. Finally, while most results require the existence of the stationary distribution of the diffusion process~\eqref{eq:cts-time-process}, in Section~\ref{subsec:occumsr-results}, we also extend them to a discounted occupancy measure.

\subsection{General worst-case approximation guarantees}\label{subsec:general-worst-case-approx}
Assume that the diffusion process~\eqref{eq:cts-time-process} has a stationary distribution $\stationary$.
Consider a Hilbert space $\mathbb{H} \mydefn \ltwospace (\stationary)$ where we use the stationary distribution $\stationary$ as the weighting function. Given an integer $\mbasis > 0$, let $(\psi_j)_{1 \leq j \leq \mbasis}$ be a collection of functions in $\ltwospace (\stationary)$. We use a finite-dimensional linear subspace $\LinSpace = \mathrm{span} (\psi_1, \psi_2, \cdots, \psi_\mbasis)$ to approximate the value function of interests. In doing so, we first define the orthonormal projection operator
\begin{align*}
  \projecttolin (\ValFun) \mydefn \arg\min_{\ValFun' \in \LinSpace} \vecnorm{\ValFun - \ValFunc'}{\stationary}.
\end{align*}
By applying the projected fixed-point approach~\cite{tsitsiklis1997analysis,mou2023optimal} to the discrete-time RL problem defined by Eq~\eqref{eq:defn-high-order-bellman-operator}, we seek to solve the fixed point to the composition of the projection operator and the Bellman operator, i.e., to find a function $\valuebar_{\Bel}^{(\numerOrder)} $ that satisfies
\begin{align}
  \valuebar_{\Bel}^{(\numerOrder)} = \projecttolin \circ \BellOp^{(\numerOrder)}  \big( \valuebar_{\Bel}^{(\numerOrder)} \big).\label{eq:defn-projected-discretized-fixed-pt}
\end{align}
Note that the function $\valuebar_{\Bel}^{(\numerOrder)} $ lives in the finite-dimensional linear subspace $\LinSpace$, which can be represented using the basis functions $(\psi_j)_{j \geq 1}$. Under such a basis function representation, the projected fixed-point equation~\eqref{eq:defn-projected-discretized-fixed-pt} is a $\mbasis$-dimensional linear equation. In particular, we note that Eq~\eqref{eq:defn-projected-discretized-fixed-pt} is equivalent to the orthoganlity conditions
\begin{align}
   \statinprod{\valuebar_{\Bel}^{(\numerOrder)} - \BellOp^{(\numerOrder)}  \big( \valuebar_{\Bel}^{(\numerOrder)} \big)}{\psi_j} = 0, \quad \mbox{for } j = 1,2, \cdots, \mbasis.\label{eq:galerkin-orthogonality-condition}
\end{align}
When trajectory data are available, such an equation can be approximated empirically and solved through efficient stochastic approximation methods. We defer detailed discussion to next subsections.

Now we start establishing the approximation guarantees for the projected fixed-point. By applying known results in discrete-time reinforcement learning~\cite{tsitsiklis1997analysis} in combination with Theorem~\ref{thm:discretization-high-order}, it is straightforward to establish approximation guarantees for the projected fixed-point $\valuebar_{\Bel}^{(\numerOrder)}$, as stated in the following corollary.
\begin{corollary}\label{cor:projected-discretized-bellman}
  Under the setup of Theorem~\ref{thm:discretization-high-order}, for stepsize satisfying $\stepsize \leq \tfrac{1}{\discount \numerOrder}$ the unique solution $\valuebar_{\Bel}^{(\numerOrder)} $ to Eq~\eqref{eq:defn-projected-discretized-fixed-pt} exists, satisfying the bound
  \begin{align*}
    \statnorm{ \valuebar_{\Bel}^{(\numerOrder)} - \ValTrue} \leq 2 \discount^{-3/2} \constScary_\numerOrder \stepsize^{\numerOrder - 1/2 } + \frac{1}{\sqrt{\discount (\numerOrder - 1) \stepsize}} \inf_{\ValFunc \in \LinSpace} \statnorm{\ValFunc - \ValTrue}.
  \end{align*}
\end{corollary}
\noindent See Section~\ref{subsec:proof-cor-projected} for the proof of this corollary.

A few remarks are in order. First, the error in the projected fixed-point $\valuebar_{\Bel}^{(\numerOrder)}$ are from two sources: the approximation error of the linear subspace $\LinSpace$, and the numerical errors induced by discrete-time approximation. Both terms come along with an amplifying factor of order $1 / \sqrt{\stepsize}$. The amplification effect makes the numerical error weaker by half an order, and more importantly, with a fixed subspace, the approximation error can blow up by taking a smaller stepsize. Such an approximation factor stems from the dependence on ``effective horizon'' in discrete-time RL problems, and is known to be unavoidable in the worst case~\cite{mou2023optimal}. Fortunately, in the next sections, we will show that the structures in the underlying diffusion process~\eqref{eq:cts-time-process} can be utilized to improve the approximation error guarantees.

\subsection{Improved approximation guarantees under ellipticity}
The approximation error bounds in~\Cref{cor:projected-discretized-bellman} involves an approximation factor of order $\stepsize^{-1/2}$, which may amplify the approximation error significantly when taking a small stepsize. Such a blow-up can be avoided under some stronger assumptions. We present the improved results in this section.

First, we need the diffusion term to be uniformly elliptic. Uniform ellipticity is a standard assumption in PDE literature, which guarantees the existence and uniqueness of the solution in a Sobolev space, as well as regularity of the solution.
\myassumption{UE$(\lammin, \lammax)$}{assume:uniform-elliptic}
{
  There exists positive constants $\lammin, \lammax$, such that
  \begin{align*}
    \lammin I_d \preceq \covMat (x) \preceq \lammax I_d, \quad \mbox{for any }x \in \real^d.
  \end{align*}
}{3cm}
When degenerate diffusion coefficients are allowed, we conjecture that a hypo-elliptic H\"{o}rmander condition~\cite{hormander1967hypoelliptic} suffices, and we defer a detailed discussion to future works.

Additionally, we require the stationary distribution $\stationary$ to satisfy a Lipschitz assumption at logarithmic scale.
\myassumption{SL$(L_\stationary)$}{assume:smooth-stationary}
{
 For any $x \in \StateSpace$, we have $\vecnorm{\nabla \log \stationary (x)}{2} \leq \smoothness_\stationary$.
}{1.6cm}

In many cases, such an assumption is implied by regularity conditions in Section~\ref{sec:time-discretization}. For example, in the Langevin diffusion case discussed in Example~\ref{example:langevin-diffusion}, for $\covMat = \IdMat_\usedim$ and $\drift = - \nabla U$, the stationary distribution $\stationary \propto e^{- U}$ satisfies Assumption~\ref{assume:smooth-stationary} with $\smoothness_\stationary = \smoothness_\drift^{(0)}$.

Finally, we require the linear subspace $\LinSpace$ to be spanned by sufficiently smooth functions
\myassumption{Reg$(c_0, \omega)$}{assume:basis-condition}
{
 For any function $\ValFun \in \LinSpace$ we have
 \begin{align*}
    \statnorm{\nabla \ValFunc} \leq c_0 \mbasis^\omega \statnorm{f}, \quad \mbox{and} \quad \statnorm{\nabla^2 \ValFunc} \leq c_0 \mbasis^{2 \omega} \statnorm{f}.
  \end{align*} 
}{2.2cm}
Intuitively, Assumption~\ref{assume:basis-condition} imposes ``reverse Poincar\'{e}'' inequalities on functions in the linear subspace. Such inequality cannot hold true globally in $\ltwospace (\stationary)$. However, we can easily verify such a condition within a finite-dimensional subspace. We provide an estimate in the following toy example.
\begin{proposition}\label{prop:fourier-basis-example}
  Let the state space $\StateSpace$ be the $\usedim$-dimensional torus $\torus^\usedim$ and let $\stationary$ be uniform. Let $\LinSpace = \mathrm{span} (\phi_1, \phi_2, \cdots, \phi_\mbasis)$ be spanned by the first $\mbasis$ Fourier basis functions. Then Assumption~\ref{assume:basis-condition} is satisfied with $\omega = \tfrac{1}{d}$ and $c_0 = \usedim^2$.
\end{proposition}
\noindent See~\Cref{subsec:app-proof-prop-fourier-basis-example} for the proof of this proposition.

Similar arguments also apply to the Gaussian space and subspaces spanned by Hermite polynomials. Note that the exponent $\omega$ decreases to $0$ as the dimension $\usedim$ increases. This is because the number of low-frequency basis functions grow exponentially fast as the dimension grows. In general, the stationary distribution $\stationary$ is unknown and we cannot find an orthonormal basis, but similar bounds can still be established for linear subspaces spanned by a set of (not necessarily orthonormal) basis functions.

Given above assumptions, we impose the following requirement on the stepsize
\begin{align}
  \stepsize \numerOrder \leq \min \Big\{ \frac{\mbasis^{- 4 \omega}}{96 e c_0^2 \usedim^2 ( \smoothness_\drift^{(0)} + \smoothness_\covMat^{(0)})^2}  , \frac{\lammin}{(2 \smoothness_\drift^{(0)} + \lammax \smoothness_\stationary + \usedim \smoothness_\covMat^{(1)} )^2}, \frac{\lammin}{\smoothness_\covMat^{(1)} \usedim^3 + \smoothness_\drift^{(1)} \usedim^2} \Big\}.\label{eq:stepsize-requirement-in-improved-approx-factor-bound}
\end{align}
Under these conditions, not only can we bound the $\ltwospace (\stationary)$-error for the projected discretized fixed-point itself, but we can also control the errors in its gradient. In doing so, we define the following Sobolev inner product for a pair $(f, g)$ of functions.
\begin{align}
  \inprod{f}{g}_{\mathbb{H}^1} \mydefn \int f (x) g(x) \stationary (x) dx + \int \inprod{\nabla f (x)}{\nabla g (x)} \stationary (x) dx,\label{eq:sobolev-one-inprod}
\end{align}
and the Sobolev norm $\sobonorm{f} \mydefn \sqrt{\inprod{f}{f}_{\mathbb{H}^1}}$. We further define the Sobolev space $\mathbb{H}^1 \mydefn \{f: \sobonorm{f} < + \infty \}$. With these notations, we can state the main theorem.

\begin{theorem}\label{thm:improved-approx-factor-elliptic}
  Under Assumptions~\fakerefassumelip{$(\numerOrder + 1)$},~\ref{assume:uniform-elliptic}, and~\ref{assume:basis-condition}. For stepsize $\stepsize$ satisfying Eq~\eqref{eq:stepsize-requirement-in-improved-approx-factor-bound}, there exist constants $c_1, c_2 > 0$ depending only on the constants $(\lammax, \lammin, \smoothness_\drift^{(0)}, \smoothness_\drift^{(1)}, \smoothness_\covMat^{(1)}, \smoothness_\stationary, \usedim, \discount )$, such that 
  \begin{align*}
    \sobonorm{\valuebar_{\Bel}^{(\numerOrder)} - \ValTrue} \leq c_1 \cdot \inf_{f \in \LinSpace} \sobonorm{f - \ValTrue} + c_2 \Big\{ \constScary_\numerOrder \stepsize^{\numerOrder} + \constScary_{\numerOrder + 1}  \stepsize^{\numerOrder + 1} \Big\}.
  \end{align*}
\end{theorem}
\noindent See Section~\ref{subsec:proof-thm-improved-approx-factor-elliptic} for the proof of this theorem.

A few remarks are in order. First, in contrast to the worst-case discretized RL guarantees in~\Cref{cor:projected-discretized-bellman}, the pre-factor $c_1$ in front of the approximation error is uniformly bounded, and does not blow up as stepsize decreases. Similarly, the numerical order exhibited by~\Cref{thm:improved-approx-factor-elliptic} matches the order of the scheme itself. Compared with~\Cref{cor:projected-discretized-bellman}, this shows the value of incorporating structures of the continuous-time dynamics in RL. Second, the constants $(c_1, c_2)$ both admit explicit expressions, which are polynomial in the relevant quantities. See~\Cref{eq:improved-approx-factor-full} in the proof for their concrete forms.

Moreover, as with existing literature of Galerkin method,~\Cref{thm:improved-approx-factor-elliptic} gives guarantees in a weighted Sobolev norm instead of the standard $\statnorm{\cdot}$-norm. Such a bound gives stronger guarantees, but also requires stronger assumptions on the ground truth $\ValTrue$. Owing to the strong norm, the $\numerOrder$-th order guarantee in~\Cref{thm:improved-approx-factor-elliptic} requires $(\numerOrder + 1)$-th order smoothness assumptions. On the other hand, the higher-order smoothness constant $\constScary_{\numerOrder + 1}$ appears only in the high-order $\stepsize^{\numerOrder + 1}$ term, making the dependence mild.

Finally, Theorem~\ref{thm:improved-approx-factor-elliptic} relies on the stepsize requirement~\eqref{eq:stepsize-requirement-in-improved-approx-factor-bound}, which involves the number of basis functions $\mbasis$, along with other problem-dependence constants. In particular, we require that $\stepsize \lesssim \mbasis^{- 4 \omega}$. Such a requirement reflects a difference between direct model-based methods and our model-free methods. As we can see from the proof of~\Cref{thm:improved-approx-factor-elliptic}, the standard $\sobonorm{\cdot}-$coercivity bound does not hold true for the discrete-time operator. Instead, we use this condition to ensure that functions in the linear subspace satisfy the desired bounds.

\subsection{Approximation guarantees for higher-order generator}\label{subsec:approx-gen}
Similar to the previous two sections, we can also project the high-order approximate generator~\eqref{eq:defn-high-order-generator} to a finite-dimensional subspace.
In particular, we consider the projected discretized diffusion equation
\begin{align}
   \projecttolin \circ \l((\geninter{\numerOrder} - \discount)  \valtdinter{\numerOrder} +\reward \r) = 0,\label{eq:defn-projected-discretized-gen}
\end{align}
which is equivalent to the Galerkin orthogonality conditions
\begin{align*}
  \statinprod{(\geninter{\numerOrder} - \discount)  \valtdinter{\numerOrder}  +\reward}{\psi_j} = 0, \quad \mbox{for } j = 1,2, \cdots, \mbasis.
\end{align*}
Similar to the high-order Bellman operator, we need a constraint on the stepsize which depends on the regularity of functions in the linear subspace $\LinSpace$. In particular, let $\Phi = (\phi_j)_{j = 1}^\mbasis$ be an orthonormal basis of $\LinSpace$, we require
\begin{align}
  \stepsize^\numerOrder \leq \frac{\discount}{2C_{a,\numerOrder}\ll \generator^{\numerOrder+1} \Phi(\state)\rl_\stationary}, \quad \mbox{where} \quad C_{a,\numerOrder} \mydefn \frac{\sum_{j=0}^\numerOrder|\coef{\numerOrder}_j|j^{\numerOrder+1}}{ (\numerOrder+1)!}. \label{eq:stepsize-requirement-in-improved-gen-factor-bound}
\end{align}
The stepsize requirement~\eqref{eq:stepsize-requirement-in-improved-gen-factor-bound} can be seen as a high-order generalization of the requirement~\eqref{eq:stepsize-requirement-in-improved-approx-factor-bound}. In particular, Eq~\eqref{eq:stepsize-requirement-in-improved-approx-factor-bound} essentially bounds the stepsize using an upper bound on the generator applied to functions in $\LinSpace$. In comparison, Eq~\eqref{eq:stepsize-requirement-in-improved-gen-factor-bound} bounds the $\numerOrder$-th power of stepsize using $\numerOrder$-th order generator applied to these functions. Under the setup of~\Cref{prop:fourier-basis-example}, they lead to the same scaling of the stepsize in terms of $(\mbasis, \usedim)$, while in general, the two requirements are not comparable.

Given such a stepsize requirement, we have the theoretical guarantees for the solution to high-order projected discrete-time diffusion equations.
\begin{theorem}\label{thm:approx-gen}
  Under Assumptions~\fakerefassumelip{$(\numerOrder + 1)$} and~\ref{assume:uniform-elliptic}. For stepsize $\stepsize$ satisfying Eq~\eqref{eq:stepsize-requirement-in-improved-gen-factor-bound}, there exist constants $c_3, \constgen_\numerOrder > 0$ depending only on the constants $(\lammax, \lammin, \smoothness_\drift^{(0)}, \smoothness_\drift^{(1)}, \smoothness_\covMat^{(1)}, \smoothness_\stationary, \usedim )$, such that
  \begin{align*}
    \sobonorm{ \valtdinter{\numerOrder} - \valuestar } \leq c_3 \inf_{\ValFunc \in \LinSpace} \sobonorm{\ValFunc - \valuestar}  +  \frac{3\constgen_\numerOrder}{\discount} \dt^\numerOrder
  \end{align*}
\end{theorem}
\noindent See Section~\ref{subsec:proof-thm-approx-gen} for the proof of this theorem.

A few remarks are in order. First, unlike~\Cref{thm:generator-high-order}, the guarantees in~\Cref{thm:approx-gen} applies to any numerical order $\numerOrder \geq 1$. So even if the original discrete-time diffusion equation becomes unstable due to backward differentiation, the projected version still enjoys desirable guarantees. This phenomenon is due to the $\ltwospace$-type arguments used in Theorem~\ref{thm:approx-gen}, as well as the regularity of functions in the linear subspace. The bounds in Theorem~\ref{thm:approx-gen} and Theorem~\ref{thm:improved-approx-factor-elliptic} are of the same order, while the precise constant pre-factors are not comparable in general. When implemented using empirical data, they may also lead to different statistical errors. We defer a detailed comparison to future works.

\subsection{Solving the projected-fixed points using empirical data}\label{subsec:sample-based}
Now that we have defined the discrete-time finite-dimensional fixed-points $\valuebar_{\Bel}^{(\numerOrder)}$ and $\valtdinter{\numerOrder}$, which serve as targets for data-driven approaches. In this section, we describe how to use a discretely-observed trajectory $(\MyState_{k \stepsize}, \Reward_{k \stepsize})_{0 \leq k \leq \totaltime / \stepsize}$ to approximate these targets. The statistical guarantees, as well as their optimality properties, are important topics for practical applications, and will be studied thoroughly in our companion paper.

Recall that $(\psi_j)_{j = 1}^\mbasis$ is a set of basis functions whose span defines the linear subspace $\LinSpace$. We use $\psi (x) \mydefn \big[ \psi_j (x) \big]_{j = 1}^\mbasis$ to denote the mapping from a state to its feature vector. Note that the bases need not to be orthonormal. We first represent the discrete-time projected fixed-point $\valuebar$ using finite-dimensional basis $\valuebar \mydefn \thetabar^\top \psi$. Under such representation, the Galerkin orthogonality condition~\eqref{eq:galerkin-orthogonality-condition} takes the form
\begin{align*}
  \Exs_{\MyState_0 \sim \stationary} \Big[ \Big( \psi (\MyState_0)^\top \thetabar - \BellOp^{(\numerOrder) } (\psi (\MyState_0)^\top \thetabar) \Big) \cdot \psi_j (\MyState_0)  \Big] = 0, \quad \mbox{for $j = 1,2,\cdots, \mbasis$}.
\end{align*}
Therefore, the projected fixed-point equation~\eqref{eq:defn-projected-discretized-fixed-pt} can be re-written as
\begin{align*}
  \Exs \big[ \psi (\MyState_0) \psi (\MyState_0)^\top \big] \thetabar = \sum_{i = 0}^{\numerOrder - 1} \int_0^{(\numerOrder - 1) \stepsize} e^{- \discount s} W_i (s) \Exs \big[ \reward (\MyState_{i \stepsize}) \psi (\MyState_0) \big] ds + e^{- \discount (\numerOrder - 1) \stepsize} \Exs \big[ \psi (\MyState_0) \psi (\MyState_{(\numerOrder - 1) \stepsize})^\top \big] \thetabar,
\end{align*}
where $\MyState_0 \sim \stationary$, and $(\MyState_t)_{t \geq 0}$ evolves according to Eq~\eqref{eq:cts-time-process}.

\begin{subequations}\label{eq:finite-sample-estimator-bellman}
Given a trajectory $(\MyState_{t})_{0 \leq t \leq \totaltime}$ observed at discrete time steps $(k\stepsize: k \geq 0)$, we can approximate the solution to fixed-point equation~\eqref{eq:defn-projected-discretized-fixed-pt} using empirical samples. In particular, given the discretely-observed trajectory $(\MyState_{k \stepsize}, \Reward_{k \stepsize})_{0 \leq k \leq \lceil \totaltime / \stepsize \rceil}$, we solve the linear equation
   \begin{align}
     \thetahat = \Big\{ \sum_{k = 0}^{\totaltime / \stepsize - \numerOrder} \psi (\MyState_{k \stepsize})  \cdot \Big( \psi (\MyState_{k \stepsize}) - e^{- \discount (\numerOrder - 1) \stepsize}  \psi (\MyState_{ (k + \numerOrder - 1) \stepsize}) \Big)^\top \Big\}^{-1} \cdot \Big\{ \stepsize \sum_{k = 0}^{\totaltime / \stepsize - \numerOrder}  \sum_{i = 0}^{\numerOrder - 1} \kappa_i \Reward_{(k + i) \stepsize} \psi (\MyState_{k \stepsize}) \Big\}, \label{eq:empirical-lstd}
   \end{align}
where the coefficients $\kappa_i$ are defined as
   \begin{align}
     \kappa_i \mydefn \frac{1}{\stepsize} \int_0^{(\numerOrder - 1) \stepsize} e^{- \discount s} W_i (s) ds.\label{eq:defn-kappa-i-coeff}
   \end{align}
   Finally, the estimated value function is given by
   \begin{align}
     \ValHat_{\Bel}^{(\numerOrder)}(\state) \mydefn \inprod{\thetahat}{\psi (\state)}, \quad\mbox{for any $\state \in \StateSpace$}.
   \end{align}
\end{subequations}

Similarly, we can use empirical trajectory to estimate the approximate projected diffusion generator in Section~\ref{subsec:approx-gen}. In order to do so, recall the definitions $A^{(\numerOrder)} = \big[ j^k \big]_{0 \leq k, j \leq \numerOrder}$, $b^{(\numerOrder)} = \big[ \bm{1}_{j = 1} \big]_{0 \leq j \leq \numerOrder}$ and $\coef{\numerOrder} \mydefn  A^{-1} b$. By representing the discrete-time projected fixed-point $\valtdinter{\numerOrder}$ using basis functions $\valtdinter{\numerOrder} = \thetabar^\top \psi$, we have
\begin{align}
\discount \Exs \big[ \psi (\MyState_0) \psi (\MyState_0)^\top \big] \thetabar  - \frac{1}{\stepsize} \sum_{i = 0}^{\numerOrder} \coef{\numerOrder}_i  \Exs \big[ \psi (\MyState_0) \psi (\MyState_{i \stepsize})^\top \big] \thetabar = \Exs \big[ \reward (\MyState_0) \psi (\MyState_0) \big]\label{eq:empirical-gen-lstd}
\end{align}
By replacing the expectations with empirical averages, the value function can be estimated by solving the $\mbasis$-dimensional linear system
\begin{align}
    \thetahat = \Big\{ \sum_{k = 0}^{\totaltime / \stepsize - \numerOrder} \psi (\MyState_{k \stepsize})  \cdot \Big( \discount \psi (\MyState_{k \stepsize}) - \frac{1}{\stepsize} \sum_{i = 0}^{\numerOrder} \coef{\numerOrder}_i  \psi (\MyState_{(k + i) \stepsize}) \Big)^\top \Big\}^{-1} \cdot \Big\{ \sum_{k = 0}^{\totaltime / \stepsize - \numerOrder} \Reward_{k \stepsize} \psi (\MyState_{k \stepsize}) \Big\},\label{eq:finite-sample-estimator-diffusion}
\end{align}
with the estimated value function given by $\ValHat_{\gen}^{(\numerOrder)} (x) \mydefn \inprod{\thetahat}{\psi (x)}$ for any $x \in \StateSpace$.

In our companion paper, we will analyze the statistical errors for the estimator constructed in Equations~\eqref{eq:finite-sample-estimator-bellman} and~\eqref{eq:finite-sample-estimator-diffusion}.

\subsection{Beyond stationarity: guarantees under discounted occupancy measure}\label{subsec:occumsr-results}
The projection procedures described above as well as theoretical guarantees in Theorem~\ref{thm:improved-approx-factor-elliptic} and~\ref{thm:approx-gen} require the existence of the stationary measure $\stationary$, which may not hold true in general. On the other hand, the theory of discounted RL does not rely on stationarity. In this section, we relax such conditions and study diffusion processes that may diverge. To start with, given an initial probability measure $\pi_0$ on $\StateSpace$, we define the discounted occupancy measure
\begin{align*}
  \occumsr \mydefn \discount \int_0^{+ \infty} e^{- \discount t} \semigroup_t^* \pi_0  dt.
\end{align*}
The discounted occupancy measure is widely used in discounted RL literature. Operationally, given a Markov process trajectory starting with $\MyState_0 \sim \pi_0$, running it for an independent exponential time $T \sim \mathrm{Exp} (\discount)$, a uniform sample at random from the trajectory $(X_t)_{0 \leq t \leq T}$ follows the distribution $\occumsr$. Consequently, the weighted inner product $\inprod{\cdot}{\cdot}_{\occumsr}$ and norm $\vecnorm{\cdot}{\occumsr}$ can be approximated empircally by averaging over $\mathrm{i.i.d.}$ trajectories killed at a constant rate.

Similar to Eq~\eqref{eq:defn-projected-discretized-fixed-pt}, we define the projected discretized fixed-point as
\begin{align}
  \valuebar^{\numerOrder}_\Bel = \projectto{\LinSpace, \occumsr} \circ \BellOp^{(\numerOrder)}  \big( \valuebar_{\Bel}^{(\numerOrder)} \big),\label{eq:defn-projected-discretized-fixed-pt-occumsr}
\end{align}
which can be equivalently expressed as Galerkin orthogonality conditions
\begin{align}
   \inprod{\valuebar^{\numerOrder}_\Bel - \BellOp^{(\numerOrder)}  \big( \valuebar^{\numerOrder}_\Bel  \big)}{\psi_j}_\occumsr = 0, \quad \mbox{for } j = 1,2, \cdots, \mbasis.\label{eq:galerkin-orthogonality-condition-occumsr}
\end{align}
We have the following theoretical guarantees under the discounted occupancy measure
\begin{corollary}\label{cor:discounted-occu-msr}
   Under Assumptions~\fakerefassumelip{$(\numerOrder + 1)$},~\ref{assume:uniform-elliptic}. Assume furthermore that Assumption~\ref{assume:smooth-stationary} is satisfied by $\occumsr$ with constant $\smoothness_\occumsr$, and that Assumption~\ref{assume:basis-condition} is satisfied under the $\ltwospace (\occumsr)$-norm. For stepsize $\stepsize$ satisfying Eq~\eqref{eq:stepsize-requirement-in-improved-approx-factor-bound}, there exist constants $c_1, c_2 > 0$ depending only on the constants $(\lammax, \lammin, \smoothness_\drift^{(0)}, \smoothness_\drift^{(1)}, \smoothness_\covMat^{(1)}, \smoothness_\occumsr, \usedim )$, such that 
  \begin{align*}
    \vecnorm{\valuebar_{\Bel}^{(\numerOrder)} - \ValTrue}{\sobospace (\occumsr)} \leq c_1 \cdot \inf_{f \in \LinSpace} \vecnorm{f - \ValTrue}{\sobospace (\occumsr)} + c_2 \Big\{ \constScary_\numerOrder \stepsize^{\numerOrder} + \constScary_{\numerOrder + 1}  \stepsize^{\numerOrder + 1} \Big\}.
  \end{align*}
\end{corollary}
\noindent See~\Cref{sec:app-proof-cor-discounted-occu-msr} for the proof of this corollary.

Note that \Cref{cor:discounted-occu-msr} goes exactly in parallel with \Cref{thm:improved-approx-factor-elliptic}, with the stationary distribution $\stationary$ replaced by the discounted occupancy measure $\occumsr$ (in both Assumption~\ref{assume:smooth-stationary} and the weighted $\ltwospace$-norm). Similar conclusion can be drawn on the high-order diffusion generator, with an analogous version of Theorem~\ref{thm:approx-gen}. We omit the details in this case for simplicity. The smoothness assumption of the function $\log \occumsr$ can be verified in many cases. We present a simple example as follows
\begin{proposition}\label{prop:occumsr-example}
  If the diffusion process~\eqref{eq:cts-time-process} is a standard Brownian motion $dX_t = dB_t$, and $\sup_\state \vecnorm{\nabla \log \pi_0 (\state)}{2} \leq \smoothness$, then we have $\sup_\state \vecnorm{\nabla \log \occumsr (\state)}{2} \leq \smoothness$.
\end{proposition}
\noindent See Appendix~\ref{subsec:proof-occumsr-example} for the proof of this proposition.

Without ergodicity of the process, the estimation of value function is impossible using a single trajectory. On the other hand, we can approximate the solution reasonably well using multiple independent exponentially-killed trajectories. In particular, consider $\mathrm{i.i.d.}$ discretely-observed trajectories
\begin{align*}
  \MyState_0^{(i)}, \Reward_0^{(i)}, \MyState_{\stepsize}^{(i)}, \Reward_\stepsize^{(i)}, \cdots, \MyState_{K_i \stepsize}^{(i)}, \Reward_{K_i \stepsize}^{(i)}, \quad \mbox{for $i = 1,2,\cdots, \numobs$},
\end{align*}
where $K_i \sim \mathrm{i.i.d.} \mathrm{Geom} (e^{- \discount \stepsize})$, independent of the trajectories. Given these data, we can derive sample-based algorithms analogous to those in Section~\ref{subsec:sample-based}.

For the high-order Bellman equation, we define the empirical estimator
\begin{multline*}
     \thetahat_\numobs = \Big\{ \sum_{i = 1}^\numobs \sum_{k = 0}^{K_i - \numerOrder} \psi (\MyState_{k \stepsize}^{(i)})  \cdot \Big( \psi (\MyState_{k \stepsize}^{(i)}) - e^{- \discount (\numerOrder - 1) \stepsize}  \psi (\MyState_{ (k + \numerOrder - 1) \stepsize}^{(i)}) \Big)^\top \Big\}^{-1}\\
      \cdot \Big\{ \stepsize\sum_{i = 1}^{\numobs}  \sum_{k = 0}^{K_i - \numerOrder}  \sum_{j = 0}^{\numerOrder - 1} \kappa_i \Reward_{(k + j) \stepsize}^{(i)} \psi (\MyState_{k \stepsize}^{(i)}) \Big\}.
\end{multline*}
We can also construct the empirical estimator based on the high-order diffusion equation.
\begin{align*}
    \thetahat_\numobs = \Big\{ \sum_{i = 1}^\numobs \sum_{k = 0}^{K_i - \numerOrder} \psi (\MyState_{k \stepsize}^{(i)})  \cdot \Big( \discount \psi (\MyState_{k \stepsize}^{(i)}) - \frac{1}{\stepsize} \sum_{j = 0}^{\numerOrder} \coef{\numerOrder}_j  \psi (\MyState_{(k + j) \stepsize}^{(i)}) \Big)^\top \Big\}^{-1} \cdot \Big\{\sum_{i = 1}^\numobs \sum_{k = 0}^{K_i - \numerOrder} \Reward_{k \stepsize}^{(i)} \psi (\MyState_{k \stepsize}^{(i)}) \Big\}.
\end{align*}
Similar to Equations~\eqref{eq:finite-sample-estimator-bellman} and~\eqref{eq:finite-sample-estimator-diffusion}, the value function estimators can be computed via $\mbasis$-dimensional linear systems. We defer the statistical analysis under the multiple trajectory setup to our companion paper.

\def\lam{\lambda}
\def\s{\sigma}
\def\hs{\hat{\s}}
\def\dt{\stepsize}
\def\dx{\delta}
\section{Numerical simulation}\label{sec:simulation}
We present the numerical simulation results in this section. We begin with the exact solutions to Equations~\eqref{eq:defn-high-order-bellman-operator} and~\eqref{eq:defn-high-order-generator}. Subsequently, in Section~\ref{subsec:sample-based}, we shift our focus to fully data-driven approaches. Some of the simulation setups explored here may go beyond the assumptions outlined in the theoretical results of previous sections. This illustrates the robustness of our proposed methods.

\begin{figure}[!ht]
\centering
         \includegraphics[width=0.25\textwidth]{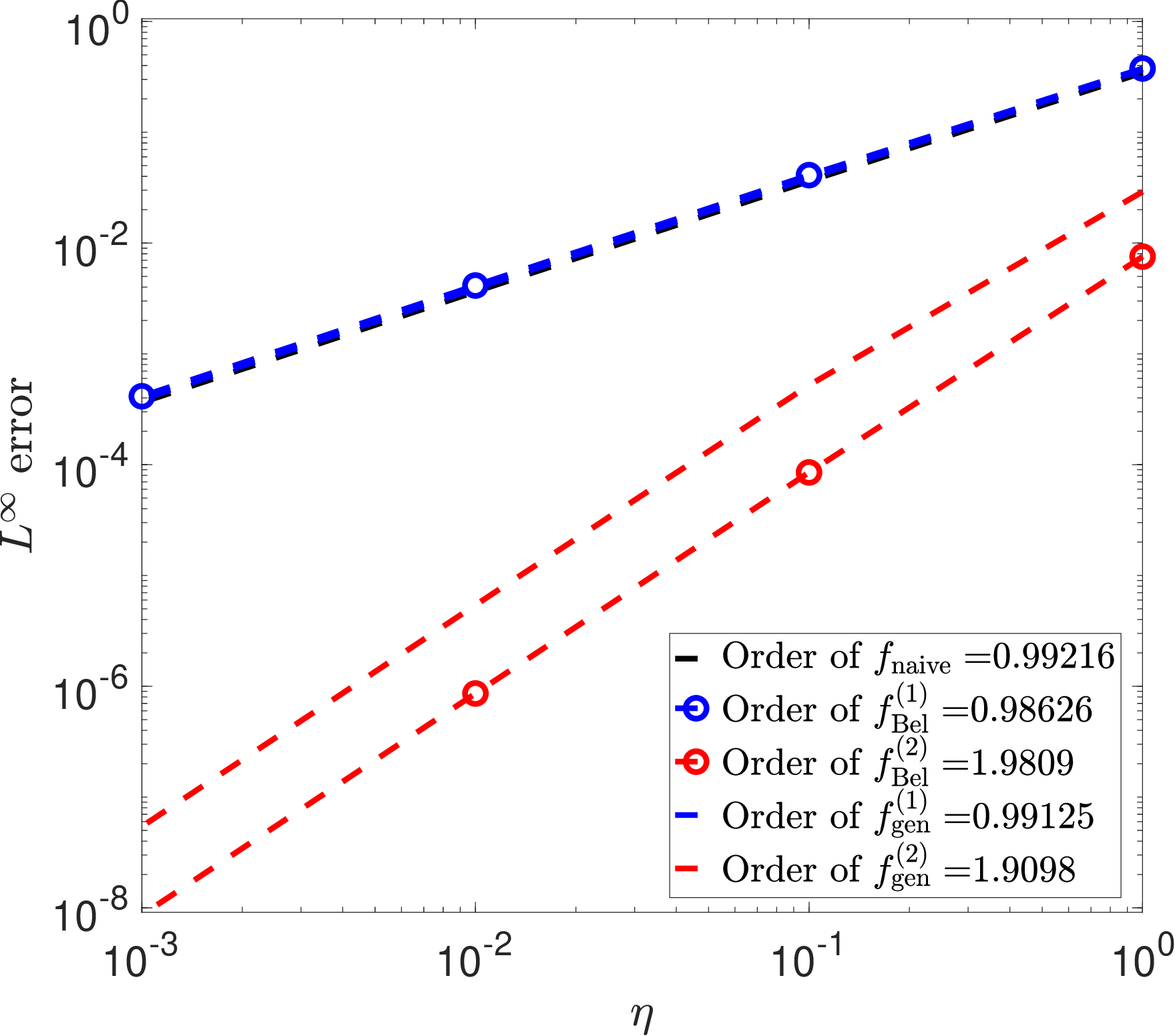}
         \includegraphics[width=0.25\textwidth]{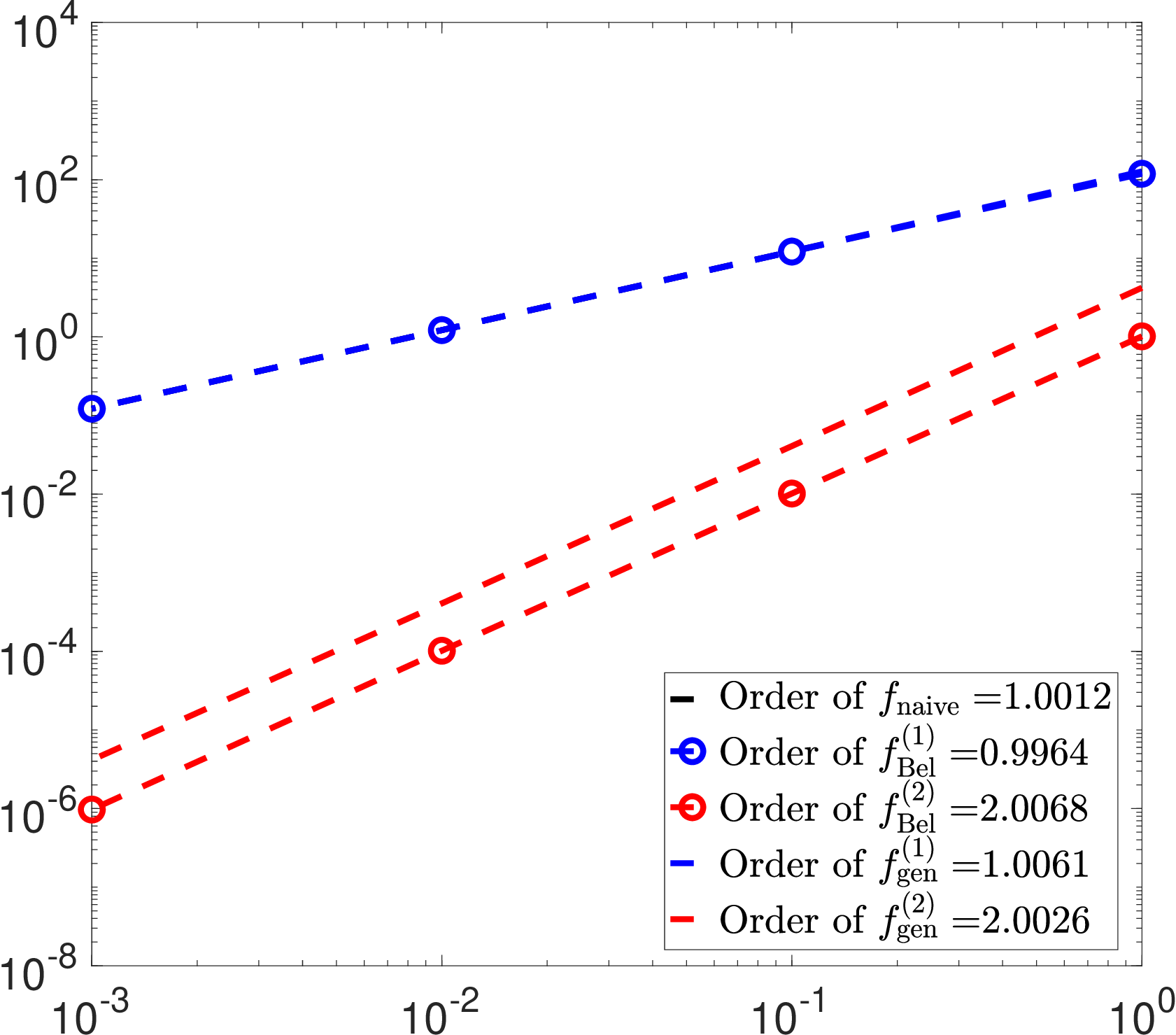}
         \includegraphics[width=0.25\textwidth]{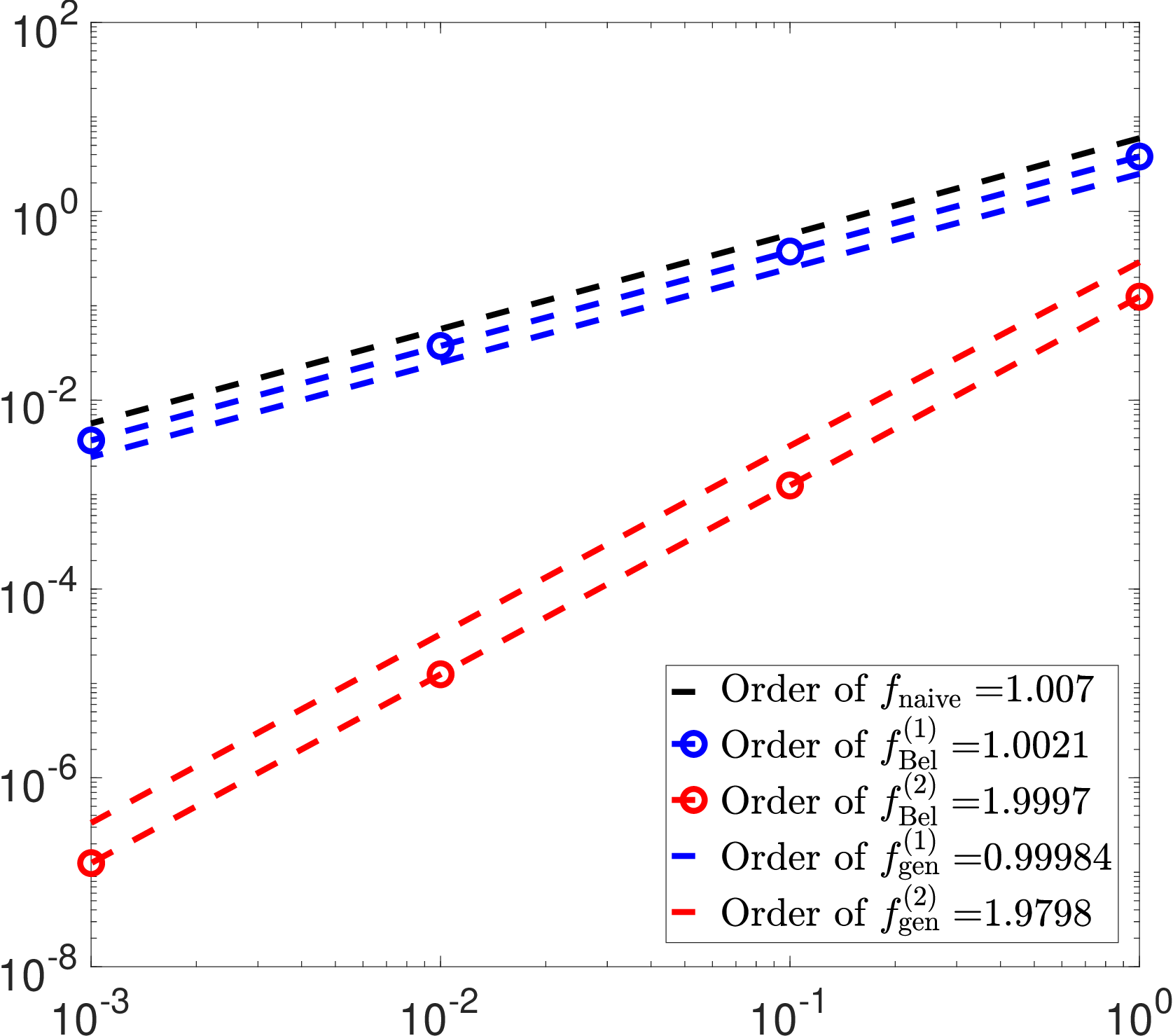}\\
         
        \includegraphics[width=0.25\textwidth]{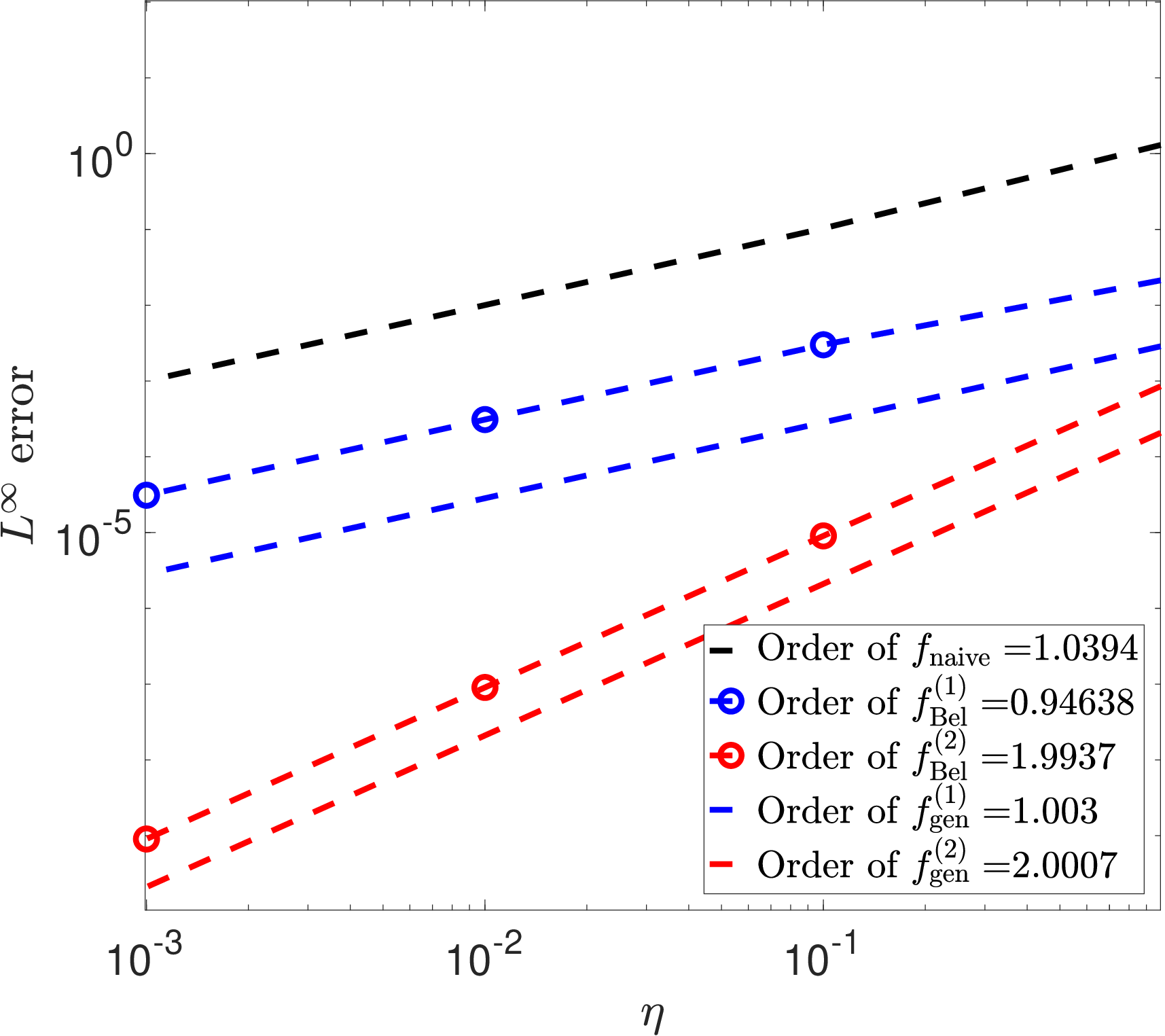}
         \includegraphics[width=0.25\textwidth]{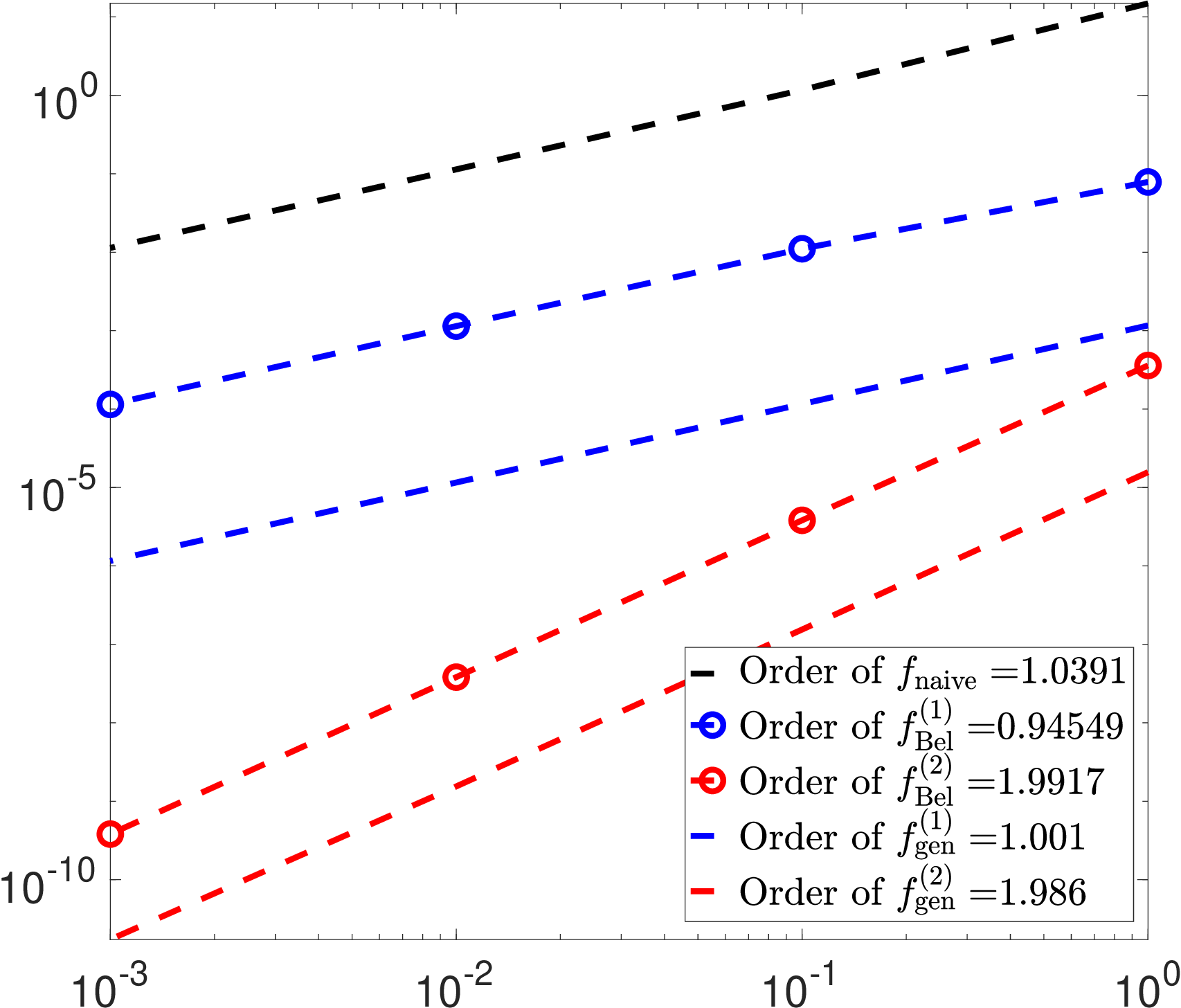}
         \includegraphics[width=0.25\textwidth]{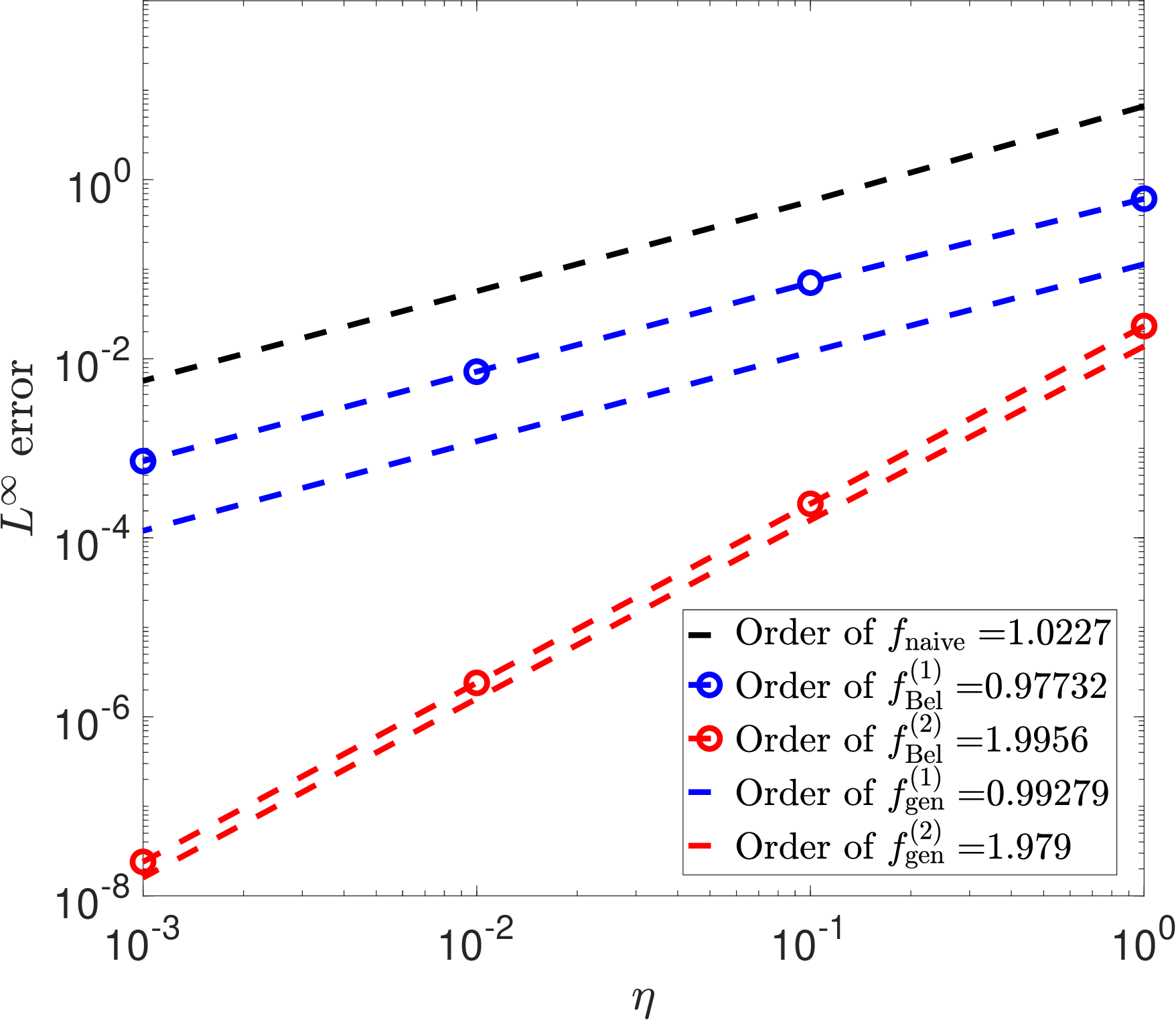}
  
\caption{The above figure plots the error of the solution as the step size \( \stepsize \) decreases. 
\textbf{Left:} The dynamics follow \eqref{deter-dyna}, and the reward is \eqref{deter-reward-1} with \( \lambda = 0.05, k = 1, \discount = 0.1 \) (above), and \( \lambda = 0.01, k = 2, \discount = 2 \) (below). 
\textbf{Middle:} The dynamics follow \eqref{deter-dyna}, and the reward is \eqref{deter-reward-2} with \( \lambda = 0.01 \), and \( \alpha = 5, b = 1, \discount = 0.1 \) (above), and \( \alpha = 2, b = 2, \discount = 2 \) (below). 
\textbf{Right:} The dynamics follow \eqref{stoch-dy}, and the reward is \eqref{stoch-reward} with \( \sigma = 0.1, \discount = 0.1 \) (above), and \( \sigma = 1, \discount = 1 \) (below).}\label{fig:exact_dt}
\end{figure}

\subsection{Exact solution to the higher-order Bellman equations}\label{subsec:exact-eg}
We first consider a deterministic linear process,
\begin{equation}\label{deter-dyna}
d\MyState_t = \lam \MyState_t dt,
\end{equation}
where the state at time $t$ has an explicit form $\MyState_t = e^{\lam t
}\MyState_0$. We test our results in two settings. In the first setting, the reward is defined as
\begin{equation}\label{deter-reward-1}
\reward(x) = \discount \cos^3(kx)-\lam \state (-3k\cos^2(kx)\sin(kx)),
\end{equation} 
where the value function can be exactly obtained as $\valuestar(x) = \cos^3(kx)$. We very $\lambda, k, \discount$ to observe the decay of the error $\ll \ValInter{\numerOrder} - \valuestar \rl,  \ll \valhatinter{\numerOrder} - \valuestar \rl$ with respect to the step size decay $\dt =[1,0.1,10^{-2}, 10^{-3}] $. The results are plotted in the left column of Figure \ref{fig:exact_dt}. 
In the second setting, the reward is defined as 
\begin{equation}\label{deter-reward-2}
\reward(x) = bx^\alpha,
\end{equation} 
where the value function can be exactly obtained as $\valuestar(x) = \frac{bx^\alpha}{\discount-\alpha\lam}$. We set $\lambda = 0.01$, and vary $\alpha, b, \discount$ to test the error decay  with respect to the step size decay $\dt$. The results are plotted in the middle column of Figure \ref{fig:exact_dt}.

Secondly, we consider a stochastic process, the Ornstein–-Uhlenbeck process,
\begin{equation}\label{stoch-dy}
d\MyState_t = \lam \MyState_t dt + \s dB_t,\quad \text{with}\quad 
\lam = -0.1.
\end{equation}
For OU process, the conditional density function for $s_t$ given $s_0 = s$ follows the normal distribution with expectation $se^{\lam t}$ and variance $\frac{\s^2}{2\lam}(e^{2\lam t}-1)$. Here we consider the case where
the reward is 
\begin{equation}\label{stoch-reward} 
r = x^\alpha,\quad \text{with}\quad \alpha = 2.
\end{equation} and the value function is $V(\state) = (\frac{\sigma^2}{2\lam} + \state^2)\frac{1}{\beta - 2\lam} - \frac{\sigma^2}{2\lam\beta}$. 
We set different $\discount, \sigma$ to test the error decay with respect to the step size $\dt$ decay, The results are plotted in the right column of Figure \ref{fig:exact_dt}.


From Figure \ref{fig:exact_dt}, it is evident that the order of the solution aligns with the theoretical results across all experiments. Note that $f_{(\text{naive})}$ refers to the naive Bellman solution defined in \eqref{eq:naive-discretized-bellman}. When $\discount$ is small, all the first-order solutions exhibit similar errors,  whereas the second-order solution demonstrates significantly smaller errors for the same $\eta$. Additionally, for large $\discount$, the first-order solution in our formulation outperforms the naive Bellman equation solution.  This is because the naive Bellman solution approximates $\int_0^\eta e^{-\discount t} \reward(\MyState_t) dt $ using $\reward(\MyState_0)$, which becomes less accurate as $\discount$ increases. 

To illustrate the improvement made by the higher-order solution compared to the naive Bellman equation solution, we plot the value functions in Figure \ref{fig:exact_v} for the examples in the second row of Figure \ref{fig:exact_dt}. This figure shows that the naive Bellman equation solution deviates significantly from the true value function $\valuestar$, whereas the second-order solutions $\ValInter{2}, \valhatinter{2}$ are almost the same as the true value function.

\begin{figure}[!h]
\centering
         \includegraphics[width=0.25\textwidth]{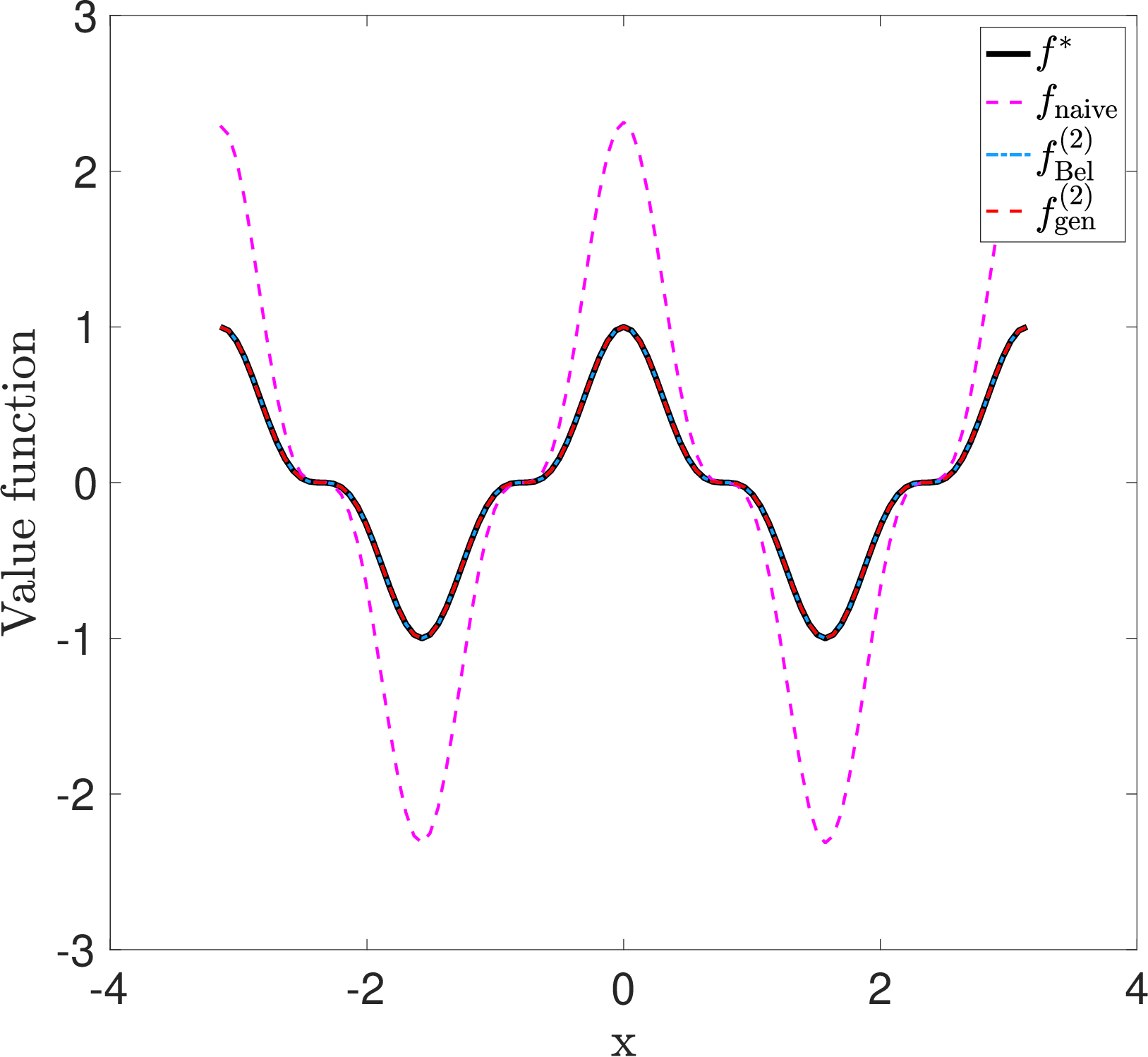}
         \includegraphics[width=0.25\textwidth]{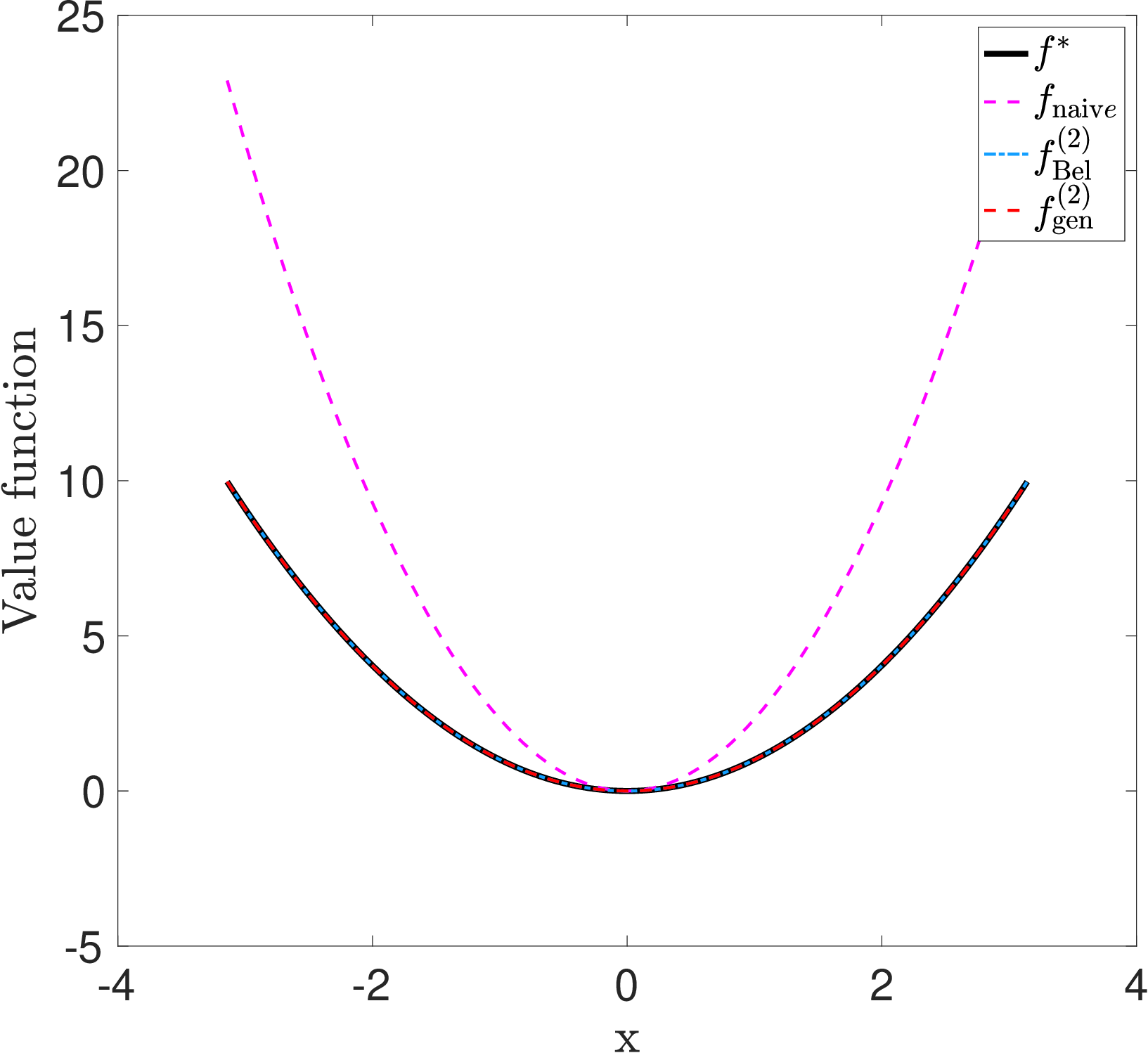}
         \includegraphics[width=0.25\textwidth]{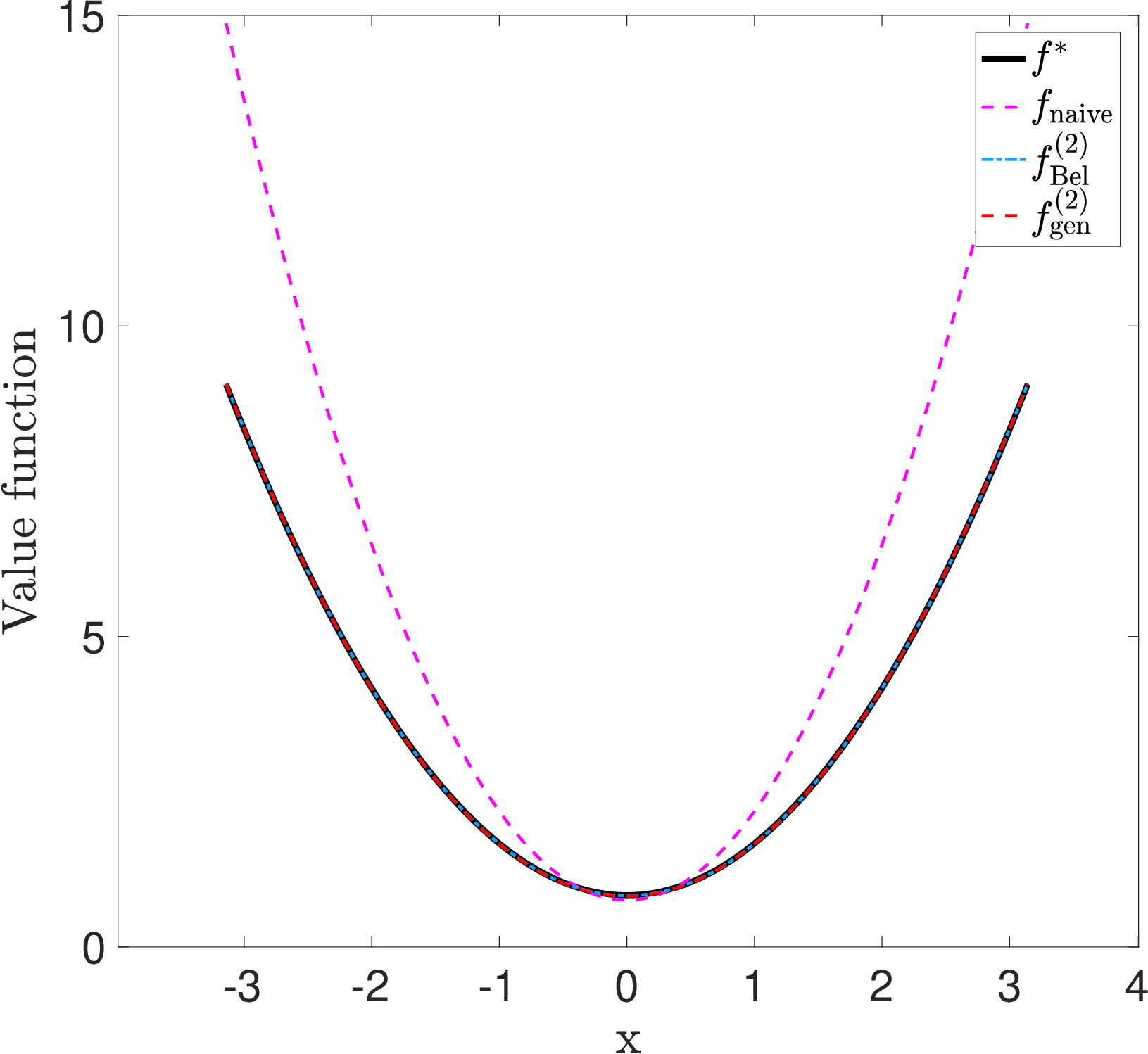}
\caption{The above figure plots the value functions for $\eta = 1$ with the same setting as the second row of Figure \ref{fig:exact_dt}}\label{fig:exact_v}
\end{figure}

Here are the details of the experiments. First, to eliminate the effect of function approximation error, we use bases such that the true value function lies in the space spanned by these finite bases.
\begin{equation}\label{numerical-bases}
\begin{aligned}
&\text{Left column}: \quad \psi(\state) =  \frac{1}{\sqrt{2\pi}}(1, \cos(\state), \sin(\state), \cdots, \cos(m\state), \sin(m\state) ) , \quad \text{with }m = 5k.\\
&\text{Middle column}: \quad \psi(\state) =  \text{the Legendre polynomials up to $\alpha$-th order}.\\
&\text{Right column} \quad \psi(\state) =  \text{the Legendre polynomials up to $\alpha$-th order}.
\end{aligned}
\end{equation}
We find the exact higher-order solutions $\ValInter{\numerOrder}, \valhatinter{\numerOrder}$  based on the formulas \eqref{eq:finite-sample-estimator-diffusion} - \eqref{eq:finite-sample-estimator-bellman}. 
For the deterministic dynamics, the coefficient $\hat{\theta}$ is calculated by the following formula, 
\[
\begin{aligned}
&\thetahat = \Big\{ \sum_{k = 0}^{K} \psi (\MyState_{k })  \cdot \Big( \psi (\MyState_{k }) - e^{- \discount (\numerOrder - 1) \stepsize}  \psi (\MyState'_{k ,(\numerOrder - 1) \stepsize}) \Big)^\top \Big\}^{-1} \cdot \Big\{ \stepsize \sum_{k = 0}^{K}  \sum_{i = 0}^{(\numerOrder - 1) \stepsize} \kappa_i r(\MyState'_{k ,i \stepsize}) \psi (\MyState_{k}) \Big\}, \\
&\thetahat_{\text{gen}} = \Big\{ \sum_{k = 0}^{K} \psi (\MyState_{k })  \cdot \Big( \discount \psi (\MyState_{k }) - \frac{1}{\stepsize} \sum_{i = 0}^{\numerOrder} \coef{\numerOrder}_i  \psi (\MyState'_{k, i\stepsize}) \Big)^\top \Big\}^{-1} \cdot \Big\{ \sum_{k = 0}^{K} r(\MyState_{k})\psi (\MyState_{k }) \Big\}.
\end{aligned}
\]
Instead of using the trajectory data, we set $\MyState_{k } = -\pi+k \dx$ with $\dx = 2\pi/K, K = 400$, and $\MyState'_{k, i\stepsize} = e^{\lambda i\eta}\MyState_{k}$. 
For the stochastic case, the following formula is used,
\[
\begin{aligned}
&\thetahat = \Big\{ \sum_{k = 0}^{K} \psi (\MyState_{k })  \cdot \Big( \psi (\MyState_{k }) - e^{- \discount (\numerOrder - 1) \stepsize}  \psi_{k ,(\numerOrder - 1) \stepsize} \Big)^\top \Big\}^{-1} \cdot \Big\{ \stepsize \sum_{k = 0}^{K}  \sum_{i = 0}^{(\numerOrder - 1) \stepsize} \kappa_i r_{k ,i \stepsize} \psi (\MyState_{k}) \Big\}, \\
&\thetahat_{\text{gen}} = \Big\{ \sum_{k = 0}^{K} \psi (\MyState_{k })  \cdot \Big( \discount \psi (\MyState_{k }) - \frac{1}{\stepsize} \sum_{i = 0}^{\numerOrder} \coef{\numerOrder}_i  \psi_{k, i\stepsize} \Big)^\top \Big\}^{-1} \cdot \Big\{ \sum_{k = 0}^{K} r(\MyState_{k})\psi (\MyState_{k}) \Big\},
\end{aligned}
\]
where $\MyState_{k} $ is the same as the deterministic case, while $\psi_{k ,i \stepsize} = \E[\psi(\MyState_{i\stepsize}) | \MyState_0 = \MyState_{k }],  r_{k ,i \stepsize} =  \E[r(\MyState_{i\stepsize}) | \MyState_0 = \MyState_{k}]$, which can be exactly calculated in our setting. The high-order solutions $\ValInter{\numerOrder}(\state) = \thetahat ^\top \psi(\state),\valhatinter{\numerOrder}(\state) = \thetahat_{\text{gen}}^\top\psi(\state)$ are evaluated at $x_j = -\pi+j\frac{2\pi}{J}$, with $j = 0, \cdots, J, J = 100$. The error is approximated by $\ll f - \valuestar \rl_\infty \approx \max_{j}|f(x_j) - \valuestar(x_j)|$.

\subsection{Approximate solution when only data is available}\label{subsec:data-eg}


When only trajectory data are available, we test our algorithm in both deterministic and stochastic environment. Since the deterministic processes cannot be ergodic, multiple trajectories are needed to learn the value function. In our simulation studies, we generate $L$ independent trajectories, and $M$ data points are collected every $\eta$ unit of time in each trajectory. The errors of the approximated solution to the true value function with respect to the number of trajectories $L$ are plotted in Figure \ref{fig:data}, where the environments are the same as those in Figure \ref{fig:exact_dt} and we set $M= 4$. 

\begin{figure}[ht!]
\centering
\begin{tabular}{ccc}
         \includegraphics[width=0.37\textwidth]{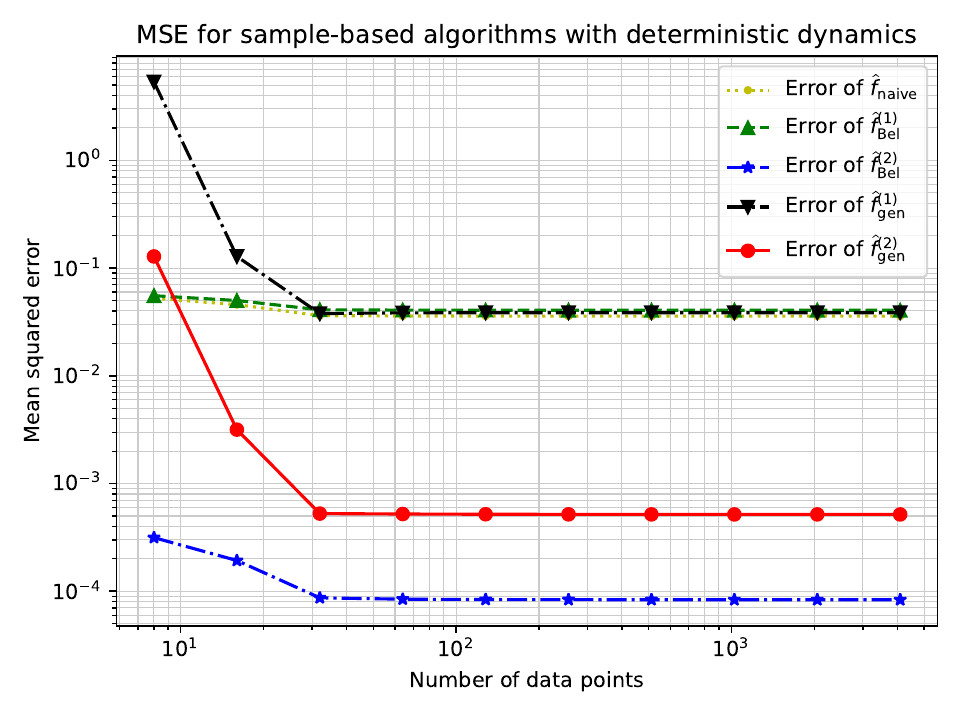}
&&
         \includegraphics[width=0.37\textwidth]{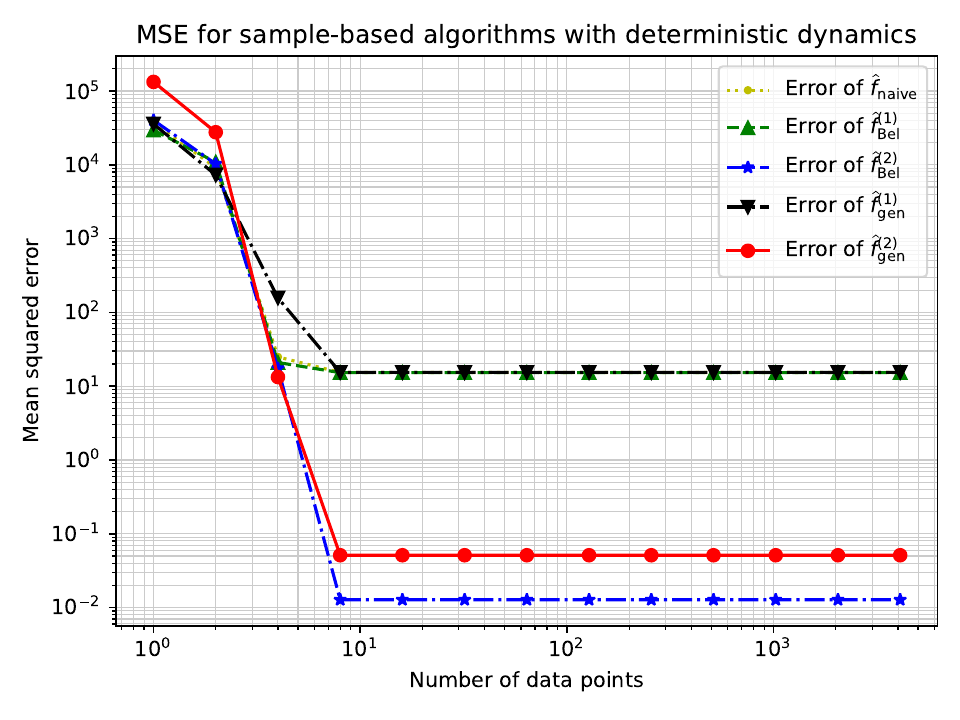}\\
(a) \( \lambda = 0.05, k = 1, \discount = 0.1 \) && (b) \( \alpha = 5, b = 1, \discount = 0.1 \) \\ 
        \includegraphics[width=0.37\textwidth]{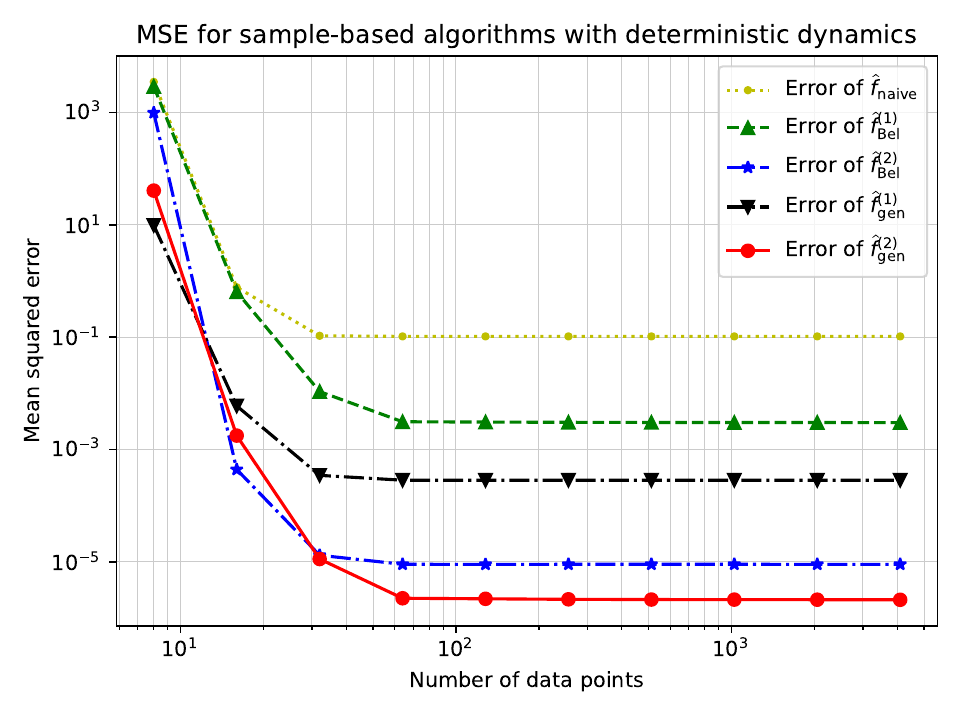}
&&
         \includegraphics[width=0.37\textwidth]{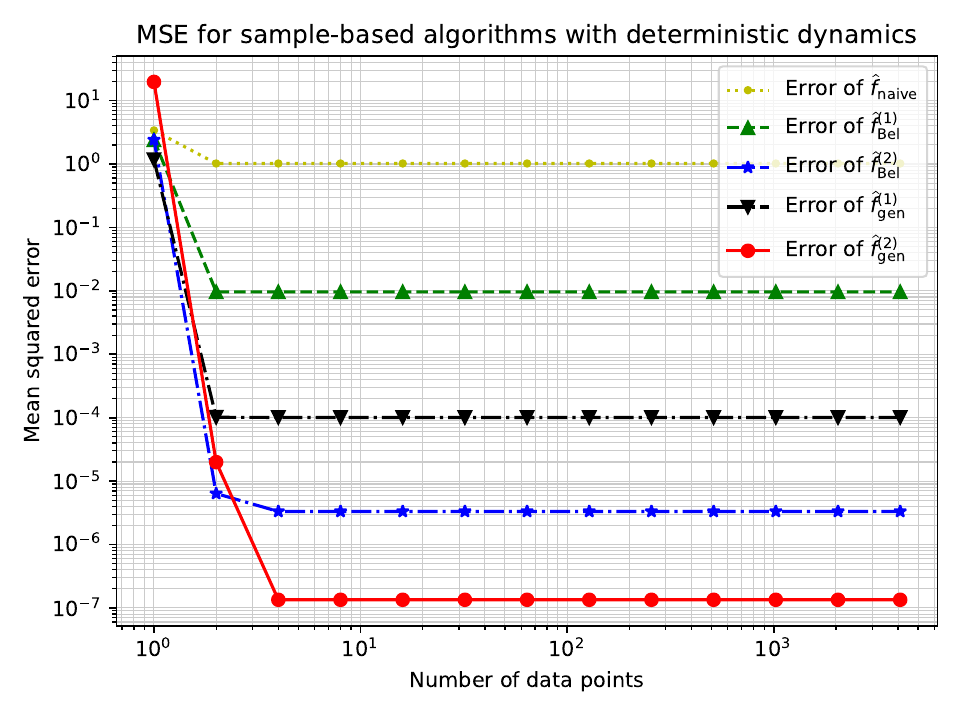}\\
(c)  \( \lambda = 0.01, k = 2, \discount = 2 \) && (d)  \( \alpha = 2, b = 2, \discount = 2 \)
\end{tabular}
\caption{The above figure plots the error of the approximated solution as the number of data increases. The specific parameter choices are marked in sub-figure titles.
The dynamics in panels (a)(c) follow Eq~\eqref{deter-dyna}, and the reward is Eq~\eqref{deter-reward-1}.
The dynamics in panels (b)(d) follow Eq~\eqref{deter-dyna} with \( \lambda = 0.01 \), and the reward is Eq~\eqref{deter-reward-2}.}
\label{fig:data}
\end{figure}

For stochastic dynamics, we use the same Ornstein--Uhlenbeck process~\eqref{stoch-dy} as in Section~\ref{subsec:exact-eg}, and learn the value functions based on a single stationary trajectory, collecting data every $\eta$ unit of time until the trajectory stops at a large time $T$.  The errors with respect to the terminal time $T$ are plotted in Figure \ref{fig:simulation-stochastic-time}, while the errors with respect to the step size $\eta$ are plotted in Figure \ref{fig:simulation-stochastic-stepsize}. In Figures \ref{fig:simulation-stochastic-time} and \ref{fig:simulation-stochastic-stepsize}, we test the algorithms in the stochastic dynamics \eqref{stoch-dy} with $\sigma = 1$, and the reward is set to be  
\begin{align*}
  \reward (\state) = \Big\{ 1 + \frac{1}{10} \state \cos (\state) + \frac{1}{2} \big(\sin (\state) - \cos^2 (\state)  \big) \Big\} \cdot \exp \Big( \sin (\state) \Big),
\end{align*}
so that the value function is exactly $\valuestar (\state) = \exp \big(\sin (\state) \big)$. In Appendix~\ref{app:sec-additional-simulation}, we present additional results with a quadratic reward function.

\begin{figure}[ht!]
\centering
\begin{tabular}{ccc}
  \widgraph{0.37\textwidth}{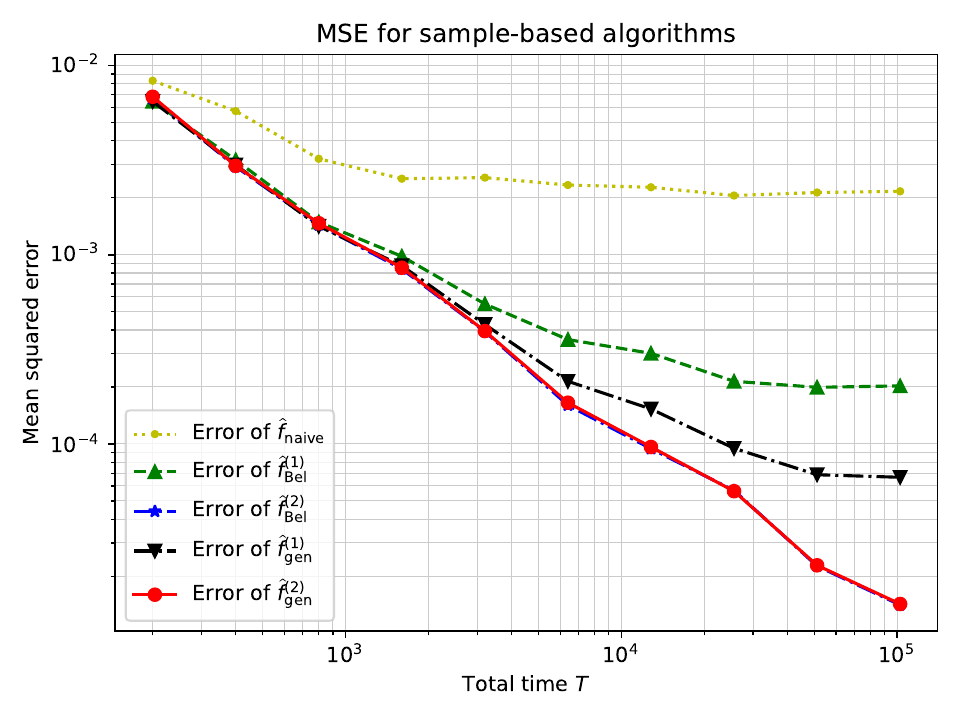} &&
  \widgraph{0.37\textwidth}{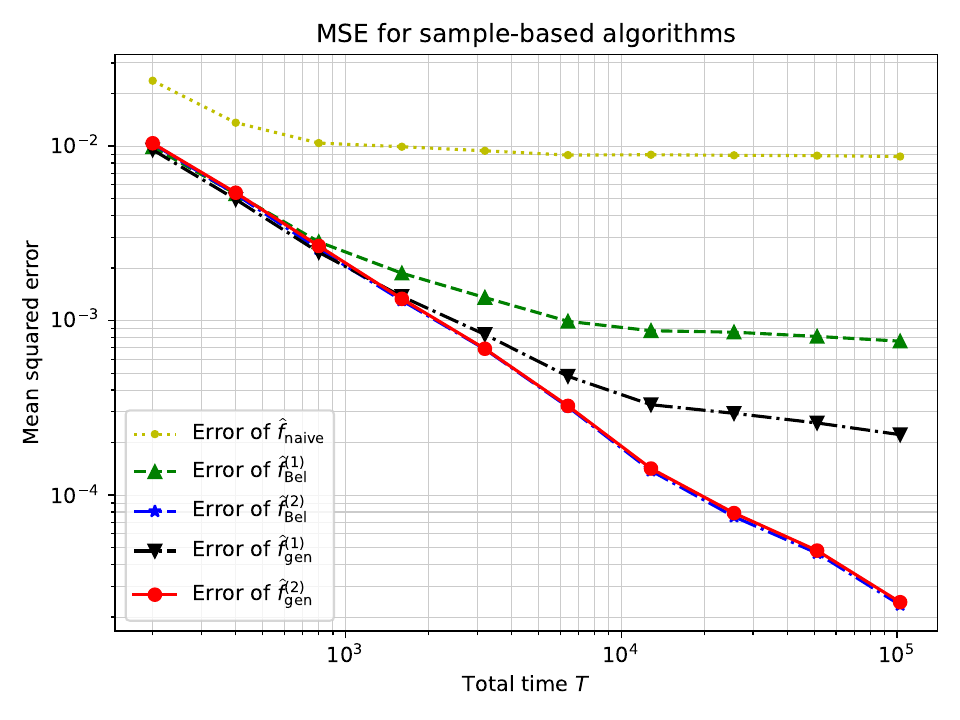} \\ (a) $\stepsize = 0.05$ && (b)  $\stepsize = 0.1$ \\
  \widgraph{0.37\textwidth}{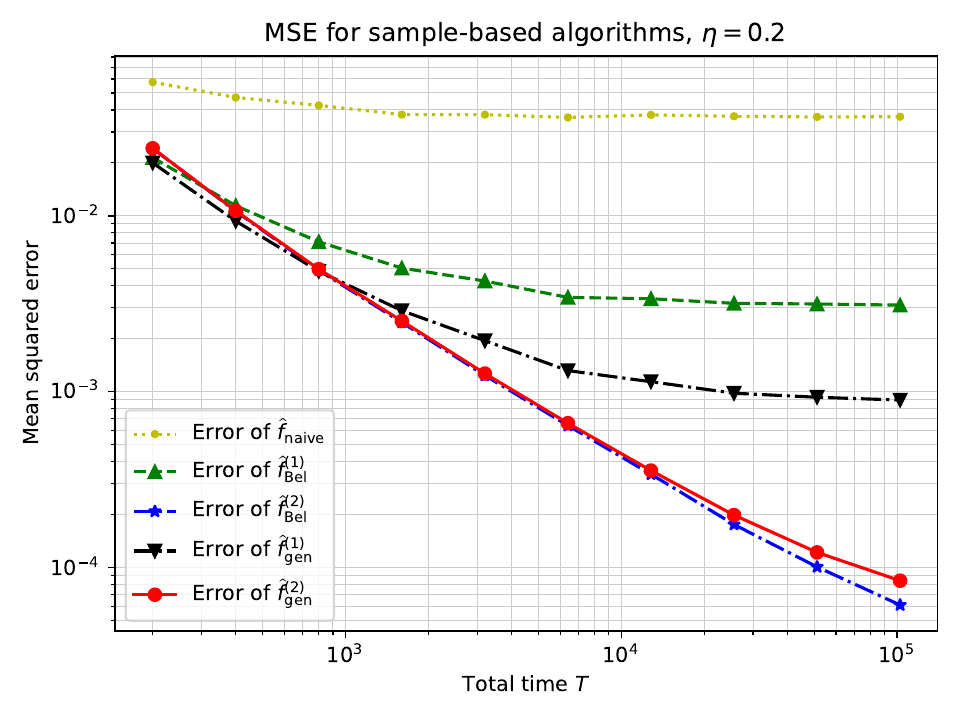} &&
  \widgraph{0.37\textwidth}{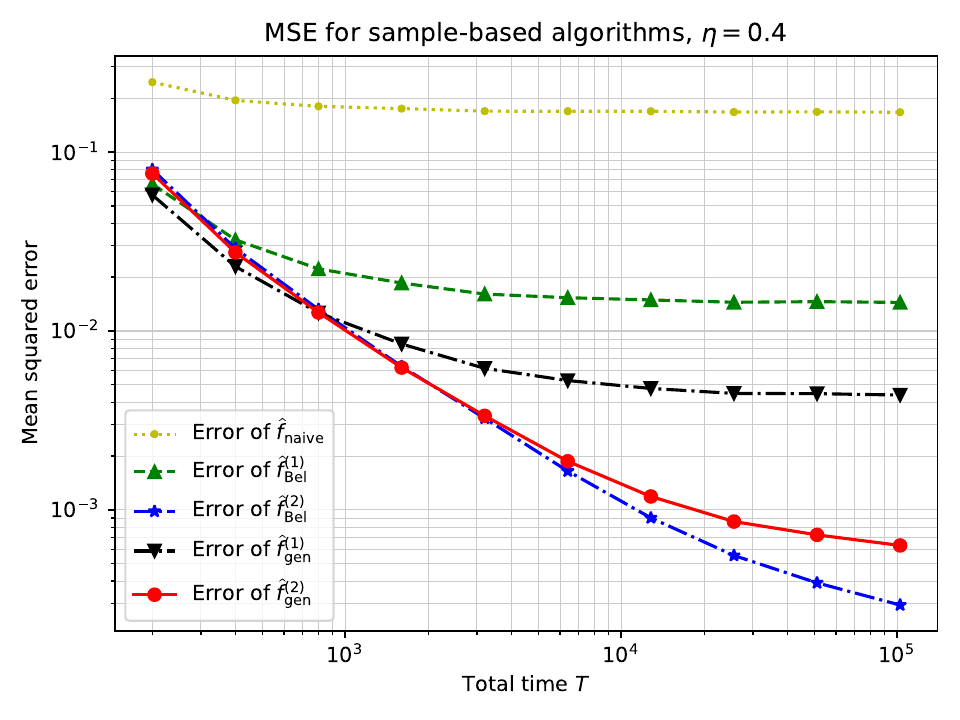} \\
(c)  $\stepsize = 0.2$ && (d)  $\stepsize = 0.4$
  \\
\end{tabular}
\caption{Plots of the mean-squared error $\Exs \big[ \statnorm{\valuehat - \valuestar}^2 \big]$ versus trajectory length
  $T$.  Each curve corresponds to a different algorithm. Each marker corresponds to a Monte Carlo
  estimate based on the empirical average of $50$ independent
  runs. As indicated by the sub-figure titles, each panel corresponds
  to a fixed stepsize $\stepsize$. Both
  axes in the plots are given by logarithmic scales.}
\label{fig:simulation-stochastic-time}
\end{figure}

\begin{figure}[ht!]
\centering
\begin{tabular}{ccc}
  \widgraph{0.37\textwidth}{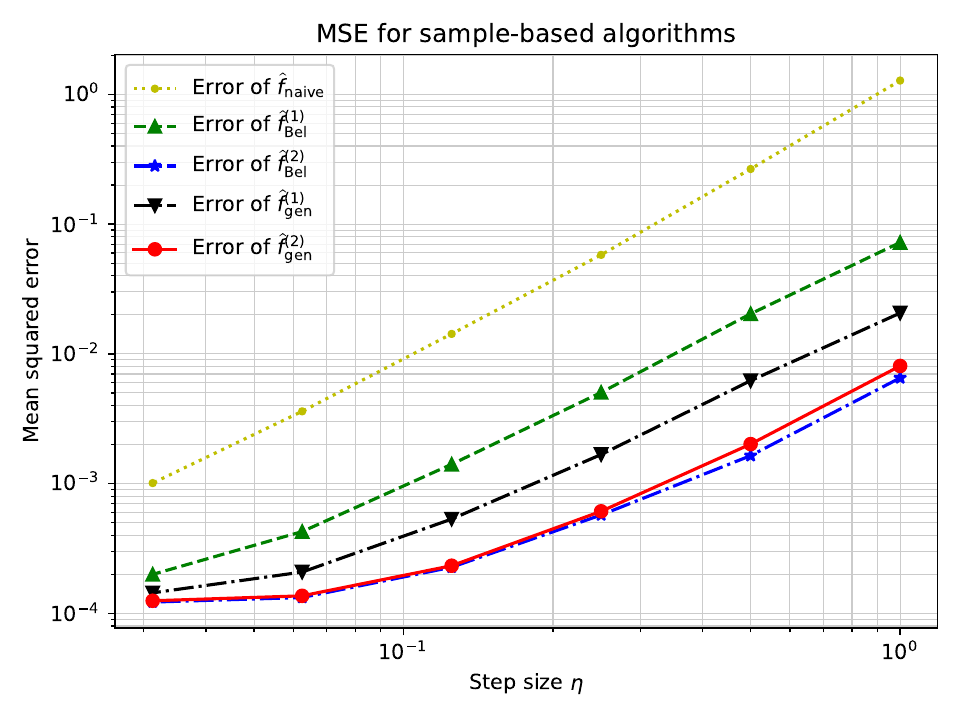} &&
  \widgraph{0.37\textwidth}{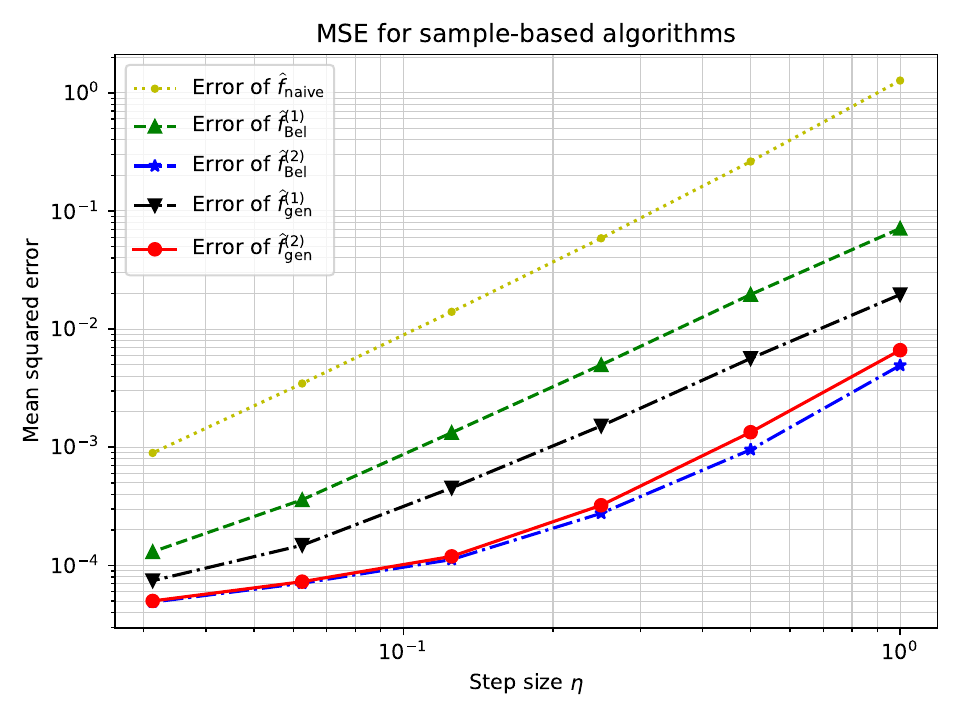} \\ (a) $T = 10000$ && (b) $T = 20000$ \\
  \widgraph{0.37\textwidth}{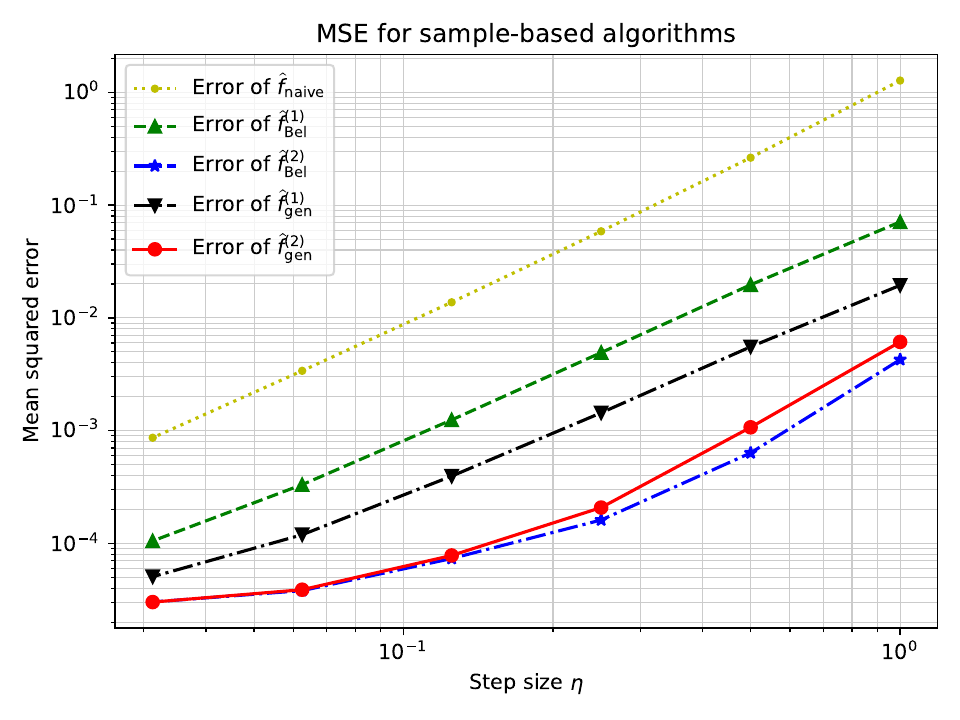} &&
  \widgraph{0.37\textwidth}{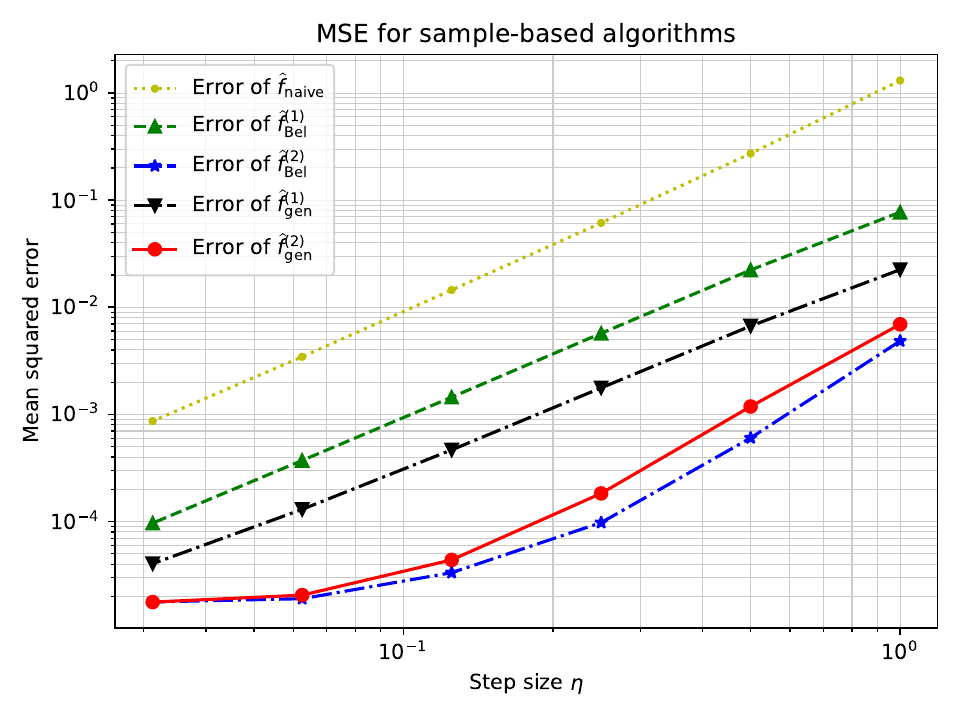} \\
(c) $T = 40000$ && (d) $T = 80000$
  \\
\end{tabular}
\caption{Plots of the mean-squared error $\Exs \big[ \statnorm{\valuehat - \valuestar}^2 \big]$ versus stepsize $\stepsize$.  Each curve corresponds to a different algorithm. Each marker corresponds to a Monte Carlo
  estimate based on the empirical average of $50$ independent
  runs. As indicated by the sub-figure titles, each panel corresponds
  to a fixed total time $T$. Both
  axes in the plots are given by logarithmic scales.}
\label{fig:simulation-stochastic-stepsize}
\end{figure}

From Figures \ref{fig:data} and \ref{fig:simulation-stochastic-time}, one can observe that as the amount of data increases, the error of the approximated solutions $ \ValHat^{(\numerOrder)}, \ValHat^{(\numerOrder, \mathrm{gen})} $ stabilizes at the error level of the corresponding exact solutions $ \ValInter{\numerOrder}, \valhatinter{\numerOrder} $. This indicates that if the exact solution is not a good approximation, no amount of data will enable convergence to the true value function. Additionally, the performance of the algorithms is robust across different types of data collection methods. Specifically, the error for the deterministic dynamics can stabilize with a small number of data points. Furthermore, the higher-order formulation does not introduce any additional instability into the results. 

From Figure \ref{fig:simulation-stochastic-stepsize}, one observes that as the terminal time $T$ increases, i.e., as the amount of data increases, the convergence rate as $\eta\to 0$ gets closer to its actual order. In Figures~\ref{fig:simulation-stochastic-stepsize} (b) and (c), we observe the first-order decay of the approximate solutions $\hat{\valuefunc}_{\mathrm{naive}}, \ValHat^{(1)}_{\mathrm{Bel}}, \ValHat^{(1)}_{\mathrm{gen}}$. For the second-order approximation solutions $\ValHat^{(2)}_{\mathrm{Bel}}, \ValHat^{(2)}_{\mathrm{gen}}$, the decay rate becomes flatter as $\eta\to 0$ because the error is dominated by the statistical error from the data in this regime.


Here are the details of the experiments.  For Figure \ref{fig:data},  we use the same bases as in \eqref{numerical-bases}. The plotted errors are mean errors over $100$ different simulations. In each simulations, we find the approximated solution using~\eqref{eq:finite-sample-estimator-bellman} and~\eqref{eq:finite-sample-estimator-diffusion} with trajectory data $(\MyState^l_0, \cdot, \MyState^l_{m\dt})_{l=1}^L$, where the initial states of each trajectory $\MyState^l_0 \sim$Unif$[-\pi,\pi]$ and the data are collected every $\eta = 0.1$ unit of time. The errors are evaluated in the same way as in Section \ref{subsec:exact-eg}. The $x$-axis plots the number of trajectories $L$.

For Figures \ref{fig:simulation-stochastic-time},~\ref{fig:simulation-stochastic-stepsize}, we use the same bases as in the left column of \eqref{numerical-bases}. We perform the simulation studies under two different scaling regimes: increasing the total time $T$ with a fixed stepsize $\stepsize$; and decreasing the stepsize $\stepsize$ with a fixed total time $T$. In order to evaluate the integrated squared error $\statnorm{\valuehat - \valuestar}^2$, we generate $N = 10^4$ independent test points $Z_1, Z_2, \cdots, Z_N$ from the stationary distribution $\stationary = \mathcal{N}(0, 5)$, and use the empirical average $\frac{1}{N} \sum_{i = 1}^N  \big(\valuehat (Z_i) - \valuestar (Z_i) \big)^2$ to approximate the error $\statnorm{\valuehat - \valuestar}^2$.

\section{Proofs}\label{sec:proofs}
We collect the proofs of the main results in this section.
\subsection{Proof of Theorem~\ref{thm:discretization-high-order}}\label{subsec:proof-thm-high-order}
By definition, we have that
\begin{align*}  
    \ValTrue (\state) &=  \int_0^{(\numerOrder - 1) \stepsize} e^{- \discount s} \Exs \big[ \reward (\MyState_s ) \mid \MyState_0 = \state \big] ds + e^{- \discount (\numerOrder - 1) \stepsize} \cdot\semigroup_{\stepsize (\numerOrder - 1)} \ValTrue (\state)\\
    \ValInter{\numerOrder} &=  \int_0^{(\numerOrder - 1) \stepsize} e^{- \discount s} \sum_{i = 0}^{\numerOrder - 1} W_i (s) \Exs \big[ \reward (\MyState_{i \stepsize}) \mid \MyState_0 = \state \big] ds + e^{- \discount (\numerOrder - 1) \stepsize}  \cdot \semigroup_{\stepsize (\numerOrder - 1)} \ValInter{\numerOrder} (\state)
\end{align*}
Taking their difference, note that
\begin{align*}
    \abss{\semigroup_{\stepsize (\numerOrder - 1)} \ValInter{\numerOrder} (\state) - \semigroup_{\stepsize (\numerOrder - 1)} \ValTrue (\state)} \leq \Exs \big[ |\ValInter{\numerOrder} (\MyState_{\stepsize (\numerOrder - 1)} ) - \ValTrue (\MyState_{\stepsize (\numerOrder - 1)} )| \mid \MyState_0 = \state \big] \leq \vecnorm{\valuestar - \ValInter{\numerOrder}}{\infty}.
\end{align*}
We therefore have the bound 
\begin{align}
  \vecnorm{\valuestar - \ValInter{\numerOrder}}{\infty} \leq \stepsize \numerOrder\max_{0 \leq s \leq \stepsize \numerOrder} \sup_{\state \in \StateSpace} \abss{ \sum_{i = 0}^{\numerOrder - 1} W_i (s) \semigroup_{i \stepsize} \reward (\state) - \semigroup_s \reward (\state)} + e^{- \discount \stepsize (\numerOrder - 1)} \vecnorm{\valuestar - \ValInter{\numerOrder}}{\infty}.\label{eq:linfty-contraction-in-high-order-time-discrete-pop-proof}
\end{align}
In other words, the error of interests $\vecnorm{\valuestar - \ValInter{\numerOrder}}{\infty}$ can be controlled using the error of Lagrangian polynomial approximation to the conditional expectation of the reward function. Define $g_\state (s) \mydefn \semigroup_s \reward (\state)$. The following lemma establishes its high-order differentiability properties.
\begin{lemma}\label{lemma:semigroup-high-order-bound}
    Under Assumption~\ref{assume:smooth-high-order}, there exists a constant $\constScary_\numerOrder > 0$, such that
    \begin{align*}
       \abss{\partial_s^{\numerOrder} g_\state (s)} \leq \constScary_\numerOrder, \quad \mbox{for any $\state \in \StateSpace$ and $s \geq 0$}.
    \end{align*}
\end{lemma}
\noindent See Section~\ref{subsubsec:proof-lemma-semigroup-high-order-bound} for the proof of this lemma.
Taking this lemma as given, we now proceed with the proof of Theorem~\ref{thm:discretization-high-order}.

By error bounds for Lagrangian polynomials~\cite{king1999afternotes}, we have
\begin{align*}
   \sup_{0 \leq s \leq \stepsize \numerOrder}  \abss{ \sum_{i = 0}^{\numerOrder - 1} W_i (s) g_\state (i \stepsize) - g_\state (s)} \leq \frac{1}{\numerOrder!} \sup_{0 \leq s \leq \stepsize \numerOrder} |\partial^\numerOrder g_\state (s)| \cdot \prod_{i = 1}^\numerOrder (i \stepsize) \leq \constScary_\numerOrder \stepsize^\numerOrder.
\end{align*}
Substituting back to Eq~\eqref{eq:linfty-contraction-in-high-order-time-discrete-pop-proof} completes the proof of the claim.

\subsubsection{Proof of Lemma~\ref{lemma:semigroup-high-order-bound}}\label{subsubsec:proof-lemma-semigroup-high-order-bound}
For any integer pair $i \geq 2$ and $h \geq 0$, we define the following function class
\begin{align*}
  \mathscr{D}_{i, h} \mydefn \Big\{
  \mbox{$h$-th order polynomials of $\big\{\partial^\alpha \reward \}_{|\alpha| \leq i}$, $\big\{ \partial^\alpha \drift_j \big\}_{|\alpha|\leq i - 2,~ j \in [\usedim]}$} 
  \mbox{ and }\big\{ \partial^\alpha \covMat_{j, k}\big\}_{|\alpha| \leq i - 2,~ j, k \in [\usedim]} \Big\}.
\end{align*}
Furthermore, we use the notation $\semigroup_s \mathscr{D}_{i, h}$ to denote the function class $\{\semigroup_s f: f \in \mathscr{D}_{i, h} \}$.

In order to prove the lemma, we use the following claim for $i = 1,2, \cdots, \numerOrder$.
\begin{align}
   \mbox{There exists } D_i \in \mathscr{D}_{2i, i + 1}, \quad \mbox{s.t.}~ \partial_s^i g_\state (s) = \semigroup_s D_i (\state), \quad \mbox{for any $s \geq 0$.} \label{eq:polynomial-induction-proof-lemma-high-order}
 \end{align}
Note that Lemma~\ref{lemma:semigroup-high-order-bound} immediately follows from Eq~\eqref{eq:polynomial-induction-proof-lemma-high-order}. In particular, suppose that $\partial_s^i g_\state (s) = \semigroup_s D_i$ for some $D_i \in \mathscr{D}_{2i, i + 1}$, we have
\begin{align*}
  \abss{\partial_s^i g_\state (s)} = \abss{\Exs \big[ D_i (\MyState_s) \mid \MyState_0 = \state \big]} \leq \sup_{\state' \in \StateSpace} |D_i (\state')|,
\end{align*}
which is bounded under Assumption~\ref{assume:smooth-high-order}.

It remains to prove Eq~\eqref{eq:polynomial-induction-proof-lemma-high-order}. We do so by induction. For the base case $i = 1$, we have that
\begin{align*}
  g_\state' (s) = \semigroup_s \generator g_\state (s) =  \semigroup_s \Big\{ \inprod{\drift }{\nabla \reward } + \frac{1}{2} \mathrm{Tr} \big( \covMat \cdot \nabla^2 \reward \big) \Big\} (\state),
\end{align*}
which takes form as a second-order polynomial of $\drift, \covMat, \nabla \reward$ and $\nabla^2 \reward$. Defining the function
\begin{align*}
  D_1 \mydefn \inprod{\drift }{\nabla \reward } + \frac{1}{2} \mathrm{Tr} \big( \covMat \cdot \nabla^2 \reward \big) \in \mathscr{D}_{2, 2},
\end{align*}
we conclude that Eq~\eqref{eq:polynomial-induction-proof-lemma-high-order} holds true for the base case.

Assume that Eq~\eqref{eq:polynomial-induction-proof-lemma-high-order} holds true for $i$-th derivative. There exists a function $D_i \in \mathscr{D}_{2i, i + 1}$ such that $\partial_s^i g_\state (s) = \semigroup_s D_i$ for any $s \geq 0$, we proceed with the proof for the $(i + 1)$-th derivative. Suppose that
\begin{align*}
  D_i = \sum_{\ell = 1}^{k_i} c_{\ell} M_{ \ell},
\end{align*}
where $\{c_\ell\}_{\ell = 1}^{k_i}$ are constants depending on $(i, d)$, and each term $M_{\ell}$ takes the following form
\begin{subequations}\label{eq:derivative-of-monomials-in-high-order-proof}
\begin{align}
  M_{\ell}  = \prod_{|\alpha| \leq 2 i} \big(\partial^\alpha \reward \big)^{q_{\alpha, \reward}} \cdot  \prod_{|\alpha| \leq 2 i - 2}~ \prod_{ j \in [\usedim]} \big(\partial^\alpha \drift_j \big)^{q_{\alpha, \drift_j}} \cdot  \prod_{|\alpha| \leq 2 i - 2}~ \prod_{ j, k \in [\usedim]} \big(\partial^\alpha \covMat_{j, k} \big)^{q_{\alpha, \covMat_{j, k}}},
\end{align}
for non-negative integer exponents $\{q_{\alpha, \reward}\}, ~\{q_{\alpha, \drift_j} \}$, and $\{q_{\alpha, \covMat_{j, k}}\}$ that sum up to at most $i + 1$.

Taking its derivatives, for each $u \in [\usedim]$, we have that
\begin{multline}
  \partial_{x_u} M_{i, \ell} = \sum_{|\alpha| \leq 2 i} q_{\alpha, \reward} \frac{M_{i, \ell}}{{\partial^\alpha \reward}} \partial^{\alpha + \coordinate_u} \reward  +  \sum_{|\alpha| \leq 2 i - 2}~ \sum_{ j \in [\usedim]} q_{\alpha, \drift_j}  \frac{M_{i, \ell}}{\partial^\alpha \drift_j}  \partial^{\alpha + \coordinate_u} \drift_j\\
   +  \sum_{|\alpha| \leq 2 i - 2}~ \sum_{ j, k \in [\usedim]} q_{\alpha, \covMat_{j, k}}  \frac{M_{i, \ell}}{\partial^\alpha \covMat_{j, k}} \partial^{\alpha + \coordinate_u} \covMat_{j, k}.\label{eq:derivative-calc-in-high-order-proof}
\end{multline}
\end{subequations}
Each term on the right-hand-side of Eq~\eqref{eq:derivative-calc-in-high-order-proof} is a monomial of $\big\{\partial^\alpha \reward \}_{|\alpha| \leq 2i + 1}$, $\big\{ \partial^\alpha \drift_j \big\}_{|\alpha|\leq 2i - 1,~ j \in [\usedim]}$, and $\big\{ \partial^\alpha \covMat_{j, k}\big\}_{|\alpha| \leq 2i - 1,~ j, k \in [\usedim]}$, with a degree at most $i + 1$, and coefficient at most $i + 1$. Therefore, we have that
\begin{align*}
  \partial_{x_u} M_{i, \ell} \in \mathscr{D}_{2i + 1, i + 1}.
\end{align*}

Once again, we can apply the arguments in Eqs~\eqref{eq:derivative-of-monomials-in-high-order-proof} to each monomial terms in the function $\partial_{x_u} M_{i, \ell}$, and we can conclude that
\begin{align*}
  \partial_{x_u}\partial_{x_v} M_{i, \ell} \in \mathscr{D}_{2i + 2, i + 1},
\end{align*}
for any $u, v\in[\usedim]$. Furthermore, by adding the $k_i$ terms together, we obtain that 
\begin{align}
  \partial_{x_u} D_i \in \mathscr{D}_{2i + 1, i + 1}\quad \mbox{and} \quad\partial_{x_u} \partial_{x_v} D_i \in \mathscr{D}_{2i + 2, i + 1}.\label{eq:partial-derivatives-in-next-function-class-high-order-proof}
\end{align}

Note that the $(i + 1)$-th order derivative of the function $g_\state (s)$ takes the form
\begin{align*}
  \partial_s^{i + 1} g_\state (s) = \semigroup_s \generator D_i (\state)=  \semigroup_s \Big[ \sum_{u \in [\usedim]} \drift_u \partial_{x_u} D_i + \frac{1}{2} \sum_{u, v \in [\usedim]} \covMat_{u, v} \partial_{x_u} \partial_{x_v} D_i \Big] (\state).
\end{align*}
Define the function
\begin{align*}
  D_{i + 1} \mydefn \sum_{u \in [\usedim]} \drift_u \partial_{x_u} D_i + \frac{1}{2} \sum_{u, v \in [\usedim]} \covMat_{u, v} \partial_{x_u} \partial_{x_v} D_i.
\end{align*}
By Eq~\eqref{eq:partial-derivatives-in-next-function-class-high-order-proof}, we have that $D_{i + 1} \in \mathscr{D}_{2i + 2, i + 2}$, which completes the induction proof.

\subsubsection{Details for second-order Bellman operators}\label{subsubsec:details-second-order-bellman}
Recall that $(\semigroup_s: s \geq 0)$ is the semigroup associated to the diffusion process~\eqref{eq:cts-time-process}. Taking the time derivative yields its generator
\begin{align*}
  \frac{d}{ds} g_\state (s) = \semigroup_s \generator \reward (\state) = \semigroup_s \Big(\inprod{\drift }{\nabla \reward } + \frac{1}{2} \mathrm{Tr} \big( \covMat \cdot \nabla^2 \reward \big) \Big) (\state).
\end{align*}
Taking another derivative, we obtain that
\begin{align*}
  &\frac{d^2}{ds^2} g_\state (s) =  \semigroup_s \generator \Big(\inprod{\drift }{\nabla \reward } + \frac{1}{2} \mathrm{Tr} \big( \covMat \cdot \nabla^2 \reward \big) \Big) (\state)\\
  &= \sum_{j, k}\semigroup_s \Big\{ \drift_j \cdot \partial_{j}  \drift_k \cdot \partial_k \reward + \drift_j \drift_k \cdot  \partial_{j} \partial_k  \reward \Big\} (\state) + \frac{1}{2} \sum_{j, k, \ell} \semigroup_s \Big\{ \drift_j \cdot \partial_j \covMat_{k, \ell} \cdot \partial_k \partial_\ell \reward + \drift_j \cdot \partial_j \partial_k \partial_\ell \reward \cdot \covMat_{k, \ell} \Big\} (\state)\\
  & \qquad + \frac{1}{2} \sum_{j, k, \ell} \semigroup_s \Big\{\covMat_{j, k} \cdot \partial_j \partial_k \drift_\ell \cdot \partial_\ell \reward + 2 \covMat_{j, k} \cdot \partial_j \drift_k \cdot \partial_k \partial_\ell \reward + \covMat_{j, k} \cdot \partial_j \partial_k \partial_\ell \reward \cdot \drift_\ell \Big\} (\state)\\
  &\qquad \qquad + \frac{1}{4} \sum_{j, k, \ell, i} \semigroup_s \Big\{ \covMat_{i, j} \cdot \partial_i \partial_j \covMat_{k, \ell} \cdot \partial_k \partial_\ell \reward + 2 \covMat_{i, j} \cdot \partial_j \covMat_{k, \ell} \cdot \partial_i \partial_k \partial_\ell \reward + \covMat_{i, j} \covMat_{k, \ell} \cdot \partial_i \partial_j \partial_k \partial_\ell \reward \Big\} (\state).
\end{align*}
Define the constant
\begin{multline}
  \constScary_2 \mydefn \Big\{ \smoothness_0^{(\drift)} \smoothness_1^{(\drift)} \smoothness_1^{(\reward)} + (\smoothness_0^{(\drift)})^2 \smoothness_2^{(\reward)} \Big\}  \usedim^2 + \Big\{ \smoothness_0^{(\covMat)} \smoothness_2^{(\covMat)} \smoothness_2^{(\reward)} + \smoothness_0^{(\covMat)} \smoothness_1^{(\covMat)} \smoothness_3^{(\reward)} + (\smoothness_0^{(\covMat)})^2 \smoothness_4^{(\reward)} \Big\} \usedim^4 \\
  + \Big\{ \smoothness_0^{(\drift)} \smoothness_1^{(\covMat)} \smoothness_2^{(\reward)} + \smoothness_0^{(\drift)} \smoothness_0^{(\covMat)} \smoothness_3^{(\reward)} + \smoothness_2^{(\drift)} \smoothness_0^{(\covMat)} \smoothness_1^{(\reward)} + \smoothness_1^{(\drift)} \smoothness_0^{(\covMat)} \smoothness_2^{(\reward)} \Big\} \usedim^3.\label{eq:scary-constant-defn}
\end{multline}
Invoking Assumption~\fakerefassumelip{(2)}, we have that
\begin{align*}
  \sup_{s \geq 0, \state \in \StateSpace} \abss{\frac{d^2}{ds^2} g_\state (s)} \leq \constScary_2,
\end{align*}

\subsection{Proof of~\Cref{thm:diffusion-second-order}}\label{proof:diffusion-second-order}
Since $\valuestar, \valhatinter{\numerOrder}$ satisfies the following two equations, 
\[
\generator \valuestar  = \discount\valuestar + \reward, \quad \geninter{\numerOrder} \valhatinter{\numerOrder} = \discount\valhatinter{\numerOrder} + \reward
\]
Subtracting the second equality from the first one gives,
\[
\begin{aligned}
\generator\valuestar -\geninter{\numerOrder}\valhatinter{\numerOrder} =& \discount (\valuestar - \valhatinter{\numerOrder})\\
\geninter{\numerOrder}(\valuestar -\valhatinter{\numerOrder}) =& \discount (\valuestar - \valhatinter{\numerOrder}) + (\geninter{\numerOrder} - \generator )\valuestar
\end{aligned}
\]
The proof of the theorem relies on the following Lemma.
\begin{lemma}\label{lemma:gen-second-order}
    For any reward function $\reward(\state)$, for $\nu = 1,2$,  the solution $\valhatinter{\numerOrder}(\state)$ to 
    \[
    \geninter{\numerOrder} \valhatinter{\numerOrder}(\state) = \discount \valhatinter{\numerOrder}(\state) + \reward(\state),
    \]
    can be bounded by 
    \[
    \ll \valhatinter{\numerOrder} \rl_\infty \leq \constbd_\numerOrder \ll \reward \rl_\infty,
    \]
    with 
    \begin{equation}\label{defofconstbd}
        \constbd_1 = \frac{1}{\discount}, \quad \constbd_2 = \frac{12}{\discount} + 16\stepsize.
    \end{equation}
\end{lemma}
By the above Lemma \ref{lemma:gen-second-order}, one has
\[
\ll \valuestar - \valhatinter{\numerOrder} \rl_\infty \leq  \constbd_\numerOrder\ll (\geninter{\numerOrder} - \generator )\valuestar \rl_\infty.
\]
Furthermore, Theorem \ref{thm:generator-high-order} yields,
\[
\ll (\geninter{\numerOrder} - \generator )\valuestar\rl_\infty \leq \constgen_\numerOrder\stepsize^\numerOrder,
\]
which implies,
\[
\ll \valuestar -\valhatinter{\numerOrder} \rl_\infty  \leq\constdiff_\numerOrder\stepsize^\numerOrder,
\]
with
\begin{equation}\label{defofconstdiff}
    \constdiff_\numerOrder = \constbd_\numerOrder
    \constgen_\numerOrder,
\end{equation}
and here $\constbd_\numerOrder,\constdiff_\numerOrder$ are defined in \eqref{defofconstgen} and \eqref{defofconstbd}, respectively.

\subsubsection{Proof of Lemma \ref{lemma:gen-second-order}}
The proof of Lemma \ref{lemma:gen-second-order} relies on Lemma \ref{lemma:sum of r} and Corollary \ref{coro:1st} and Lemma \ref{lemma:2nd}. 
\begin{lemma}\label{lemma:sum of r}
For $\valhatinter{\numerOrder}(\state)$ satisfying $\geninter{\numerOrder}\valhatinter{\numerOrder}(\state) = \discount\valhatinter{\numerOrder}(\state) + \reward(\state)$, it also satisfies the following equality,
\[
\valhatinter{\numerOrder}(\state) = \E\l[\sum_{j = 0}^\infty \coefd{\numerOrder}_j \reward(\MyState_{j\stepsize})\coefc{\numerOrder}_0| \MyState_0 = \state\r],
\]
with
\[
\coefd{\numerOrder}_{-k} = 0, \ 1\leq k \leq \numerOrder-1, \quad \coefd{\numerOrder}_0 = 1, \quad \coefd{\numerOrder}_j = \sum_{k=1}^\numerOrder \coefc{\numerOrder}_kd_{j-k}, \ j\geq 1,
\]
where 
\begin{equation}\label{def of c}
\coefc{\numerOrder}_0 = \frac{\stepsize}{\discount \stepsize - \coef{\numerOrder}_0},\quad \coefc{\numerOrder}_k = \frac{\coef{\numerOrder}_k}{\discount \stepsize - \coef{\numerOrder}_0}, \quad 1\leq k \leq \numerOrder
\end{equation}
and $\coef{\numerOrder}$ is defined by \eqref{def of A b}
\end{lemma}

\begin{proof}

First note that $\geninter{\numerOrder}\valhatinter{\numerOrder}(\state) = \discount\valhatinter{\numerOrder}(\state) + \reward(\state)$ can be equivalently written as
\[
 \valhatinter{\numerOrder}(\state) = \reward(\state)\coefc{\numerOrder}_0 + \E\l[\sum_{j=1}^\numerOrder\coefc{\numerOrder}_j \valhatinter{\numerOrder}(\MyState_{j\stepsize}) |\MyState_0 = \state  \r]
\]
We first prove the following claim. Claim:  
\[
\begin{aligned}
 \valhatinter{\numerOrder}(\state) =&  \E\l[\sum_{j=0}^{k-1} \coefd{\numerOrder}_jr(\MyState_{j\stepsize})\coefc{\numerOrder}_0 + \coefd{\numerOrder}_{k}\valhatinter{\numerOrder}(\MyState_{k\stepsize}) +  \sum_{j=2}^{\numerOrder}\l(\sum_{l=j}^{\numerOrder}d_{k+j-l-1}c_l\r)\valhatinter{\numerOrder}(\MyState_{(k+(j-1))\stepsize}) |\MyState_0 = \state\r]
 \end{aligned}
\]
We prove the above claim by induction:
First note that the above equation is satisfied for $k=1$ by the definition of $\geninter{\numerOrder}$. Then assume that it is satisfied for $k$, we prove that it is also satisfied for $k+1$ 
\[
\begin{aligned}
 &\valhatinter{\numerOrder}(\state)\\
  =&  \E_\state \Big[\sum_{j=0}^{k-1} \coefd{\numerOrder}_j\reward(\MyState_{j\stepsize})\coefc{\numerOrder}_0 + \coefd{\numerOrder}_{k}\valhatinter{\numerOrder}(\MyState_{k\stepsize}) +  \sum_{j=2}^{\numerOrder} \big( \sum_{l=j}^{\numerOrder}d_{k+j-l-1}\coefc{\numerOrder}_l \big) \valhatinter{\numerOrder}(\MyState_{(k+(j-1))\stepsize})  \Big]\\
= & \E_\state \Big[\sum_{j=0}^{k-1} \coefd{\numerOrder}_j\reward(\MyState_{j\stepsize})\coefc{\numerOrder}_0 + \coefd{\numerOrder}_{k}\reward(\MyState_{k\stepsize})\coefc{\numerOrder}_0 + \coefd{\numerOrder}_{k}\sum_{j=1}^\numerOrder\coefc{\numerOrder}_j\valhatinter{\numerOrder}(\MyState_{(k+j)\dt})  +  \sum_{j=2}^{\numerOrder} \big(\sum_{l=j}^{\numerOrder}\coefd{\numerOrder}_{k+j-l-1}\coefc{\numerOrder}_l\big) \valhatinter{\numerOrder}(\MyState_{(k+(j-1))\dt}) \Big]\\
= & \E_\state \Big[\sum_{j=0}^{k-1} \coefd{\numerOrder}_j\reward(\MyState_{j\dt})\coefc{\numerOrder}_0 + d_{k}\reward(\MyState_{k\dt})\coefc{\numerOrder}_0 + \coefd{\numerOrder}_{k}\sum_{j=1}^{\numerOrder-1}\coefc{\numerOrder}_j\valhatinter{\numerOrder}(\MyState_{(k+j)\dt})  \Big]\\
 &\quad +  \E_\state \Big[ \sum_{j=1}^{\numerOrder-1}\l(\sum_{l=j+1}^{\numerOrder}\coefd{\numerOrder}_{k-l+j}\coefc{\numerOrder}_l\r)\valhatinter{\numerOrder}(\MyState_{(k+j)\dt})  + \coefd{\numerOrder}_{k}\coefc{\numerOrder}\valhatinter{\numerOrder}(\MyState_{(k+\numerOrder)\dt})   \Big]\\
= & \E_\state \Big[ \sum_{j=0}^{k} \coefd{\numerOrder}_j\reward(\MyState_{j\dt})\coefc{\numerOrder}_0 +  \sum_{j=1}^{\numerOrder-1}\l(\sum_{l=j+1}^{\numerOrder}\coefd{\numerOrder}_{k-l+j}\coefc{\numerOrder}_l + \coefd{\numerOrder}_{k}\coefc{\numerOrder}_j\r)\valhatinter{\numerOrder}(\MyState_{(k+j)\dt})  + \coefd{\numerOrder}_{k}\coefc{\numerOrder}\valhatinter{\numerOrder}(\MyState_{(k+\numerOrder)\dt})  \Big]\\
= & \E_\state \Big[\sum_{j=0}^{k} d_j\reward(\MyState_{j\dt})\coefc{\numerOrder}_0  + \coefd{\numerOrder}_{k+1}\valhatinter{\numerOrder}(\MyState_{(k+1)\dt}) + \sum_{j=2}^{\numerOrder}\l(\sum_{l=j}^{\numerOrder}\coefd{\numerOrder}_{k+j-l}\coefc{\numerOrder}_l\r)\valhatinter{\numerOrder}(\MyState_{(k+j)\dt}) \Big]
 \end{aligned}
\]
As the claim holds, by setting $k\to\infty$, one completes the proof of the Lemma.

\end{proof}

\begin{corollary}\label{coro:1st}
    For $\coefd{1}$ defined in Lemma \ref{lemma:sum of r}, one has $\coefd{1}_j = (\coefc{1}_1)^j= \frac{1}{(\discount\dt + 1)^j}$.
\end{corollary}
\begin{proof}
    By the definition of $\coefd{1},\coefc{1}$ in Lemma \ref{lemma:sum of r}, it is straghtforward to obtain the above corollary. 
\end{proof}
\begin{lemma}\label{lemma:2nd}
For $\coefd{2}$ defined in Lemma \ref{lemma:sum of r}, if $\discount\stepsize<1/3$, one has
    \[
    |\coefd{2}_j| \leq 12\l( \frac{1}{(1+\discount\dt)^j} + \frac1{2^j}\r).
    \]   
\end{lemma} 
\begin{proof}
        Let $x_j = (\coefd{2}_{j+1},\coefd{2}_j)^\top$, and \[
A = \begin{bmatrix}
\coefc{2}_1 & \coefc{2}_2 \\
1 & 0
\end{bmatrix}, \quad \coefc{2}_1 = \frac{4}{2\discount\dt+3}, \quad \coefc{2}_2 = -\frac{1}{2\discount\dt+3},
\]
then $x_j = A^j x_0$. Since the eigenvalues and eigenvectors of $A$ are
\[
\begin{aligned}
    &k_1 = \frac{\coefc{2}_1+\sqrt{(\coefc{2}_1)^2+4\coefc{2}_2}}{2} = \frac{1}{2-\sqrt{1-2\discount\dt}}, e_1 = (k_1,1)^\top ; \\
    &k_2 =  \frac{\coefc{2}_1-\sqrt{(\coefc{2}_1)^2+4\coefc{2}_2}}{2} = \frac{1}{2+\sqrt{1-2\discount\dt}}, e_2 = (k_2,1)^\top .
\end{aligned}
\]
Let $\Lambda = \text{diag}(k_1, k_2)$, and $P = [e_1,e_2], Q = P^{-1}$, one has
\[
x_j = P\Lambda^jQx_0, \quad \coefd{2}_j = k_1^jP_{21}(Qx_0)_1 + k_2^j P_{22}(Qx_0)_2.
\]
Since
\[
Q = \begin{bmatrix}
\frac{3+2\discount\dt}{2\sqrt{1-2\discount\dt}} &\frac12 - \frac{1}{\sqrt{1-2\discount\dt}} \\
-\frac{3+2\discount\dt}{2\sqrt{1-2\discount\dt}} &\frac12 + \frac{1}{\sqrt{1-2\discount\dt}}
\end{bmatrix},
\]
which yields
\[
|P_{21}(Qx_0)_1| = |Q_{11}| | \coefd{2}_1 | + |Q_{12}| |\coefd{2}_0| \leq 4(|\coefd{2}_1| + |\coefd{2}_0|) \leq 12,
\]
\[
|P_{22}(Qx_0)_2| = |Q_{21}| | \coefd{2}_1 | + |Q_{22}| |\coefd{2}_0| \leq 4(|\coefd{2}_1| + |\coefd{2}_0|) \leq 12,
\]
\[
k_1 \leq \frac{1}{1+\discount\dt}, \quad k_2 \leq \frac{1}{2}.
\]
Therefore,
\[
|\coefd{2}_j| \leq 12( |k_1|^j + |k_2|^j) \leq  12\l( \frac{1}{(1+\discount\dt)^j} + \frac1{2^j}\r) .
\]
\end{proof}

\paragraph{Proof of Lemma \ref{lemma:gen-second-order}}
By Lemma \ref{lemma:sum of r}, one has, 
\[
\ll \valhatinter{\numerOrder}(\state) \rl_\infty \leq \ll \reward \rl_\infty \coefc{\numerOrder}_0 \sum_{j=0}^\infty |\coefd{\numerOrder}_j|.
\]
By Corollary \ref{coro:1st}, one has,
\[
\ll \valhatinter{1}(\state) \rl_\infty \leq  \ll \reward \rl_\infty \frac{\dt}{\discount\dt+1} \sum_{j=0}^\infty \frac{1}{(1+\discount\dt)^j}  =  \constbd_1 \ll r \rl_\infty.
\]
By Corollary \ref{lemma:2nd}, one has,
\[
\ll \valhatinter{2}(\state) \rl_\infty \leq  12 \ll \reward \rl_\infty \frac{\dt}{\discount\dt + 3/2} \sum_{j=0}^\infty \l( \frac{1}{(1+\discount\dt)^j} + \frac1{2^j}\r) \leq  \constbd_2 \ll \reward \rl_\infty.
\]

\subsection{Proof of Corollary~\ref{cor:projected-discretized-bellman}}\label{subsec:proof-cor-projected}
Given a pair of functions $\ValFunc_1, \ValFunc_2 \in \ltwospace (\stationary)$, we note that
\begin{align*}
  \Big\{ \BellOp^{(\numerOrder)} (\ValFunc_1) -  \BellOp^{(\numerOrder)} (\ValFunc_2) \Big\} (\state) = e^{- \discount (\numerOrder - 1) \stepsize } \cdot \semigroup_{(\numerOrder - 1) \stepsize} (\ValFunc_1 - \ValueFunc_2) (\state).
\end{align*}
Note that the stationary distribution $\stationary$ is also stationary for the discrete-time Markov transition kernel $\semigroup_{(\numerOrder - 1) \stepsize}$, by Cauchy--Schwarz inequality, we note that
\begin{align}
  \statnorm{\semigroup_{(\numerOrder - 1) \stepsize} h}^2 \leq \int \big( \semigroup_{(\numerOrder - 1) \stepsize} h (x) \big)^2 d \stationary (x) \leq \int \semigroup_{(\numerOrder - 1) \stepsize} h^2 (x) d \stationary (x)  = \statnorm{h}^2,\label{eq:semigroup-is-non-expansive}
\end{align}
for any function $h \in \ltwospace (\stationary)$. Consequently, we have the contraction property
\begin{align}
  \statnorm{\BellOp^{(\numerOrder)} (\ValFunc_1) -  \BellOp^{(\numerOrder)} (\ValFunc_2)} \leq e^{- \discount (\numerOrder - 1) \stepsize } \statnorm{\ValFunc_1 - \ValFunc_2}.\label{eq:discretized-bellman-l2-contraction}
\end{align}
Note that the orthonormal projection operator $\projecttolin$ is non-expansive under $\ltwospace (\stationary)$, i.e.,
\begin{align*}
  \statnorm{\projecttolin (\ValueFunc_1) - \projecttolin (\ValueFunc_2)} \leq \statnorm{\ValFunc_1 - \ValFunc_2},
\end{align*}
for any pair $\ValFun_1, \ValFun_2 \in \ltwospace (\stationary)$. We conclude that the composite operator $\projecttolin \circ \BellOp^{(\numerOrder)}$ is a contraction operator with contraction factor $e^{- \discount (\numerOrder - 1) \stepsize }$ in the space $\ltwospace (\stationary)$. The existence and uniqueness of the fixed point $\valuebar_{\Bel}^{(\numerOrder)}$ follow from Banach fixed point theorem.

Note that $\ValInter{\numerOrder}$ solves the fixed-point equation~\eqref{eq:defn-high-order-bellman-operator}, due to existing literature~\cite{tsitsiklis1997analysis,mou2023optimal}, the contraction bound~\eqref{eq:discretized-bellman-l2-contraction} leads to a worst-case approximation factor upper bound.
\begin{align}
  \statnorm{\valuebar_{\Bel}^{(\numerOrder)} - \ValInter{\numerOrder}}^2 \leq \frac{1}{1 - e^{- 2 \discount (\numerOrder - 1) \stepsize }} \inf_{\ValFunc \in \LinSpace} \statnorm{\ValFunc - \ValInter{\numerOrder}}^2. \label{eq:projected-fixed-pt-approx-factor-worst-case-bound}
\end{align}
By triangle inequality, we have
\begin{multline*}
  \inf_{\ValFunc \in \LinSpace} \statnorm{\ValFunc - \ValInter{\numerOrder}} \leq \inf_{\ValFunc \in \LinSpace} \statnorm{\ValFunc - \ValTrue} + \statnorm{\ValInter{\numerOrder} - \ValTrue} \\
  \leq \inf_{\ValFunc \in \LinSpace} \statnorm{\ValFunc - \ValTrue} + \vecnorm{\ValInter{\numerOrder} - \ValTrue}{\infty} \leq \inf_{\ValFunc \in \LinSpace} \statnorm{\ValFunc - \ValTrue} + \discount^{-1} \constScary_\numerOrder \stepsize^\numerOrder,
\end{multline*}
where the last inequality follows from Theorem~\ref{thm:discretization-high-order}.

Similarly, we note that
\begin{multline*}
  \statnorm{\valuebar_{\Bel}^{(\numerOrder)} - \ValInter{\numerOrder}} \geq \statnorm{\valuebar_{\Bel}^{(\numerOrder)} - \ValTrue} - \statnorm{\ValTrue  - \ValInter{\numerOrder}}\\ \geq \statnorm{\valuebar_{\Bel}^{(\numerOrder)} - \ValTrue} - \vecnorm{\ValTrue  - \ValInter{\numerOrder}}{\infty} \geq \statnorm{\valuebar_{\Bel}^{(\numerOrder)} - \ValTrue} - \discount^{-1} \constScary_\numerOrder \stepsize^\numerOrder.
\end{multline*}
Substituting them back to Eq~\eqref{eq:projected-fixed-pt-approx-factor-worst-case-bound} and noting that $1 - e^{- 2 \discount (\numerOrder - 1) \stepsize } \geq \discount (\numerOrder - 1) \stepsize$ for $\stepsize \leq \tfrac{1}{2\discount \numerOrder}$, we conclude that
\begin{align*}
  \statnorm{\valuebar_{\Bel}^{(\numerOrder)} - \ValTrue} \leq \Big\{1 + \frac{1}{\sqrt{\discount (\numerOrder - 1) \stepsize}} \Big\} \discount^{-1} \constScary_\numerOrder \stepsize^\numerOrder + \frac{1}{\sqrt{\discount (\numerOrder - 1) \stepsize}} \inf_{\ValFunc \in \LinSpace} \statnorm{\ValFunc - \ValTrue},
\end{align*}
which proves the desired claim.

\subsection{Proof of Theorem~\ref{thm:improved-approx-factor-elliptic}}\label{subsec:proof-thm-improved-approx-factor-elliptic}
Note that the discretized Bellman equation takes the form
\begin{align*}
   \frac{1}{(\numerOrder - 1) \stepsize} \big( \IdMat - e^{- \discount (\numerOrder - 1) \stepsize} \semigroup_{(\numerOrder - 1) \stepsize} \big)  \ValInter{\numerOrder} = \frac{1}{(\numerOrder - 1) \stepsize} \sum_{i = 0}^{\numerOrder - 1}\int_0^{(\numerOrder - 1) \stepsize} e^{- \discount s} W_i (s) \semigroup_s \reward,
\end{align*}
while the projected version can be re-written as
\begin{align*}
  \frac{1}{(\numerOrder - 1) \stepsize} \projecttolin \Big\{ \big( \IdMat - e^{- \discount (\numerOrder - 1) \stepsize} \semigroup_{(\numerOrder - 1) \stepsize} \big) \valuebar_\Bel^{(\numerOrder)} \Big\}   = \frac{1}{(\numerOrder - 1) \stepsize} \sum_{i = 0}^{\numerOrder - 1}\int_0^{(\numerOrder - 1) \stepsize} e^{- \discount s} W_i (s) \projecttolin \big( \semigroup_s \reward \big).
\end{align*}
Comparing the two equations above, we conclude that
\begin{align}
  \statinprod{\frac{\IdMat - e^{- \discount (\numerOrder - 1) \stepsize} \semigroup_{(\numerOrder - 1) \stepsize} }{(\numerOrder - 1) \stepsize} \Big( \ValInter{\numerOrder}  - \valuebar^{(\numerOrder)}_{\Bel} \Big)}{h} = 0, \quad \mbox{for any $h \in \LinSpace$}.\label{eq:galerkin-orthogonality-condition-in-improved-projected-fixed-pt}
\end{align}

Following C\'{e}a's lemma~\cite{cea1964approximation}, the analysis of the projected fixed-point relies on coercivity and boundedness estimates for the linear operator $ \tfrac{1}{t} ( \IdMat - e^{- \discount t} \semigroup_{t} )$. The following two lemmas give the relevant upper and lower bounds.

\begin{lemma}\label{lemma:discrete-time-psd-operator}
  For any $t > 0$ and $\ValFunc \in \ltwospace (\stationary)$, we have the lower bound
  \begin{align*}
    \statinprod{\ValFunc}{t^{-1} \big(\IdMat - e^{- \discount t} \transition_t \big) \ValFunc} \geq \frac{\discount}{t} \int_0^t e^{- 2 \discount s} \statnorm{\semigroup_s \ValFunc}^2 ds + \frac{\lammin}{2t} \int_0^t e^{- 2 \discount s} \statnorm{\nabla \semigroup_s \ValFunc}^2 ds + \frac{1}{2t} \statnorm{\ValFunc - e^{- \discount t} \semigroup_t \ValFunc}^2.
  \end{align*}
\end{lemma}

\noindent See~\Cref{subsubsec:proof-lemma-discrete-time-psd-operator} for the proof of this lemma.

We define the inner product structure
\begin{align*}
  \soboavginprod{f}{g}{t} \mydefn \statinprod{f}{g} + \statinprod{\nabla f}{\nabla g} + \frac{1}{t} \int_0^t \statinprod{\nabla \semigroup_s f}{\nabla \semigroup_s g} ds,
\end{align*}
and the induced norm $\soboavgnorm{f}{t} \mydefn \sqrt{\soboavginprod{f}{f}{t}}$.

\begin{lemma}\label{lemma:discrete-time-operator-norm-bound}
  Under above setup, if the probability measure $\stationary$ (not necessarily the stationary measure) satisfies $\statnorm{\semigroup_t f} \leq e^{\beta_0 t} \statnorm{f}$ for any $t > 0$ and $f \in \ltwospace (\stationary)$
  For any $t > 0$ and $f, g \in \ltwospace (\stationary)$, we have the upper bound
  \begin{align*}
    \abss{\statinprod{g}{t^{-1} \big(\IdMat - e^{- \discount t} \transition_t \big) f}} \leq \big( \discount e^{\beta_0 t} + \lammax + \smoothness_\drift^{(0)} \sqrt{\usedim} + \lammax \smoothness_\stationary \usedim + \smoothness_\covMat^{(1)} \usedim^2 \big)\cdot \soboavgnorm{f}{t} \cdot \soboavgnorm{g}{t}
  \end{align*}
\end{lemma}

\noindent See~\Cref{subsubsec:proof-lemma-discrete-time-operator-norm-bound} for the proof of this lemma.

Note that the lower bound in Lemma~\ref{lemma:discrete-time-psd-operator} uses a weaker norm than the $\soboavgnorm{\cdot}{t}$-norm defined above. In order to establish the desired approximation ratio guarantee, we use the following additional result.
\begin{lemma}\label{lemma:reverse-smoothing-estiamte}
  Under the setup of Theorem~\ref{thm:improved-approx-factor-elliptic}, if the probability measure $\stationary$ (not necessarily the stationary measure) satisfies $\statnorm{\semigroup_t f} \leq e^{\beta_0 t} \statnorm{f}$, for any $t \leq \numerOrder \stepsize \wedge \tfrac{1}{\beta_0}$ and any $f \in \LinSpace$, we have
  \begin{align*}
    \frac{1}{t} \int_0^t e^{- 2 \discount s} \statnorm{\semigroup_s \ValFunc}^2 ds + \frac{1}{t} \int_0^t e^{- 2 \discount s} \statnorm{\nabla \semigroup_s \ValFunc}^2 ds + \frac{1}{t} \statnorm{\ValFunc - e^{- \discount t} \semigroup_t \ValFunc}^2 \geq \frac{\lammin}{2 e^2 \lammax} \soboavgnorm{f}{t}^2.
  \end{align*}
\end{lemma}
\noindent See~\Cref{subsubsec:proof-lemma-reverse-smoothing-estiamte} for the proof of this lemma.

Let $\stationary$ be the stationary distribution of the diffusion process~\eqref{eq:cts-time-process}, for any function $\ValFun \in \ltwospace (\stationary)$ and $t > 0$, we have
\begin{align*}
  \statnorm{\semigroup_t \ValFunc}^2 = \Exs_\stationary \Big[ \Big\{ \Exs \big[ \ValFunc (\MyState_{t}) \mid \MyState_0 \big] \Big\}^2  \Big] \leq \Exs_\stationary \big[ \ValFunc (\MyState_{t})^2 \big] = \statnorm{\ValFunc}^2,
\end{align*}
which shows that the semigroup is non-expansive under the $\statnorm{\cdot}$-norm. So we can apply Lemmas~\ref{lemma:discrete-time-operator-norm-bound} and~\ref{lemma:reverse-smoothing-estiamte} with $\beta_0 = 0$.

Taking these lemmas as given, we proceed with the proof of Theorem~\ref{thm:improved-approx-factor-elliptic}. Our proof relies on a generalization of C\'{e}a's lemma, which is established in our previous work~\cite{mou2023optimal}. We include the proof here for completeness. Note that Lemmas~\ref{lemma:discrete-time-psd-operator} and~\ref{lemma:reverse-smoothing-estiamte} imply that
\begin{align*}
   \statinprod{\ValFunc}{t^{-1} \big(\IdMat - e^{- \discount t} \transition_t \big) \ValFunc} \geq \frac{\min(\discount, 1, \lammin) \lammin}{e^2 \lammax} \soboavgnorm{\ValueFunc}{t}^2,\quad \mbox{for any $\ValueFunc \in \LinSpace$.}
\end{align*}
Define $\widetilde{\ValFunc} \mydefn \arg\min_{\ValFun \in \LinSpace} \soboavgnorm{\ValFunc - \ValInter{\numerOrder}}{\stepsize (\numerOrder - 1)}$. For simplicity, we introduce the notations $\ValueFunc^\circ \mydefn \ValInter{\numerOrder}$ and $\valuebar \mydefn \valuebar_{\Bel}^{(\numerOrder)}$. Furthermore, we define the operator $\mathcal{L} \mydefn ((\numerOrder - 1) \stepsize)^{-1} \big(\IdMat - e^{- \discount (\numerOrder - 1) \stepsize} \semigroup_{(\numerOrder - 1) \stepsize} \big)$. Under these notations, we have
\begin{multline}
  \frac{\min(\discount, 1, \lammin) \lammin}{e^2 \lammax} \soboavgnorm{\valuebar - \widetilde{\ValueFunc}}{(\numerOrder - 1)\stepsize}^2 
  \leq \statinprod{\valuebar - \widetilde{\ValueFunc}}{\mathcal{L} (\valuebar - \widetilde{\ValueFunc})}\\
  = \statinprod{\valuebar - \widetilde{\ValueFunc}}{\mathcal{L} (\valuebar - \ValueFunc^\circ)}  + \statinprod{\valuebar - \widetilde{\ValueFunc}}{\mathcal{L} (\ValueFunc^\circ - \widetilde{\ValueFunc})}.\label{eq:lower-bound-in-generalized-cea-proof}
\end{multline}
Now we bound the two terms on the right-hand-side of Eq~\eqref{eq:lower-bound-in-generalized-cea-proof}. By the orthogonality condition~\eqref{eq:galerkin-orthogonality-condition-in-improved-projected-fixed-pt}, since $\valuebar - \widetilde{\ValueFunc} \in \LinSpace$, we have
\begin{align*}
  \statinprod{\valuebar - \widetilde{\ValueFunc}}{\mathcal{L} (\valuebar - \ValueFunc^\circ)} = 0.
\end{align*}
For the second term, Lemma~\ref{lemma:discrete-time-operator-norm-bound} implies that
\begin{align*}
  \statinprod{\valuebar - \widetilde{\ValueFunc}}{\mathcal{L} (\ValueFunc^\circ - \widetilde{\ValueFunc})} \leq \big( \discount + \lammax + \smoothness_\drift^{(0)} \sqrt{\usedim} + \lammax \smoothness_\stationary \usedim + \smoothness_\covMat^{(1)} \usedim^2 \big)\cdot \soboavgnorm{\valuebar - \widetilde{\ValueFunc}}{ (\numerOrder - 1) \stepsize} \cdot \soboavgnorm{\ValueFunc^\circ - \widetilde{\ValueFunc}}{(\numerOrder - 1) \stepsize}.
\end{align*}
Substituting back to Eq~\eqref{eq:lower-bound-in-generalized-cea-proof}, we conclude that
\begin{align}
  \soboavgnorm{\ValueFunc^\circ - \valuebar}{(\numerOrder - 1) \stepsize} &\leq  \soboavgnorm{\ValueFunc^\circ - \widetilde{\ValueFunc}}{(\numerOrder - 1) \stepsize} +  \soboavgnorm{\valuebar - \widetilde{\ValueFunc}}{(\numerOrder - 1) \stepsize} \nonumber\\
  &\leq  \Big\{1 + \frac{ 2e^2 \big(\discount + \smoothness_\drift^{(0)} \sqrt{\usedim} + \lammax \smoothness_\stationary \usedim + \smoothness_\covMat^{(1)} \usedim^2 \big) \lammax }{\min(\discount, 1, \lammin) \lammin} \Big\} \soboavgnorm{\ValueFunc^\circ - \widetilde{\ValueFunc}}{(\numerOrder - 1) \stepsize}.\label{eq:cea-lemma-final-in-improved-approx-factor-proof}
\end{align}
Define the projection operator $\projectto{\LinSpace, \mathbb{H}^1} : \ValueFunc \mydefn \arg\min_{g \in \LinSpace} \soboavgnorm{f - g}{(\numerOrder - 1) \stepsize}$.
In order to relate the right-hand-side of Eq~\eqref{eq:cea-lemma-final-in-improved-approx-factor-proof} to the approximation of the true value function $\ValTrue$, we start with the inequality
\begin{align}
  &\soboavgnorm{\ValueFunc^\circ - \widetilde{\ValueFunc}}{(\numerOrder - 1) \stepsize} = \soboavgnorm{(\IdMat - \projectto{\LinSpace, \mathbb{H}^1}) \ValueFunc^\circ }{(\numerOrder - 1) \stepsize} \nonumber\\
  & \leq \soboavgnorm{(\IdMat - \projectto{\LinSpace, \mathbb{H}^1}) \ValTrue}{(\numerOrder - 1) \stepsize} + \soboavgnorm{(\IdMat - \projectto{\LinSpace, \mathbb{H}^1}) (\ValTrue - \ValueFunc^\circ)}{(\numerOrder - 1) \stepsize}\nonumber \\
  & \leq \inf_{g \in \LinSpace}\soboavgnorm{g - \ValTrue}{(\numerOrder - 1) \stepsize} + \soboavgnorm{\ValTrue - \ValInter{\numerOrder}}{(\numerOrder - 1) \stepsize}.\label{eq:relate-projection-error-to-valtrue}
\end{align}

The $\soboavgnorm{\cdot}{t}$-norm involves gradient norm bounds under the diffusion semigroups. In order to study them, we utilize the following two results from existing literature.
\begin{proposition}\label{prop:semigroup-grad-estimate-local-grow}[\cite{wang2005character}, Proposition 1.2]
  Under above setup, for any bounded Lipschitz function $u$ on $\StateSpace$ and $t > 0$, we have
  \begin{align*}
    \vecnorm{\nabla \semigroup_t u (x)}{2}^2 \leq e^{2 c t} \semigroup_t \Big( \vecnorm{\nabla u }{2}^2 \Big) (x), \quad \mbox{for any $x \in \StateSpace$},
  \end{align*}
  for any constant $c$ satisfying
  \begin{align*}
    c \geq \frac{1}{\lammin} \sup_{x \in \StateSpace} \Big\{ \sum_{i, j \in [\usedim]} \vecnorm{\nabla \covMat_{i,j} (x)}{2}^2 + \opnorm{\nabla \cdot \drift (x)}^2 \Big\}.
  \end{align*}
\end{proposition}

\begin{proposition}\label{prop:semigroup-grad-estimate-regularize}[\cite{picard2002gradient}, Theorem 2.1 and Example 2.1]
  Under above setup, there exists a constant $\constScaryReg$ depending only on $(\lammax, \lammin, \smoothness_\drift^{(0)}, \smoothness_\drift^{(1)}, \smoothness_\covMat^{(1)}, \usedim )$, such that for any bounded Lipschitz function $u$ on $\StateSpace$ and $t > 0$, we have
  \begin{align*}
    \vecnorm{\nabla \semigroup_t u (x)}{2} \leq \frac{\constScaryReg}{\sqrt{t} \wedge 1} \sqrt{\semigroup_t \big(u^2 \big) (x)}, \quad \mbox{for any $x \in \StateSpace$},
  \end{align*}
\end{proposition}
Note that both propositions provide point-wise bounds on the function $\nabla \semigroup_t u$, while they capture different phenomena. The estimate by~\cite{picard2002gradient} reveals regularization effects of the diffusion semigroup, but their bound blows up at $t = 0$. On the other hand,~\cite{wang2005character} utilizes the gradient at time $t = 0$ to obtain a bound that does not suffer from such blow-up. We will use both bounds in our analysis.

By Proposition~\ref{prop:semigroup-grad-estimate-local-grow}, for stepsize satisfying $\stepsize \numerOrder \leq \frac{\lammin}{\smoothness_\covMat^{(1)} \usedim^3 + \smoothness_\drift^{(1)} \usedim^2}$, we have $\vecnorm{\nabla \semigroup_t f (\state)}{2}^2 \leq e \semigroup_t (f^2) (\state)$ for any $t \in [0, (\numerOrder - 1) \stepsize]$, $x \in \StateSpace$, and bounded Lipschitz function $f$. Integration in time and space yields
\begin{align*}
  \frac{1}{(\numerOrder - 1) \stepsize} \int_0^{(\numerOrder - 1) \stepsize} \statnorm{\nabla \semigroup_t f}^2 dt \leq e \statnorm{\nabla f}^2,
\end{align*}
and therefore
\begin{align}
  \sobonorm{f} \leq \soboavgnorm{f}{(\numerOrder - 1) \stepsize} \leq (1 + e) \sobonorm{f},\label{eq:norm-equiv-h1}
\end{align}
for any bounded Lipschitz function $f$. Since the set of bounded Lipschitz functions is dense in the Hilbert space $\mathbb{H}^1$, Eq~\eqref{eq:norm-equiv-h1} holds true for any $f \in \mathbb{H}^1$.

The last piece of our proof involves bounding the quantity $\soboavgnorm{\ValTrue - \ValueFunc^\circ}{(\numerOrder - 1) \stepsize}$.
\begin{lemma}\label{lemma:h1-norm-bound-in-improved-discretization-proof}
  Under above setup, for any probability measure $\stationary$ on $\StateSpace$, we have
  \begin{align*}
    \vecnorm{\ValTrue - \ValInter{\numerOrder}}{\sobospace (\stationary)} \leq e \numerOrder \constScary_{\numerOrder + 1} \stepsize^{\numerOrder + 1} + 3 \big( \numerOrder +  \discount^{-1} \big) \constScaryReg \constScary_\numerOrder \stepsize^{\numerOrder}.
  \end{align*}
\end{lemma}
\noindent See Section~\ref{subsubsec:proof-lemma-h1-norm-bound-in-improved-discretization-proof} for the proof of this lemma.

Equipped with Lemma~\ref{lemma:h1-norm-bound-in-improved-discretization-proof}, we are now able to prove Theorem~\ref{thm:improved-approx-factor-elliptic}. By Eqs~\eqref{eq:cea-lemma-final-in-improved-approx-factor-proof},~\eqref{eq:relate-projection-error-to-valtrue}, and~\eqref{eq:norm-equiv-h1}, we have
\begin{align}
  &\sobonorm{\valuebar^{(\numerOrder)}_{\Bel} - \ValTrue} \leq \soboavgnorm{\valuebar^{(\numerOrder)}_{\Bel} - \ValTrue}{(\numerOrder - 1) \stepsize} \nonumber \\
  &\leq \soboavgnorm{\valuebar^{(\numerOrder)}_{\Bel} - \ValInter{\numerOrder} }{(\numerOrder - 1) \stepsize} + \soboavgnorm{\valuestar - \ValInter{\numerOrder} }{(\numerOrder - 1) \stepsize} \nonumber \\
  &\leq \Big\{1 + \frac{ 2e^2 \big(\discount + \smoothness_\drift^{(0)} \sqrt{\usedim} + \lammax \smoothness_\stationary \usedim + \smoothness_\covMat^{(1)} \usedim^2 \big) \lammax }{\min(\discount, 1, \lammin) \lammin} \Big\} \cdot \Big\{ \inf_{g \in \LinSpace}\soboavgnorm{g - \ValTrue}{(\numerOrder - 1) \stepsize} + \soboavgnorm{\ValTrue - \ValInter{\numerOrder}}{(\numerOrder - 1) \stepsize} \Big\} \nonumber \\
  &\qquad +  \soboavgnorm{\valuestar - \ValInter{\numerOrder} }{(\numerOrder - 1) \stepsize} \nonumber \\
  &\leq \Big\{1 + \frac{ 2e^2 \big(\discount + \smoothness_\drift^{(0)} \sqrt{\usedim} + \lammax \smoothness_\stationary \usedim + \smoothness_\covMat^{(1)} \usedim^2 \big) \lammax }{\min(\discount, 1, \lammin) \lammin} \Big\} \inf_{g \in \LinSpace}\soboavgnorm{g - \ValTrue}{(\numerOrder - 1) \stepsize} \nonumber\\
  &\qquad + \Big\{2 + \frac{ 2e^2 \big(\discount + \smoothness_\drift^{(0)} \sqrt{\usedim} + \lammax \smoothness_\stationary \usedim + \smoothness_\covMat^{(1)} \usedim^2 \big) \lammax }{\min(\discount, 1, \lammin) \lammin} \Big\} \cdot \Big\{ e \numerOrder \constScary_{\numerOrder + 1} \stepsize^{\numerOrder + 1} + 3 \big( \numerOrder +  \discount^{-1} \big) \constScaryReg \constScary_\numerOrder \stepsize^{\numerOrder} \Big\},\label{eq:improved-approx-factor-full}
\end{align}
which completes the proof of Theorem~\ref{thm:improved-approx-factor-elliptic}.

\subsubsection{Proof of Lemma~\ref{lemma:discrete-time-psd-operator}}\label{subsubsec:proof-lemma-discrete-time-psd-operator}

We start with the following expansion for any $t > 0$:
\begin{align}
  \ValFunc - e^{- \discount t} \semigroup_t \ValFunc = \int_0^t e^{- \discount s} \Big\{ \discount \semigroup_s \ValFunc - \generator \semigroup_s \ValFunc \Big\} ds.\label{eq:integral-err-expansion-for-discrete-time-psd-proof}
\end{align}
We start by showing positive definiteness of the linear operator in the integrand.
\begin{align}
  \statinprod{\ValFunc}{\discount \ValFunc - \generator \ValFunc} \geq \discount \statnorm{\ValFunc}^2 + \frac{\lammin}{2} \statnorm{\nabla \ValFunc}^2.\label{eq:positive-definite-operator}
\end{align}
The proof of Eq~\eqref{eq:positive-definite-operator} is deferred to the end of this section. Taking this lower bound as given, we now proceed with the proof of Lemma~\ref{lemma:discrete-time-psd-operator}.

Taking the inner product of the function $\ValFunc$ and the integrand in Eq~\eqref{eq:integral-err-expansion-for-discrete-time-psd-proof}, we note that
\begin{align*}
  &\statinprod{\ValFunc}{e^{- \discount s} \big( \discount \semigroup_s \ValFunc - \generator \semigroup_s \ValFunc \big)}\\
   &= \statinprod{e^{- \discount s} \semigroup_s \ValFunc}{e^{- \discount s} \big( \discount \semigroup_s \ValFunc - \generator \semigroup_s \ValFunc \big)} + \statinprod{\ValFunc - e^{- \discount s} \semigroup_s \ValFunc}{e^{- \discount s} \big( \discount \semigroup_s \ValFunc - \generator \semigroup_s \ValFunc \big)}\\
   &\geq e^{- 2 \discount s} \discount  \statnorm{\semigroup_s \ValFunc}^2 + e^{- 2 \discount s} \frac{\lammin}{2} \statnorm{\nabla \semigroup_s \ValFunc}^2 + \int_0^s e^{- \discount \ell}\statinprod{\discount \semigroup_\ell \ValFunc - \generator \semigroup_\ell \ValFunc}{e^{- \discount s} \big( \discount \semigroup_s \ValFunc - \generator \semigroup_s \ValFunc \big)} d \ell,
\end{align*}
where we apply Eq~\eqref{eq:positive-definite-operator} to the function $e^{- \discount s} \semigroup_s \valuefunc$ to bound the first term, and use It\^{o}'s formula to expand the second term.

Integrating both sides of this inequality with respect to $s$, we conclude that
\begin{multline}
  \statinprod{\ValFunc - e^{- \discount t} \semigroup_t \ValFunc}{\valuefunc} \geq \int_0^t e^{- 2 \discount s} \Big\{ \discount  \statnorm{\semigroup_s \ValFunc}^2 + \frac{\lammin}{2} \statnorm{\nabla \semigroup_s \ValFunc}^2 \Big\} ds\\
   + \int_0^t \int_0^s \statinprod{e^{- \discount \ell}  \big( \discount \semigroup_\ell \ValFunc - \generator \semigroup_\ell \ValFunc \big)}{e^{- \discount s} \big( \discount \semigroup_s \ValFunc - \generator \semigroup_s \ValFunc \big)} d \ell ds.\label{eq:full-decomp-in-psd-discrete-operator-lemma}
\end{multline}
It suffices to study the last term in Eq~\eqref{eq:full-decomp-in-psd-discrete-operator-lemma}. Applying Fubini theorem, we have that
\begin{align*}
&  \int_0^t \int_0^s \statinprod{e^{- \discount \ell}  \big( \discount \semigroup_\ell \ValFunc - \generator \semigroup_\ell \ValFunc \big)}{e^{- \discount s} \big( \discount \semigroup_s \ValFunc - \generator \semigroup_s \ValFunc \big)} d \ell ds \\
  &= \int_0^t \int_\ell^t \statinprod{e^{- \discount \ell}  \big( \discount \semigroup_\ell \ValFunc - \generator \semigroup_\ell \ValFunc \big)}{e^{- \discount s} \big( \discount \semigroup_s \ValFunc - \generator \semigroup_s \ValFunc \big)}  ds d\ell,
\end{align*}
which leads to the equation
\begin{align*}
  &\int_0^t \int_0^s \statinprod{e^{- \discount \ell}  \big( \discount \semigroup_\ell \ValFunc - \generator \semigroup_\ell \ValFunc \big)}{e^{- \discount s} \big( \discount \semigroup_s \ValFunc - \generator \semigroup_s \ValFunc \big)} d \ell ds\\
  & = \frac{1}{2} \int_0^t \int_0^t \statinprod{e^{- \discount \ell}  \big( \discount \semigroup_\ell \ValFunc - \generator \semigroup_\ell \ValFunc \big)}{e^{- \discount s} \big( \discount \semigroup_s \ValFunc - \generator \semigroup_s \ValFunc \big)}  ds d\ell \\
  &= \frac{1}{2} \statnorm{\int_0^t e^{- \discount s} \big( \discount \semigroup_s \ValFunc - \generator \semigroup_s \ValFunc \big)  ds}^2 = \frac{1}{2} \statnorm{\ValFunc - e^{- \discount t} \semigroup_t \ValFunc}^2.
\end{align*}
Substituting back to Eq~\eqref{eq:full-decomp-in-psd-discrete-operator-lemma}, we conclude the proof of Lemma~\ref{lemma:discrete-time-psd-operator}.

\paragraph{Proof of Eq~\eqref{eq:positive-definite-operator}:}
Following integration-by-parts formula, we note that
\begin{align}
  - \statinprod{\ValFunc}{ \generator \ValFunc} &= - \int \ValFunc (x) \Big\{ \inprod{\drift (x)}{\nabla \ValFunc (x)} + \frac{1}{2} \mathrm{Tr}\big( \covMat (x) \cdot \nabla^2 \ValFunc (x) \big) \Big\} \stationary (x) dx \nonumber\\
  &= - \int \ValFunc (x) \stationary (x) \inprod{\drift (x)}{\nabla \ValFunc (x)} dx + \frac{1}{2} \int \inprod{\nabla \cdot (\ValFunc (x) \stationary (x) \covMat (x))}{\nabla \ValFunc (x)} dx \nonumber \\
  &=: - I_1 + \frac{1}{2} I_2.\label{eq:decomp-in-lemma-psd}
\end{align}
For the term $I_2$, we note that
\begin{align*}
 I_2 = \int \stationary (x) \nabla \ValFunc (x)^\top \covMat (x) \nabla \ValFunc (x) dx + \int \ValFunc (x) \Big[ \nabla \cdot \big( \stationary \covMat \big)(x) \Big]^\top \nabla \ValFunc (x) dx 
\end{align*}
Applying integration by parts once more, we note that
\begin{align*}
  \int \ValFunc (x) \Big[ \nabla \cdot \big( \stationary \covMat \big)(x) \Big]^\top \nabla \ValFunc (x) dx  = -  \int \ValFunc (x) \Big[ \nabla \cdot \big( \stationary \covMat \big)(x) \Big]^\top \nabla \ValFunc (x) dx - \int \Big[ \nabla^2 \cdot \big( \stationary \covMat \big)(x) \Big] \ValFunc^2 (x) dx
\end{align*}
So we have
\begin{align*}
  \int \ValFunc (x) \Big[ \nabla \cdot \big( \stationary \covMat \big)(x) \Big]^\top \nabla \ValFunc (x) dx = -\frac{1}{2} \int \Big[ \nabla^2 \cdot \big( \stationary \covMat \big)(x) \Big] \ValFunc^2 (x) dx,
\end{align*}
and therefore,
\begin{align*}
  I_2 = \statnorm{\covMat^{1/2} \nabla \ValFunc}^2 - \frac{1}{2} \int \Big[ \nabla^2 \cdot \big( \stationary \covMat \big)(x) \Big] \ValFunc^2 (x) dx
  \geq \lammin \statnorm{\nabla \ValFunc}^2 - \frac{1}{2} \int \Big[ \nabla^2 \cdot \big( \stationary \covMat \big)(x) \Big] \ValFunc^2 (x) dx.
\end{align*}
On the other hand, applying integration by parts to $I_1$ yields
\begin{align*}
  I_1 = \int \ValFunc (x) \stationary (x) \inprod{\drift (x)}{\nabla \ValFunc (x)} dx
  = \frac{1}{2} \int \stationary (x) \inprod{\drift (x)}{\nabla \big( \ValFunc^2 (x) \big)} dx
  = - \frac{1}{2} \int \ValFunc^2 (x)  \nabla \cdot \big( \stationary (x) \drift (x) \big) dx.
\end{align*}
Substituting them into Eq~\eqref{eq:decomp-in-lemma-psd}, we conclude that
\begin{align}
    - \statinprod{\ValFunc}{ \generator \ValFunc} \geq \frac{\lammin}{2}\statnorm{\nabla \ValFunc}^2 - \frac{1}{2} \int \ValFunc^2 (x) \Big\{ - \nabla \cdot \big( \stationary (x) \drift (x) \big) + \frac{1}{2} \nabla^2 \cdot \big( \stationary \covMat \big)(x)  \Big\} dx\label{eq:final-bound-in-lemma-psd}
\end{align}
Note that $\stationary$ is the stationary distribution of the diffusion process~\eqref{eq:cts-time-process}. The Fokker--Planck equation yields
\begin{align*}
   - \nabla \cdot \big( \stationary (x) \drift (x) \big) + \frac{1}{2} \nabla^2 \cdot \big( \stationary \covMat \big)(x) = 0.
\end{align*}
Substituting back to Eq~\eqref{eq:final-bound-in-lemma-psd} concludes the proof of the lemma.

\subsubsection{Proof of Lemma~\ref{lemma:discrete-time-operator-norm-bound}}\label{subsubsec:proof-lemma-discrete-time-operator-norm-bound}

Similar to the proof of Lemma~\ref{lemma:discrete-time-psd-operator}, we can write the quantity of interest as a time integral
\begin{align}
  \statinprod{g}{t^{-1}\big(\IdMat - e^{- \discount t} \semigroup_t \big) f} = \frac{1}{t} \int_0^t \statinprod{g}{ e^{- \discount s} (\discount \IdMat - \generator)\semigroup_s f} ds.\label{eq:integral-decomp-in-operator-norm-lemma-proof}
\end{align}
For the first term in the right-hand-side of Eq~\eqref{eq:integral-decomp-in-operator-norm-lemma-proof}, we use a coarse bound
\begin{align*}
  \abss{\frac{1}{t} \int_0^t \statinprod{g}{ e^{- \discount s} \discount \semigroup_s f} ds } \leq \discount \statnorm{g} \cdot \sup_{0 \leq s \leq t} \statnorm{\semigroup_s f} \leq \discount e^{\beta_0 t} \statnorm{g} \cdot \statnorm{f},
\end{align*}
where in the last step, we use the fact $\statnorm{\semigroup_s f} \leq e^{\beta_0 s}\statnorm{f}$.

Now we study the second term of the integrand. We note that
\begin{align}
  \statinprod{g}{\generator \semigroup_s f} = \int g (x) \stationary (x) \Big\{ \inprod{\drift}{\nabla \semigroup_s f} + \frac{1}{2} \mathrm{Tr} \big( \covMat \cdot \nabla^2 \semigroup_s f \big) \Big\} (x)  dx.\label{eq:expansion-inprod-in-op-norm-bound-proof}
\end{align}
Integration by parts yields
\begin{multline*}
  \int g (x) \stationary (x) \mathrm{Tr} \big( \covMat (x) \cdot \nabla^2 \semigroup_s f  (x) \big)   dx\\
  = - \int  \nabla g (x)^\top \covMat (x)  \nabla \semigroup_s f  (x) \stationary (x)  dx - \int g (x) \big[ \nabla \cdot (\stationary\covMat)(x) \big]^\top  \nabla \semigroup_s f  (x)  dx
\end{multline*}
Note that
\begin{align*}
  \nabla \cdot (\stationary\covMat)(x) = \stationary (x) \cdot \Big[ \nabla \cdot \covMat (x) + \covMat (x) \cdot \nabla \log \stationary (x) \Big].
\end{align*}
By substituting above equations back to Eq~\eqref{eq:expansion-inprod-in-op-norm-bound-proof} and applying Cauchy--Schwarz inequality, we arrive at the bound
\begin{align}
   \abss{\statinprod{g}{\generator \semigroup_s f}} &\leq \lammax \cdot \statnorm{\nabla g} \cdot \statnorm{\nabla \semigroup_s f}  \nonumber\\
   &\qquad + \statnorm{g} \cdot \statnorm{\nabla \semigroup_s f} \cdot \sup_{\state \in \StateSpace} \Big( \vecnorm{\drift (x)}{2} + \vecnorm{\nabla \cdot \covMat (\state)}{2} + \vecnorm{\covMat \cdot \nabla \log \stationary (x)}{2} \Big) \nonumber\\
   &\leq \big( \lammax + \smoothness_\drift^{(0)} \sqrt{\usedim} + \lammax \smoothness_\stationary \usedim + \smoothness_\covMat^{(1)} \usedim^2 \big) \statnorm{\nabla \semigroup_s f} \cdot \big( \statnorm{\nabla g} + \statnorm{g} \big).\label{eq:opnorm-bound-for-generator}
\end{align}
Substituting back to Eq~\eqref{eq:integral-decomp-in-operator-norm-lemma-proof}, we conclude that
\begin{align*}
  &\abss{\statinprod{g}{t^{-1}\big(\IdMat - e^{- \discount t} \semigroup_t \big) f}} \\
  &\leq \frac{1}{t} \int_0^t  \Big\{ \discount \abss{\statinprod{g}{\semigroup_s f} } +  \abss{\statinprod{g}{\generator\semigroup_s f } } \Big\} ds\\
  &\leq \discount e^{\beta_0 t} \statnorm{g} \cdot \statnorm{f} + \big( \lammax + \smoothness_\drift^{(0)} \sqrt{\usedim} + \lammax \smoothness_\stationary \usedim + \smoothness_\covMat^{(1)} \usedim^2 \big) \frac{1}{t} \int_0^t \statnorm{\nabla \semigroup_s f} ds \cdot \big( \statnorm{\nabla g} + \statnorm{g} \big),
\end{align*}
which completes the proof of this lemma.

\subsubsection{Proof of Lemma~\ref{lemma:reverse-smoothing-estiamte}}\label{subsubsec:proof-lemma-reverse-smoothing-estiamte}

First, we note that for $0 \leq s \leq t$, we have $\statnorm{\semigroup_t f} = \statnorm{\semigroup_{t - s} \semigroup_s f} \leq e^{\beta_0 (t - s)} \statnorm{\semigroup_s f}$. Consequently, for $t \leq 1 \wedge \tfrac{1}{\beta_0}$, we have the lower bound
\begin{align}
  \frac{1}{t} \int_0^t e^{- 2\discount s} \statnorm{\semigroup_s f}^2 ds + \frac{1}{t} \statnorm{f - e^{- \discount t} \semigroup_t f}^2 \geq  \frac{1}{e} \statnorm{e^{- \discount t}\semigroup_t f}^2 + \statnorm{f - e^{- \discount t}\semigroup_t f}^2 \geq \frac{1}{2e} \statnorm{f}^2,\label{eq:reverse-smoothing-estimate-bound-1}
\end{align}
where we use Young's inequality in the last step.

Moreover, for $t \leq \tfrac{1}{2 \discount}$, we have
\begin{align}
  \frac{1}{t} \int_0^t e^{- 2\discount s} \statnorm{\nabla \semigroup_s f}^2 ds \geq  \frac{e^{-1}}{t} \int_0^t\statnorm{\nabla \semigroup_s f}^2 ds.\label{eq:reverse-smoothing-estimate-bound-2}
\end{align}

It suffices to bound $\statnorm{\nabla f}^2$ from above using the positive semi-definite structure induced in Lemma~\ref{lemma:discrete-time-psd-operator}. This is not possible for a general function $f$. However, the regularity of the functions in the space $\LinSpace$ allows us to bound this quantity when the time period $t$ is small.

For any function $f \in \LinSpace$, we note that
\begin{align*}
  \frac{d}{ds} \Big[ \statnorm{\covMat^{1/2} \nabla \semigroup_s f}^2 \Big] &= \int \inprod{\covMat (x) \nabla \semigroup_s f (x)}{\partial_s \nabla \semigroup_s f (x)} \stationary (x) dx\\
  & = \int \inprod{\covMat (x) \nabla \semigroup_s f (x)}{\nabla \big( \semigroup_s \generator f (x) \big)} \stationary (x) dx,
\end{align*}
where in the last step, we interchange time and space derivatives, and use the fact $\generator \semigroup_s = \semigroup_s \generator$.

Integration by parts yields
\begin{multline}
   \int \inprod{\covMat (x) \nabla \semigroup_s f (x)}{\nabla \big( \generator \semigroup_s f (x) \big)} \stationary (x) dx \\
   = - \int \Big[ \inprod{ \nabla \semigroup_s f}{\covMat \nabla \log \stationary + \nabla \cdot \covMat} + \mathrm{Tr} \big( \covMat \cdot \nabla^2 \semigroup_s f \big) \Big] (x) \cdot \semigroup_s \generator f (x) \stationary (x) dx.\label{eq:integration-by-parts-in-reverse-smoothing-bound}
\end{multline}
For the last term of Eq~\eqref{eq:integration-by-parts-in-reverse-smoothing-bound}, we note that
\begin{align*}
  &\int  \mathrm{Tr} \big( \covMat (x) \cdot \nabla^2 \semigroup_s f (x) \big) \cdot \semigroup_s \generator f (x) \stationary (x) dx\\
  &= 2 \int  \Big[ \inprod{\nabla \semigroup_s f}{\drift} + \frac{1}{2} \mathrm{Tr} \big( \covMat \cdot \nabla^2 \semigroup_s f \big) \Big] (x) \cdot \semigroup_s \generator f (x) \stationary (x) dx - 2 \int \inprod{\nabla \semigroup_s f (x)}{\drift (x)}  \semigroup_s \generator f (x) \stationary (x) dx\\
  &= 2 \statnorm{\semigroup_s \generator f}^2 - 2 \int \inprod{\nabla \semigroup_s f (x)}{\drift (x)}  \semigroup_s \generator f (x) \stationary (x) dx.
\end{align*}
Substituting back to Eq~\eqref{eq:integration-by-parts-in-reverse-smoothing-bound} yields
\begin{multline*}
  \int \inprod{\covMat (x) \nabla \semigroup_s f (x)}{\nabla \big( \generator \semigroup_s f (x) \big)} \stationary (x) dx \\
  =  - 2 \statnorm{\semigroup_s \generator f}^2 + \int \Big[ \inprod{ \nabla \semigroup_s f}{2 \drift - \covMat \nabla \log \stationary - \nabla \cdot \covMat} \Big] (x) \cdot \semigroup_s \generator f (x) \stationary (x) dx
\end{multline*}
Now we study the two terms in the expression above, for the first term, the expansive factor bound for the semigroup yields
\begin{align}
\statnorm{\semigroup_s \generator f} \leq e^{\beta_0 s} \statnorm{\generator f} \leq e  \statnorm{\generator f}.\label{eq:contraction-bound-in-reverse-smoothing-proof}
\end{align}
For the second term, we apply Cauchy--Schwarz inequality, and obtain the bound
\begin{align*}
  &\abss{\int \Big[ \inprod{ \nabla \semigroup_s f}{2 \drift - \covMat \nabla \log \stationary - \nabla \cdot \covMat} \Big] (x) \cdot \semigroup_s \generator f (x) \stationary (x) dx} \\
  &\leq \statnorm{\inprod{ \nabla \semigroup_s f}{2 \drift - \covMat \nabla \log \stationary - \nabla \cdot \covMat} } \cdot \statnorm{ \semigroup_s \generator f}\\
  &\leq \big(2 \smoothness_\drift^{(0)} + \lammax \smoothness_\stationary + \usedim \smoothness_\covMat^{(1)} \big)\statnorm{ \nabla \semigroup_s f } \cdot \statnorm{ \generator f},
\end{align*}
where in the last step, we invoke the smoothness assumptions~\ref{assume:smooth-high-order} and~\ref{assume:smooth-stationary}, and apply Eq~\eqref{eq:contraction-bound-in-reverse-smoothing-proof} to the quantity $\statnorm{ \semigroup_s \generator f}$.

Collecting the bounds above, we conclude that
\begin{align}
  \abss{\frac{d}{ds} \Big[ \statnorm{\covMat^{1/2} \nabla \semigroup_s f}^2 \Big]}
   &\leq e^2\statnorm{\generator f}^2 +  \big(2 \smoothness_\drift^{(0)} + \lammax \smoothness_\stationary + \usedim \smoothness_\covMat^{(1)} \big)\statnorm{ \nabla \semigroup_s f } \cdot \statnorm{ \generator f} \nonumber \\
  &\leq 12 \statnorm{\generator f}^2 +  \frac{(2 \smoothness_\drift^{(0)} + \lammax \smoothness_\stationary + \usedim \smoothness_\covMat^{(1)} )^2}{\lammin} \statnorm{ \covMat^{1/2}\nabla \semigroup_s f }^2\label{eq:derivative-bound-in-proof-of-reverse-smoothing-lemma}
\end{align}
In order to prove the lemma, we use the following estimate for functions in $\LinSpace$.
\begin{align}
  \statnorm{\generator f} \leq 2 c_0 \usedim \big( \smoothness_\drift^{(0)} + \smoothness_\covMat^{(0)} \big) \mbasis^{2 \omega} \statnorm{f}, \quad \mbox{for } f \in \LinSpace. \label{eq:reverse-poincare-for-good-funcs}
\end{align}
We prove Eq~\eqref{eq:reverse-poincare-for-good-funcs} at the end of this section. Taking Eq~\eqref{eq:reverse-poincare-for-good-funcs} as given, we now proceed the proof of this lemma.

Applying Eq~\eqref{eq:reverse-poincare-for-good-funcs} to Eq~\eqref{eq:derivative-bound-in-proof-of-reverse-smoothing-lemma}, and invoking Gr\"{o}nwall's inequality for the reversed process from time $s$ to time $0$, we have that
\begin{align*}
   \statnorm{\covMat^{1/2} \nabla f}^2 \leq \exp \Big[ \frac{(2 \smoothness_\drift^{(0)} + \lammax \smoothness_\stationary + \usedim \smoothness_\covMat^{(1)} )^2}{\lammin} s \Big] \Big\{ \statnorm{\covMat^{1/2} \nabla \semigroup_s f}^2 + 48 s c_0^2 \usedim^2 \big( \smoothness_\drift^{(0)} + \smoothness_\covMat^{(0)} \big)^2 \mbasis^{4 \omega} \statnorm{f}^2 \Big\}.
\end{align*}
For $ t \leq \frac{\lammin}{(2 \smoothness_\drift^{(0)} + \lammax \smoothness_\stationary + \usedim \smoothness_\covMat^{(1)} )^2}$, by taking average over $s \in [0, t]$, we have the bound
\begin{align*}
  \statnorm{\covMat^{1/2} \nabla f}^2 \leq  \frac{e}{t} \int_0^t \statnorm{\covMat^{1/2} \nabla \semigroup_s f}^2 ds + 48 e c_0^2 \usedim^2 \big( \smoothness_\drift^{(0)} + \smoothness_\covMat^{(0)} \big)^2 t \mbasis^{4 \omega} \statnorm{f}^2,
\end{align*}
and consequently, for $t \leq \tfrac{1}{96 e c_0^2 \usedim^2 ( \smoothness_\drift^{(0)} + \smoothness_\covMat^{(0)})^2}  \mbasis^{- 4 \omega}$ and $f \in \LinSpace$, we have
\begin{align}
  \statnorm{\nabla f}^2 \leq  \frac{e \lammax}{\lammin} \cdot \frac{1}{t}\int_0^t \statnorm{ \nabla \semigroup_s f}^2 ds + \frac{1}{2} \statnorm{f}^2.\label{eq:reverse-smoothing-estimate-bound-3}
\end{align}
Combining the bounds~\eqref{eq:reverse-smoothing-estimate-bound-1},~\eqref{eq:reverse-smoothing-estimate-bound-2}, and~\eqref{eq:reverse-smoothing-estimate-bound-3}, we conclude that
\begin{align*}
  &\frac{1}{t} \int_0^t e^{- 2 \discount s} \statnorm{\semigroup_s \ValFunc}^2 ds + \frac{1}{t} \int_0^t e^{- 2 \discount s} \statnorm{\nabla \semigroup_s \ValFunc}^2 ds + \frac{1}{t} \statnorm{\ValFunc - e^{- \discount t} \semigroup_t \ValFunc}^2 \\
  &\geq \frac{1}{2} \statnorm{f}^2 +  \frac{1}{e t} \int_0^t \statnorm{\nabla \semigroup_s \ValFunc}^2 ds\\
  &\geq \frac{1}{4} \statnorm{f}^2 +  \frac{1}{2 e} t^{-1} \int_0^t \statnorm{\nabla \semigroup_s \ValFunc}^2 ds + \frac{\lammin}{2 e^2 \lammax} \statnorm{\nabla f}^2\\
  & \geq \frac{\lammin}{2 e^2 \lammax} \soboavgnorm{f}{ t}^2,
\end{align*}
which completes the proof of Lemma~\ref{lemma:reverse-smoothing-estiamte}.

\paragraph{Proof of Eq~\eqref{eq:reverse-poincare-for-good-funcs}:} By triangle inequality, we have
\begin{multline}
  \statnorm{\generator f} = \statnorm{\inprod{\drift}{\nabla f} + \frac{1}{2} \mathrm{Tr} (\covMat \nabla^2 f)} \\
  \leq \statnorm{\inprod{\drift}{\nabla f}} + \frac{1}{2} \statnorm{ \mathrm{Tr} (\covMat \nabla^2 f)} \leq \usedim \cdot \Big\{ \smoothness_\drift^{(0)} \statnorm{\nabla f} +  \smoothness_\covMat^{(0)} \statnorm{\nabla^2 f} \Big\}.\label{eq:basic-decomp-in-reverse-poincare-proof}
\end{multline}
By Assumption~\ref{assume:basis-condition}, since $f \in \LinSpace$, we have
\begin{align*}
   \statnorm{\nabla f} \leq c_0 \mbasis^\omega \statnorm{f}, \quad \mbox{and} \statnorm{\nabla^2 \phi_j} \leq c_0 j^{2\omega},
\end{align*}
Substituting back to Eq~\eqref{eq:basic-decomp-in-reverse-poincare-proof} completes the proof of Eq~\eqref{eq:reverse-poincare-for-good-funcs}.

\subsubsection{Proof of Lemma~\ref{lemma:h1-norm-bound-in-improved-discretization-proof}}\label{subsubsec:proof-lemma-h1-norm-bound-in-improved-discretization-proof}
By Theorem~\ref{thm:discretization-high-order}, we have
\begin{align}
  \statnorm{\ValTrue - \ValInter{\numerOrder}} \leq \vecnorm{\ValTrue - \ValInter{\numerOrder}}{\infty} \leq \constScary_\numerOrder \stepsize^{\numerOrder}.\label{eq:l2-norm-bound-in-h1-lemma-proof}
\end{align}
Now we start analyzing the gradient terms. Note that
\begin{align*}
  \ValTrue &=  \int_0^{+\infty} e^{- \discount t} \semigroup_s \reward = \sum_{k = 0}^{+ \infty} \int_0^{(\numerOrder - 1) \stepsize}  e^{- \discount (s + k (\numerOrder - 1) \stepsize ) } \semigroup_{s + k (\numerOrder - 1) \stepsize} \reward ds, \quad \mbox{and}\\
  \ValInter{\numerOrder} &= \sum_{k = 0}^{+ \infty} \int_0^{(\numerOrder - 1) \stepsize} \sum_{i = 0}^{\numerOrder - 1}  W_i (s) e^{- \discount (s + k (\numerOrder - 1) \stepsize ) } \semigroup_{k (\numerOrder - 1) \stepsize + i \stepsize} \reward ds
\end{align*}
By triangle inequality, the difference of their gradients admits the upper bound
\begin{align}
  &\vecnorm{\nabla \ValTrue (\state) - \nabla\ValInter{\numerOrder} (\state)}{2} \nonumber\\
  &\leq \sum_{k = 0}^{+ \infty} \int_0^{(\numerOrder - 1) \stepsize}  e^{- \discount (s + k (\numerOrder - 1) \stepsize ) } \vecnorm{ \sum_{i = 0}^{\numerOrder - 1}  W_i (s) \nabla \semigroup_{k (\numerOrder - 1) \stepsize + i \stepsize} \reward (\state) - \nabla \semigroup_{k (\numerOrder - 1) \stepsize + s} \reward (\state)}{2}ds \nonumber \\
  &\leq \numerOrder \stepsize\sum_{k = 0}^{+ \infty} e^{- \discount k (\numerOrder - 1) \stepsize  } \sup_{0 \leq s \leq (\numerOrder - 1) \stepsize} \vecnorm{\sum_{i = 0}^{\numerOrder - 1}  W_i (s) \nabla \semigroup_{k (\numerOrder - 1) \stepsize + i \stepsize} \reward (\state) - \nabla \semigroup_{k (\numerOrder - 1) \stepsize + s} \reward (\state)}{2}.\label{eq:triangle-ineq-in-grad-norm-bound-for-h1-lemma}
\end{align}
For any $\state \in \StateSpace$ and $k \in \mathbb{N}$, consider vector-valued function $h_{\state, k}: s \mapsto \nabla  \semigroup_{k (\numerOrder - 1) \stepsize + s} \reward (\state)$ for time $s \in [0, (\numerOrder - 1) \stepsize]$. Note that within the time interval $[k (\numerOrder - 1) \stepsize, (k + 1) (\numerOrder - 1) \stepsize]$, the polynomial $s \mapsto \sum_{i = 0}^{\numerOrder - 1}  W_i (s) \nabla \semigroup_{k (\numerOrder - 1) \stepsize + i \stepsize} \reward$ is the $\numerOrder$-th order Lagrangian polynomial for the function $h_{\state, k}$. By classical error bounds for Lagrangian polynomial approximations~\cite{king1999afternotes}, we have
\begin{align}
  \vecnorm{\sum_{i = 0}^{\numerOrder - 1}  W_i (s) \nabla \semigroup_{k (\numerOrder - 1) \stepsize + i \stepsize} \reward (\state) - \nabla \semigroup_{k (\numerOrder - 1) \stepsize + s} \reward (\state)}{2} \leq \stepsize^\numerOrder \cdot \sup_{t \in [0, (\numerOrder - 1) \stepsize]} \vecnorm{\partial_t^\numerOrder h_{\state, k}}{2},\label{eq:lagrange-poly-bound-for-grad}
\end{align}
for any $\state \in \StateSpace$.

Now we study the time derivatives of the function $h_{\state, k}$. By interchanging the derivatives and interchanging the semigroup with its generator, we can derive the expression for its first-order time derivative. 
\begin{align*}
  \partial_t h_{\state, k} (t) = \nabla \partial_t  \semigroup_{k (\numerOrder - 1) \stepsize + t} \reward (\state) = \nabla \semigroup_{k (\numerOrder - 1) \stepsize + t} \generator \reward (\state).
\end{align*}
Recursively applying the arguments above, we can compute its higher order derivatives
\begin{align*}
  \partial_t^j h_{\state, k} (t) = \nabla \partial_t  \semigroup_{k (\numerOrder - 1) \stepsize + t} \reward (\state) = \nabla \semigroup_{k (\numerOrder - 1) \stepsize + t} \generator^j \reward (\state),
\end{align*}
for $j = 1,2, \cdots, \numerOrder$.

First, we note that by Eq~\eqref{eq:polynomial-induction-proof-lemma-high-order} in the proof of Lemma~\ref{lemma:semigroup-high-order-bound}, $\generator^j \reward$ takes form of a $(j + 1)$-th order polynomial of $\big\{\partial^\alpha \reward \}_{|\alpha| \leq 2j}$, $\big\{ \partial^\alpha \drift_i \big\}_{|\alpha|\leq 2j - 2,~ i \in [\usedim]}$ and $\big\{ \partial^\alpha \covMat_{i, k}\big\}_{|\alpha| \leq 2j - 2,~ i, k \in [\usedim]}$. So under Assumption~\fakerefassumelip{$(\numerOrder + 1)$}, we have
\begin{subequations}\label{eq:two-generator-bound-in-h1-bound-proof} 
\begin{align}
  \sup_{x \in \StateSpace} \abss{ \generator^{\numerOrder} \reward (\state)} &\leq \constScary_{\numerOrder}, \quad \mbox{and} \label{eq:generator-bound-in-h1-bound-proof} \\
  \sup_{x \in \StateSpace} \vecnorm{ \nabla \generator^{\numerOrder} \reward (\state)}{2} &\leq \constScary_{\numerOrder + 1}.\label{eq:grad-generator-bound-in-h1-bound-proof} 
\end{align}
\end{subequations}
To derive the desired bounds on the time derivatives of $h_{\state, k}$, we combine Eqs~\eqref{eq:two-generator-bound-in-h1-bound-proof} with the gradient estimates in Proposition~\ref{prop:semigroup-grad-estimate-local-grow} and~\ref{prop:semigroup-grad-estimate-regularize}.

By Eq~\eqref{eq:grad-generator-bound-in-h1-bound-proof} and Proposition~\ref{prop:semigroup-grad-estimate-local-grow}, when the stepsize satisfies $\stepsize \numerOrder \leq \frac{\lammin}{\smoothness_\covMat^{(1)} \usedim^3 + \smoothness_\drift^{(1)} \usedim^2}$, we have
\begin{align*}
  \vecnorm{\partial_t^\numerOrder h_{\state, 0} (t)}{2} = \vecnorm{\nabla \semigroup_t \generator^{\numerOrder} \reward (\state)}{2} \leq e \sqrt{\semigroup_t \Big( \vecnorm{\nabla \generator^{\numerOrder} \reward  }{2}^2 \Big) (\state)} \leq e \constScary_{\numerOrder + 1}, \quad \mbox{for any }t \in [0, \stepsize (\numerOrder - 1)].
\end{align*}
On the other hand, by Eq~\eqref{eq:generator-bound-in-h1-bound-proof} and Proposition~\ref{prop:semigroup-grad-estimate-regularize}, for any $k \geq 1$, we have
\begin{align*}
 \vecnorm{\partial_t^\numerOrder h_{\state, k} (t)}{2} = \vecnorm{\nabla \semigroup_{k (\numerOrder - 1) \stepsize + t} \generator^{\numerOrder} \reward (\state)}{2} \leq \frac{\constScaryReg}{\sqrt{k  (\numerOrder - 1) \stepsize} \wedge 1} \sqrt{\semigroup_t \big(\generator^{\numerOrder} \reward \big)^2 (x)} \leq \frac{\constScaryReg \constScary_\numerOrder}{\sqrt{k  (\numerOrder - 1) \stepsize} \wedge 1},
\end{align*}
for any $\state \in \StateSpace$ and $t \in [0, \stepsize (\numerOrder - 1)]$.

According to Eq~\eqref{eq:lagrange-poly-bound-for-grad}, we conclude that
\begin{align*}
  \vecnorm{\sum_{i = 0}^{\numerOrder - 1}  W_i (t) \nabla \semigroup_{k (\numerOrder - 1) \stepsize + i \stepsize} \reward (\state) - \nabla \semigroup_{k (\numerOrder - 1) \stepsize + t} \reward (\state)}{2} \leq \begin{cases}
    e \constScary_{\numerOrder + 1} \stepsize^\numerOrder & k = 0,\\
  \frac{1}{\sqrt{k  (\numerOrder - 1) \stepsize} \wedge 1} \constScaryReg \constScary_\numerOrder \stepsize^\numerOrder & k \geq 1.
  \end{cases}
\end{align*}
Substituting back to Eq~\eqref{eq:triangle-ineq-in-grad-norm-bound-for-h1-lemma}, we obtain that
\begin{align}
  &\vecnorm{\nabla \ValTrue (\state) - \nabla\ValInter{\numerOrder} (\state)}{2} \nonumber \\
   &\leq \numerOrder \stepsize\sum_{k = 0}^{+ \infty} e^{- \discount k (\numerOrder - 1) \stepsize  } \sup_{0 \leq s \leq (\numerOrder - 1) \stepsize} \vecnorm{\sum_{i = 0}^{\numerOrder - 1}  W_i (s) \nabla \semigroup_{k (\numerOrder - 1) \stepsize + i \stepsize} \reward (\state) - \nabla \semigroup_{k (\numerOrder - 1) \stepsize + s} \reward (\state)}{2}\nonumber \\
   &\leq \numerOrder \stepsize \cdot \Big\{e \constScary_{\numerOrder + 1} \stepsize^\numerOrder + \sum_{k = 1}^{\lceil (\numerOrder - 1)^{-1} \stepsize^{-1} \rceil}  \frac{1}{\sqrt{k  (\numerOrder - 1) \stepsize}} \constScaryReg \constScary_\numerOrder \stepsize^\numerOrder + \sum_{k = \lceil (\numerOrder - 1)^{-1} \stepsize^{-1} \rceil + 1}^{+ \infty} e^{- \discount k (\numerOrder - 1) \stepsize  }  \constScaryReg \constScary_\numerOrder \stepsize^\numerOrder \Big\}\nonumber \\
   &\leq e \numerOrder \constScary_{\numerOrder + 1} \stepsize^{\numerOrder + 1} + 2 \Big\{ \numerOrder +  \frac{e^{- \discount}}{\discount} \Big\} \constScaryReg \constScary_\numerOrder \stepsize^{\numerOrder},\label{eq:grad-diff-ptwise-bound-in-h1-lemma-proof}
\end{align}
for any $\state \in \StateSpace$. Therefore, we can bound the Sobolev norm for the value function error
\begin{align}
  \statnorm{\nabla \ValTrue (\state) - \nabla\ValInter{\numerOrder} (\state)} \leq e \numerOrder \constScary_{\numerOrder + 1} \stepsize^{\numerOrder + 1} + 2 \Big\{ \numerOrder +  \frac{e^{- \discount}}{\discount} \Big\} \constScaryReg \constScary_\numerOrder \stepsize^{\numerOrder}.\label{eq:grad-diff-l2-bound-in-h1-lemma-proof}
\end{align}
Collecting Eqs~\eqref{eq:l2-norm-bound-in-h1-lemma-proof},~\eqref{eq:grad-diff-l2-bound-in-h1-lemma-proof}, we conclude that
\begin{align*}
  \soboavgnorm{\ValTrue - \ValInter{\numerOrder}}{(\numerOrder - 1)\stepsize} \leq e \numerOrder \constScary_{\numerOrder + 1} \stepsize^{\numerOrder + 1} + \constScary_\numerOrder \stepsize^\numerOrder + 2 \Big\{ \numerOrder +  \frac{e^{ - \discount}}{\discount} \Big\} \constScaryReg \constScary_\numerOrder \stepsize^{\numerOrder},
\end{align*}
which proves the desired lemma.

\def\b{\discount}
\def\mLd{\generator_{\numerOrder,\dt}}
\def\nb{\nabla}
\def\rho{\stationary}
\def\eucl{2}
\def\ep{e_p}
\def\eg{e_g}
\def\evec{e}
\def\he{\hat{e}}
\def\coefgkep{c_3} 
\def\coefgktd{c_4}

\subsection{Proof of Theorem \ref{thm:approx-gen}}\label{subsec:proof-thm-approx-gen}
Througout the proof, we use the following known facts from the proof of Theorem~\ref{thm:improved-approx-factor-elliptic}.
\begin{subequations}
\begin{align}
  - \statinprod{\generator f}{f} &\geq \frac{\lammin}{2}\ll \nb f \rl_\rho^2.\label{eq:coercivity-in-approx-gen-proof}\\
  \statnorm{\semigroup_s f } &\leq \statnorm{f}.\label{eq:coercivity-in-time}\\
  \abss{ \statinprod{\generator f}{g}} & \leq \big( \smoothness_\drift^{(0)} \sqrt{\usedim} + \smoothness_\covMat^{(1)} \usedim^2 + \lammax \smoothness_\stationary \usedim \big) \statnorm{g} \cdot \statnorm{\nabla f} + \lammax \statnorm{\nabla f} \cdot \statnorm{\nabla g}.\label{eq:gen-upper}
\end{align}
\end{subequations}
Eq~\eqref{eq:coercivity-in-approx-gen-proof} is from Eq~\eqref{eq:positive-definite-operator} in the the proof of Lemma~\ref{lemma:discrete-time-psd-operator}; Eq~\eqref{eq:coercivity-in-time} is from Eq~\eqref{eq:semigroup-is-non-expansive} in the proof of~\Cref{cor:projected-discretized-bellman}; and Eq~\eqref{eq:gen-upper} is from Eq~\eqref{eq:opnorm-bound-for-generator} in the proof of~\Cref{lemma:discrete-time-operator-norm-bound}.

Given these inequalities, we now proceed with the proof of Theorem~\ref{thm:approx-gen}. We start by decomposing the true value function into two parts.
\[
\valuestar(\state) = \valuestar_\LinSpace(\state) + \ep(\state)
\]
where $\valuestar_\LinSpace(\state) = \sum_{j=1}^\mbasis\valuestar_j\phi_j(\state)$ could be any function in the linear space spanned by $\{\phi_j\}_{j=1}^\mbasis$, while $\ep = \valuestar(\state) - \valuestar_\LinSpace(\state)$.
Additionally, we write  $\valtdinter{\numerOrder} (\state) = \sum_{j=1}^\mbasis \valuebar^{(\numerOrder)}_j \phi_j(\state)$ as the Galerkin solution that satisfies,  
\begin{align*}
  \statinprod{\b \valtdinter{\numerOrder} - \geninter{\numerOrder} \valtdinter{\numerOrder} - r(\state)}{\phi_j} = 0, \quad \mbox{for $j = 1,2,\cdots, \mbasis$.}
\end{align*}

The error $\valuestar(\state) - \valtdinter{\numerOrder} (\state) $ can be separated into two parts, 
\[
 \valuestar(\state) - \valtdinter{\numerOrder}(\state)= \eg(\state) + \ep(\state), \quad \text{where }\eg =  \valuestar_\LinSpace - \valtdinter{\numerOrder}, \quad \ep = \valuestar - \valuestar_\LinSpace.
\]

Similar to the derivation of \eqref{eq:gen-genapprox}, one has
\[
\geninter{\numerOrder} \valuefunc(\state) = \generator \valuefunc(\state) + \underbrace{\sum_{j=0}^\numerOrder\frac{\coef{\numerOrder}_j}\dt \int_0^{j\dt} \semigroup_\xi\generator^{\numerOrder+1}\valuefunc(\state)\frac{(j\dt - \xi)^\numerOrder}{\numerOrder!} d \xi}_{\mLd\valuefunc(\state)}.
\]
where we define $\mLd$ as the difference between the generator $\generator$ and the approximated generator $\geninter{\numerOrder}$. Therefore, the Galerkin solution satisfies
\begin{align}\label{ineq3}
    \statinprod{\b \valtdinter{\numerOrder} - \generator \valtdinter{\numerOrder} - r(\state) - \mLd \valtdinter{\numerOrder}}{\phi_j} = 0, \quad \mbox{for any $j \in [\mbasis]$}.
\end{align}
On the other hand, the exact solution satisfies  $\statinprod{\b \valuestar - \generator \valuestar - r}{\phi_j}$ for any $j \in [\mbasis]$,
which implies
\[
    \begin{aligned}
    \statinprod{\b \valuestar_\LinSpace- \generator \valuestar_\LinSpace - \big( r - \b \ep + \generator \ep \big)}{\phi_j} = 0, \quad \mbox{for $j  \in [\mbasis]$}.
\end{aligned}
\]
Therefore, subtracting \eqref{ineq3} from the above equation gives
\[
\statinprod{\b \eg - \generator \eg - \big( -\mLd \valtdinter{\numerOrder} - \b \ep + \generator \ep \big)}{ \phi_j }, \quad \mbox{for $j \in [\mbasis]$}.
\]
Define $\evec\in\R^\mbasis$ as a $\mbasis$-dimensional vector, where $\evec_j =(\valuebar^{(\numerOrder)}_j - \valuestar_j) $.
Multiplying $\evec_j$ to the $j$-th component of the above equation and summing it over $1\leq j \leq \mbasis$ gives
\begin{equation}\label{ineq4}
    \begin{aligned}
    &\statinprod{  (\b -  \generator)  \eg(\state)}{ \eg(\state)}= -\statinprod{ \mLd\valtdinter{\numerOrder} }{ \eg }- \statinprod{ (\b - \generator) \ep}{ \eg}\\
\end{aligned}
\end{equation}
By Eq~\eqref{eq:coercivity-in-approx-gen-proof}, we have
\begin{equation}\label{ineq5}
\la  (\b - \generator)  \eg(\state), \eg(\state)\ra_\rho \geq \b \ll  \eg \rl_\rho^2 + \frac{\lammin}2\ll \nb \eg \rl^2_\rho
\end{equation}
Furthermore, 
\begin{equation}\label{ineq6}
\begin{aligned}
    &\lv \la  \mLd \valtdinter{\numerOrder}, \eg(\state)\ra_\rho\rv
   = \lv \sum_{j=0}^\numerOrder\frac{\coef{\numerOrder}_j}\dt \int_0^{j\dt} \la  \semigroup_s\generator^{\numerOrder+1}\valtdinter{\numerOrder}, \eg \ra_\rho \frac{(j\dt - s)^\numerOrder}{\numerOrder!} d s\rv\\
   \leq & \sum_{j=0}^\numerOrder\frac{|\coef{\numerOrder}_j|}\dt \int_0^{j\dt} \ll \semigroup_s\generator^{\numerOrder+1}\valtdinter{\numerOrder} \rl_\rho \ll \eg \rl_\rho \frac{(j\dt - s)^\numerOrder}{\numerOrder!} d s
   \leq \sum_{j=0}^\numerOrder\frac{|\coef{\numerOrder}_j|}{\dt \numerOrder!} \ll \generator^{\numerOrder+1}\valtdinter{\numerOrder}(\state) \rl_\rho\ll \eg(\state) \rl_\rho \int_0^{j\dt}  (j\dt - s)^\numerOrder d s\\
   \leq &\l(\frac{\sum_{j=0}^\numerOrder|\coef{\numerOrder}_j|j^{\numerOrder+1}}{ (\numerOrder+1)!} \l(\ll \valuebar^{(\numerOrder)}_j \rl_\eucl\ll \generator^{\numerOrder+1} \Phi \rl_\rho\r)\ll \eg(\state)\rl_\rho\r)\dt^\numerOrder\\
   \leq &\frac{\sum_{j=0}^\numerOrder|\coef{\numerOrder}_j|j^{\numerOrder+1}}{ (\numerOrder+1)!}  \ll \generator^{\numerOrder+1}\Phi(\state) \rl_\rho\l( \ll \eg(\state) \rl_\xi + \ll \ep(\state) \rl_\xi + \ll \valuestar(\state) \rl_\xi  \r) \ll \eg(\state)\rl_\rho\dt^\numerOrder 
\end{aligned}
\end{equation}
where Holder inequality is used in the first inequality, Eq~\eqref{eq:coercivity-in-time} is applied in the second inequality, 
Cauchy–Schwarz inequality is used in the third inequality. The last inequality holds because first by the orthogonality of the bases, one has 
$\ll [\valuebar^{(\numerOrder)}_j]_{j \geq 1} \rl_\eucl = \ll \valtdinter{\numerOrder}\rl_\eucl$, and 
then the definition of $\valtdinter{\numerOrder} (\state) = \valuestar(\state) - \eg(\state) - \ep(\state) $ is applied. 
Moreover, by~\Cref{eq:gen-upper}, we have
\begin{equation}\label{ineq7}
   \lv \la (\b - \generator)\ep(\state), \eg(\state)\ra_\rho \rv \leq \b\ll \ep \rl_\rho\ll \eg \rl_\rho + C_1\ll \eg \rl_\rho \ll \nb \ep \rl_\rho+ C_2\ll \nb \eg \rl_\rho\ll \nb \ep \rl_\rho,
\end{equation}
where we define the constants 
\begin{equation}\label{def of c1 c2}
C_1 \mydefn \big( \smoothness_\drift^{(0)} \sqrt{\usedim} + \smoothness_\covMat^{(1)} \usedim^2 + \lammax \smoothness_\stationary \usedim \big), \quad C_2 \mydefn \lammax.
\end{equation} 

Inserting \eqref{ineq5}, \eqref{ineq6}, \eqref{ineq7} into \eqref{ineq4} gives,
\[
\begin{aligned}
    \b\ll \eg \rl^2_\rho + \frac{\lammin
    }{2}\ll \nb\eg\rl^2_\rho \leq& C_3\ll \eg \rl^2_\rho\dt^\numerOrder + \l[C_3(\ll \ep \rl_\rho  + \ll \valuestar \rl_\rho)\stepsize^\numerOrder + \b\ll \ep \rl_\rho+ C_1 \ll \nb \ep \rl_\rho\r]\ll \eg \rl_\rho + C_2\ll \nb \eg \rl_\rho\ll \nb \ep \rl_\rho.
\end{aligned}
\]
where 
\begin{equation}\label{defofc2}
    C_3 = \frac{\sum_{j=0}^\numerOrder|\coef{\numerOrder}_j|j^{\numerOrder+1}}{ (\numerOrder+1)!}  \ll \generator^{\numerOrder+1}\Phi(\state) \rl_\rho.
\end{equation}
When $C_3\stepsize^\numerOrder \leq \frac{\discount}{2}$, one has
\[
\begin{aligned}
    \ll \eg \rl^2_\rho \leq& \frac{4}{\discount}\l[\frac{1}{\discount}\l(C_3(\ll \ep \rl_\rho  + \ll \valuestar \rl_\rho)\stepsize^\numerOrder + \b\ll \ep \rl_\rho+ C_1 \ll \nb \ep \rl_\rho\r)^2 + \frac{C_2^2}{2\lammin}\ll \nb \ep \rl_\rho^2 \r].
\end{aligned}
\]
Since $\valuestar$ satisfies $\la (\b - \generator) \valuestar, \valuestar\ra_\rho = \la r, \valuestar \ra_\rho$, by applying Eq~\eqref{eq:coercivity-in-approx-gen-proof} gives $\ll \valuestar \rl_\rho \leq \frac1\discount\ll r \rl_\rho$, one can bound $\eg$ by $\ll \ep \rl_{\mathbb{H}^1}$, 
\[
\ll \eg \rl_\rho \leq \frac{2C_3}{\discount^2}\ll r\rl_\rho\stepsize^\numerOrder + \frac{2}{\discount}(\b+C_3\stepsize^\numerOrder)\ll \ep \rl_\rho + \l(\frac{2C_1}{\discount}+ \frac{\sqrt{2}C_2}{\sqrt{\discount\lammin}}\r)\ll \nb \ep \rl_\rho .
\]
Since the above inequality holds for any $e_p = \valuestar - \valuestar_\LinSpace$, one has
\[
\ll \valtdinter{\numerOrder} - \valuestar \rl_\rho \leq \coefgkep \inf_{\valuestar_\LinSpace = \sum_{j=1}^\mbasis \valuestar_j\phi_j } \ll \valuestar_\LinSpace - \valuestar \rl_{\mathbb{H}^1} + c_4\dt^\numerOrder
\]
with 
\begin{equation}\label{defofcoefgk}
    \coefgkep = \max\l\{\frac{2}{\discount}(\b+C_3\stepsize^\numerOrder), \l(\frac{2C_1}{\discount}+ \frac{\sqrt{2}C_2}{\sqrt{\discount\lammin}}\r)\r\}, \quad c_4 = \frac{2C_3}{\discount^2}\ll r\rl_\rho.
\end{equation}
where the constants $C_1, C_2$ are defined in \eqref{def of c1 c2} by the coefficients in Eq~\eqref{eq:gen-upper}, and the constant $C_3$ is defined in Eq~\eqref{defofc2}.
Now let us look at the assumption $C_3\stepsize^\numerOrder \leq \frac{\discount}{2}$, which requires
\[
\eta^\numerOrder <  \frac{\b}{2C_{a,\numerOrder}\ll \generator^{\numerOrder+1}\Phi \rl}, \quad \text{where }\quad C_{a,\numerOrder} = \frac{\sum_{j=0}^\numerOrder|\coef{\numerOrder}_j|j^{\numerOrder+1}}{ (\numerOrder+1)!} .
\]

\section{Discussion and future work}\label{sec:discussion}
In this paper, we proposed and analyzed a new class of discrete-time Bellman equations for policy evaluation problems in continuous-time diffusion processes, which lead to practical RL algorithms that are flexible with existing function approximation methods. Compared with na\'{i}ve discrete-time RL methods, the new methods enjoy improved numerical accuracy guarantees by employing higher-order methods. Moreover, we revealed new phenomena for projected fixed-point RL methods in continuous-time dynamics -- with the help of elliptic structures, the approximation pre-factor is uniformly bounded, independent of the effective horizon of the discretized problem.

The new algorithms and associated theoretical analyses open up several interesting directions of future research. Let us discuss a few to conclude.

\begin{itemize}
  \item First, a natural question is regarding the statistical guarantees for the sample-based algorithms~\eqref{eq:finite-sample-estimator-bellman} and~\eqref{eq:finite-sample-estimator-diffusion}, as well as their optimality properties. In our companion paper, we will partially address this question by establishing sharp non-asymptotic guarantees in the linear case. Yet, the general question of optimal statistical methods and adaptive trade-offs remains open for a broad class of machine learning models used to fit the value function and/or the underlying dynamics.
  \item The theoretical results in this paper require some technical assumptions that may be stringent in some practical applications. On the other hand, the simulation studies in~\Cref{sec:simulation} can go beyond these assumptions. So it is interesting to study whether our algorithms achieve similar guarantees under relaxed assumptions; and if not, new methods that work. In particular, we conjecture that the global bounded/Lipschitz assumptions~\ref{assume:smooth-high-order} can be replaced by local versions, and that the elliptic assumption~\ref{assume:uniform-elliptic} can be replaced by hypo-ellipticity. It is also interesting to extend the results to diffusion processes with jumps.
  \item Finally, note that this work focused on the policy evaluation problem, i.e., a linear second-order elliptic equation. It is important to study the policy optimization problem, which involves solving a non-linear Hamilton--Jacobi--Bellman equation. In particular, accurate numerical methods that are flexible with existing RL approaches with function approximations, in a flavor similar to the present paper, are desirable.
\end{itemize}

\section*{Acknowledgement}
This work was partially supported by NSERC grant RGPIN-2024-05092 to WM, and NSF grant DMS-2411396 to YZ.

\bibliographystyle{alpha}
\bibliography{references}

\appendix
\section{Additional results for the high-order generator}

We first bound the distance between the operators $\geninter{\numerOrder}$ and $\generator$.
\begin{lemma}\label{thm:generator-high-order}
  If Assumption~\ref{assume:smooth-high-order} holds true for some integer $\numerOrder > 0$, then for the true value function defined in \eqref{eq:defn-value-func}, we have 
    \begin{align}\label{eq:generator-high-order}
    \abss{ \generator \valuestar(\state) - \geninter{\numerOrder} \valuestar(\state)}
    \leq  \constgen_\numerOrder\stepsize^\numerOrder 
\end{align}
  with a constant $\constgen_\numerOrder$ defined in \eqref{defofconstgen} that depends on $\big\{\smoothness_i^{\drift}\big\}_{i = 0}^{2 \numerOrder - 2}$, $\big\{\smoothness_i^{\covMat}\big\}_{i = 0}^{2 \numerOrder - 2}$, $\big\{\smoothness_i^{\reward}\big\}_{i = 0}^{2 \numerOrder}$, discount factor $\discount$ and problem dimension $\usedim$.
\end{lemma}
See Section \eqref{proof:generator-high-order} for its proof.

\subsection{Proof of Lemma~\ref{thm:generator-high-order}}\label{proof:generator-high-order}
Define $h_\state(s) := \semigroup_s \valuestar(\state)$.  First by Taylor expansion, one has,
\[
\begin{aligned}
&\geninter{\numerOrder} \valuestar(x) = \frac{1}{\stepsize}\sum_{j=0}^\numerOrder \coef{\numerOrder}_j h_\state(j\stepsize)
= \frac1\stepsize\sum_{j=0}^\numerOrder \coef{\numerOrder}_j \l( \sum_{k=0}^\numerOrder  \pt_s^kh_\state(0)\frac{(j\stepsize )^k}{k!} + \int_0^{j\stepsize}\pt^{\numerOrder+1}_s h_\state(\xi)\frac{(j\stepsize - \xi)^\numerOrder }{\numerOrder!} d\xi \r)\\
= & \frac1\stepsize\sum_{k=0}^\numerOrder  \pt_s^kh_\state(0) \frac{\stepsize^k}{k!} \sum_{j=0}^\numerOrder \coef{\numerOrder}_jj^k +  \frac1\stepsize\sum_{j=0}^\numerOrder \coef{\numerOrder}_j \int_0^{j\stepsize}\pt^{\numerOrder+1}_s h_\state(\xi)\frac{(j\stepsize - \xi)^\numerOrder }{\numerOrder!} d\xi  \\
\end{aligned}
\]
Note that by the definition of $\coef{\numerOrder}$ in \eqref{def of A b}, it also satisfies
\[
\sum_{j=0}^\numerOrder \coef{\numerOrder}_j j^k = \l \{
\begin{aligned}
&0, \quad k \neq 1, \quad 0\leq k \leq i,\\
&1, \quad k = 1,
\end{aligned}\r.
\]
which yields
\begin{equation}\label{eq:gen-genapprox}
    \begin{aligned}
&\geninter{\numerOrder} \valuestar(x) = \generator \valuestar(x) +  \frac1\stepsize\sum_{j=0}^\numerOrder \coef{\numerOrder}_j \int_0^{j\stepsize}\pt^{\numerOrder+1}_s h_\state(\xi)\frac{(j\stepsize - \xi)^\numerOrder }{\numerOrder!} d\xi \\
=& \generator \valuestar(x) +  \frac1\stepsize\sum_{j=0}^\numerOrder \coef{\numerOrder}_j \int_0^{j\stepsize}\semigroup_\xi\generator^{\numerOrder+1}\valuestar(x)\frac{(j\stepsize - \xi)^\numerOrder }{\numerOrder!} d\xi . 
\end{aligned}
\end{equation}
Since
\[
\ll \semigroup_\xi\generator^{\numerOrder+1}\valuestar(x) \rl_\infty \leq \ll \generator^{\numerOrder+1}\valuestar(x) \rl_\infty ,
\]
one has,
\begin{equation}\label{eq:ineq_1}
    \begin{aligned}
    &\lv \generator \valuestar(\state) - \geninter{\numerOrder} \valuestar(\state)\rv 
    \leq  \ll \generator^{\numerOrder+1}\valuestar(x)\rl_\infty \sum_{j=0}^\numerOrder \frac1\stepsize\lv \coef{\numerOrder}_j\rv\int_0^{j\stepsize} \frac{(j\stepsize - \xi)^\numerOrder }{\numerOrder!} d \xi  \\
    =&  \ll \generator^{\numerOrder+1}\valuestar(x)\rl_\infty \frac{\sum_{j=0}^\numerOrder |\coef{\numerOrder}_j|j^{\numerOrder+1}}{(\numerOrder+1)!}  \stepsize^\numerOrder 
\end{aligned}
\end{equation}
The true value function satisfies $\generator \valuestar = \discount \valuestar(\state)- \reward(\state)$, which implies
\begin{equation}\label{ineq8}
    \begin{aligned}
 \ll \generator^{\numerOrder+1}\valuestar(\state)\rl_\infty \leq&  \discount \ll \generator^{\numerOrder}\valuestar(\state) \rl_\infty + \ll \generator^{\numerOrder}\reward(\state) \rl_\infty 
 \leq  \discount^{\numerOrder+1} \ll  \valuestar(\state) \rl_\infty  +\sum_{j=0}^\numerOrder \discount^{\numerOrder-j} \ll  \generator^{j}\reward(\state) \rl_\infty
 \end{aligned}
\end{equation}
Note that by the definition of $\valuestar(\state)$ in \eqref{eq:defn-value-func}, one has
\[
\lv \valuestar(\state) \rv \leq \ll \reward(\state) \rl_\infty \int_0^\infty e^{-\discount t}dt  = \frac1\discount\ll \reward(\state) \rl_\infty  
\]
Inserting the above inequality and Lemma \ref{lemma:semigroup-high-order-bound} into \eqref{ineq8} gives,
\[
 \ll \generator^{\numerOrder+1}\valuestar(\state)\rl_\infty \leq \discount^\numerOrder\ll \reward \rl_\infty + \sum_{j=0}^\numerOrder\discount^{\numerOrder-j}\ll \pt_s^j\semigroup_sr(\state)|_{s=0}\rl_\infty \leq \discount^\numerOrder\ll \reward \rl_\infty + \sum_{j=0}^\numerOrder\discount^{\numerOrder-j}\constScary_j,
\]
where the constant $\constScary_j$ is defined in Lemma \ref{lemma:semigroup-high-order-bound}. Inserting the above inequality into \eqref{eq:ineq_1} gives \eqref{eq:generator-high-order}
with a constant $\constgen_\numerOrder$ defined as  
\begin{equation}\label{defofconstgen}
    \constgen_\numerOrder = \l(\discount^\numerOrder\ll \reward \rl_\infty + \sum_{j=0}^\numerOrder\discount^{\numerOrder-j}\constScary_j\r) \frac{\sum_{j=0}^\numerOrder |\coef{\numerOrder}_j|j^{\numerOrder+1}}{(\numerOrder+1)!} 
\end{equation}
which completes the proof. 

\section{Proof of~\Cref{cor:discounted-occu-msr}} \label{sec:app-proof-cor-discounted-occu-msr}
The proof follows the same structure as~\Cref{thm:improved-approx-factor-elliptic}, while the key step is to establish an analogue of the coercive condition in Lemma~\ref{lemma:discrete-time-psd-operator}. We have the following result
\begin{lemma}\label{lemma:psd-operator-under-occumsr}
Under the setup of~\Cref{cor:discounted-occu-msr}, we have the lower bound
  \begin{align*}
    \inprod{\ValFunc}{t^{-1} \big(\IdMat - e^{- \discount t} \transition_t \big) \ValFunc}_\occumsr \geq \frac{\discount}{2t} \int_0^t e^{- 2 \discount s} \vecnorm{\semigroup_s \ValFunc}{\occumsr}^2 ds + \frac{\lammin}{2t} \int_0^t e^{- 2 \discount s} \vecnorm{\nabla \semigroup_s \ValFunc}{\occumsr}^2 ds + \frac{1}{2t} \vecnorm{\ValFunc - e^{- \discount t} \semigroup_t \ValFunc}{\occumsr}^2.
  \end{align*}
\end{lemma}
\noindent See~\Cref{subsec:proof-lemma-psd-operator-occumsr} for the proof of this lemma.

Taking this lemma as given, we now proceed with the proof of~\Cref{cor:discounted-occu-msr}. Note that for any $\ltwospace$-integrable function $f$ and $t > 0$, the discounted occupancy measure satisfies
\begin{align*}
  \vecnorm{\semigroup_t f}{\occumsr}^2 &= \discount \int_0^{+ \infty} e^{- \discount s} \Exs \Big[ \Big\{ \Exs \big[ f (\MyState_{t + s}) \mid \MyState_s \big] \Big\}^2 \Big] ds\\
  &\leq \discount \int_0^{+ \infty} e^{- \discount s} \Exs \Big[f (\MyState_{t + s})^2 \Big] ds\\
  &= e^{\discount t} \discount \int_t^{+ \infty} e^{- \discount s} \Exs \big[ f(\MyState_s)^2 \big] ds\\
  &\leq e^{\discount t} \vecnorm{f}{\occumsr}^2.
\end{align*}
By bounding the expansion factor of the semigroup $(\semigroup_t : t \geq 0)$ under the $\vecnorm{\cdot}{\occumsr}$-norm, we can apply Lemmas~\ref{lemma:discrete-time-operator-norm-bound} and~\ref{lemma:reverse-smoothing-estiamte} with $\beta_0 = \discount / 2$. Note that the rest parts of the proof of Theorem~\ref{thm:improved-approx-factor-elliptic} does not depend on the fact that $\stationary$ is the stationary distribution. We can apply the same arguments to the occupancy measure $\occumsr$, and conclude the final bound of interests.

\subsection{Proof of Lemma~\ref{lemma:psd-operator-under-occumsr}}\label{subsec:proof-lemma-psd-operator-occumsr}
Similar to the proof of Lemma~\ref{lemma:discrete-time-psd-operator}, we first derive a lower bound on the limiting differential operator, and then extend it to the case of integral operators. Note that
\begin{align*}
  \inprod{\ValFunc}{\ValFunc - e^{- \discount t} \semigroup_t \ValFunc}_{\occumsr} = \int_0^t  \inprod{\ValFunc}{\discount \semigroup_s \ValFunc - \generator \semigroup_s \ValFunc}_{\occumsr} ds.
\end{align*}
We claim the following inequality
\begin{align}
   \inprod{\ValFunc}{\discount \ValFunc - \generator \ValFunc}_{\occumsr} \geq \frac{\discount}{2}\vecnorm{\ValFunc}{\occumsr}^2 + \frac{\lammin}{2} \vecnorm{\nabla \ValFunc}{\occumsr}^2. \label{eq:psd-bound-for-gen-under-occumsr}
 \end{align}
We prove this inequality at the end of this section. Taking this inequality as given, we have
\begin{align*}
 &\inprod{\ValFunc}{\discount \semigroup_s \ValFunc - \generator \semigroup_s \ValFunc}_{\occumsr} \\
   &= \inprod{e^{- \discount s} \semigroup_s \ValFunc}{e^{- \discount s} \big( \discount \semigroup_s \ValFunc - \generator \semigroup_s \ValFunc \big)}_\occumsr + \inprod{\ValFunc - e^{- \discount s} \semigroup_s \ValFunc}{e^{- \discount s} \big( \discount \semigroup_s \ValFunc - \generator \semigroup_s \ValFunc \big)}_\occumsr \\
   &\geq e^{- 2 \discount s} \frac{\discount}{2}  \vecnorm{\semigroup_s \ValFunc}{\occumsr}^2 + e^{- 2 \discount s} \frac{\lammin}{2} \vecnorm{\nabla \semigroup_s \ValFunc}{\occumsr}^2 + \int_0^s e^{- \discount \ell}\inprod{\discount \semigroup_\ell \ValFunc - \generator \semigroup_\ell \ValFunc}{e^{- \discount s}_\occumsr \big( \discount \semigroup_s \ValFunc - \generator \semigroup_s \ValFunc \big)}_\occumsr d \ell,
\end{align*}
By symmetry, a time integration of the last term satisfies
\begin{align*}
  \int_0^t \int_0^s e^{- \discount \ell}\inprod{\discount \semigroup_\ell \ValFunc - \generator \semigroup_\ell \ValFunc}{e^{- \discount s}_\occumsr \big( \discount \semigroup_s \ValFunc - \generator \semigroup_s \ValFunc \big)}_\occumsr d \ell ds = \frac{1}{2} \vecnorm{\ValFunc - e^{- \discount t} \semigroup_t \ValFunc}{\occumsr}^2.
\end{align*}
Putting them together completes the proof of Lemma~\ref{lemma:psd-operator-under-occumsr}.

\paragraph{Proof of Eq~\eqref{eq:psd-bound-for-gen-under-occumsr}:} Note that in deriving Eq~\eqref{eq:final-bound-in-lemma-psd}, we do not need any information about the underlying weighting measure. Consequently, following the same line of arguments as in the proof of Eq~\eqref{eq:positive-definite-operator}, we have
\begin{align*}
  - \inprod{\ValFunc}{\generator \ValFunc}_\occumsr \geq \frac{\lammin}{2}\vecnorm{\nabla \ValFunc}{\occumsr}^2 - \frac{1}{2} \int \ValFunc^2 (x) \Big\{ - \nabla \cdot \big( \occumsr (x) \drift (x) \big) + \frac{1}{2} \nabla^2 \cdot \big( \occumsr \covMat \big)(x)  \Big\} dx.
\end{align*}
In order to study the latter term, we apply integration by parts formula to the definition of $\occumsr$, and obtain that
\begin{align*}
  \occumsr &= \int_0^{+ \infty} \big( \discount e^{- \discount t} \big) \cdot \semigroup_t^* \pi_0 dt\\
  &= - e^{- \discount t} \semigroup_t^* \pi_0 \big|_{0}^{+ \infty} + \int_0^{+ \infty}  e^{- \discount t}\cdot \big( \partial_t  \semigroup_t^* \pi_0 \big) dt\\
  &= \pi_0 + \int_0^{+ \infty} e^{- \discount t} \generator^* \semigroup_t^* \pi_0 dt\\
  &= \pi_0 + \frac{1}{\discount} \generator^* \occumsr.
\end{align*}
Therefore, the discounted occupancy measure satisfies the PDE $\occumsr - \tfrac{1}{\discount} \generator^* \occumsr = \pi_0$. Substituting back to the lower bound above, we obtain the fact
\begin{align*}
  \discount \vecnorm{\ValFunc}{\occumsr}^2 - \inprod{\ValFunc}{\generator \ValFunc}_\occumsr \geq \frac{\lammin}{2}\vecnorm{\nabla \ValFunc}{\occumsr}^2  + \frac{\discount}{2} \vecnorm{\ValFunc}{\occumsr}^2 + \frac{\discount}{2} \vecnorm{\ValFunc}{\pi_0}^2,
\end{align*}
which complete the proof of Eq~\eqref{eq:psd-bound-for-gen-under-occumsr}.

\section{Proof for some auxiliary examples}
We collect the proofs of some auxiliary results about the examples throughout the paper.

\subsection{Proof of Proposition~\ref{prop:fourier-basis-example}}\label{subsec:app-proof-prop-fourier-basis-example}
Note that the Fourier basis are given in the form of
\begin{align*}
  \phi_j (\state) = \exp \big( - 2 \pi i \inprod{\state}{\omega^{(j)}} \big), \quad \mbox{for any $\state \in \torus^\usedim$},
\end{align*}
where $\omega^{(j)} \in \mathbb{Z}^\usedim$ is a signed multi-index. For standard Fourier basis of torus, we assume that $\big( \vecnorm{\omega^{(j)}}{\infty} \big)_{j = 1,2,\cdots}$ are 
sorted in a non-decreasing order. Let $\ValFunc = \sum_{j = 1}^\mbasis \alpha_j \phi_j$ be its basis representation. We have
\begin{align*}
  \statnorm{\partial_k f}^2 = \statnorm{\sum_{j = 1}^\mbasis \alpha_j \partial_k \phi_j}^2 =  \statnorm{\sum_{j = 1}^\mbasis \alpha_j \inprod{e_k}{\omega^{(j)}}\phi_j}^2 = \sum_{j = 1}^\mbasis \alpha_j^2 \inprod{e_k}{\omega^{(j)}}^2,
\end{align*}
and similarly
\begin{align*}
  \statnorm{\partial_k \partial_\ell f}^2 = \statnorm{\sum_{j = 1}^\mbasis \alpha_j \partial_k \partial_\ell \phi_j}^2 =  \statnorm{\sum_{j = 1}^\mbasis \alpha_j \inprod{e_k}{\omega^{(j)}} \inprod{e_\ell}{\omega^{(j)}}\phi_j}^2 = \sum_{j = 1}^\mbasis \alpha_j^2 \inprod{e_k}{\omega^{(j)}}^2 \inprod{e_\ell}{\omega^{(j)}}^2.
\end{align*}
Consequently, for any $f \in \LinSpace$, we have
\begin{subequations}\label{eqs:fourier-example-calc}
\begin{align}
  \statnorm{\nabla f}^2 &= \sum_{k = 1}^\usedim \statnorm{\partial_k f}^2 \leq \usedim \max_{k \in [\usedim],~ j \in [\mbasis]} \inprod{e_k}{\omega^{(j)}}^2 \cdot \statnorm{f}^2 \leq \usedim \cdot \vecnorm{\omega^{(\mbasis)}}{\infty}^2 \cdot  \statnorm{f}^2\\
   \statnorm{\nabla^2 f}^2 &\leq \sum_{1 \leq k, \ell \leq \usedim } \statnorm{\partial_k \partial_\ell f}^2 \leq \usedim^2 \max_{k, \ell \in [\usedim],~ j \in [\mbasis]} \inprod{e_k}{\omega^{(j)}}^2 \cdot \statnorm{f}^2 \leq \usedim^2 \cdot \vecnorm{\omega^{(\mbasis)}}{\infty}^4 \cdot  \statnorm{f}^2.
\end{align}
\end{subequations}
It suffices to bound the $\ell^\infty$-norm of the multi-indices. We do so by counting the number of multi-indices with the same $\ell^\infty$-norm. For each $k \geq 0$, we note that
\begin{align*}
  \big| \{j: \vecnorm{\omega^{(j)}}{\infty} \leq k \} \big| = (2 k + 1)^\usedim.
\end{align*}
Since the sequence $\big( \vecnorm{\omega^{(j)}}{\infty} \big)_{j \geq 1}$ is ordered non-decreasingly, we have
\begin{align*}
  \vecnorm{\omega^{(\mbasis)}}{\infty} \leq  1 +  \lceil \frac{1}{2} \mbasis^{1 / \usedim} - 1 \rceil \leq \mbasis^{1 / \usedim}.
\end{align*}
Substituting back to~\Cref{eqs:fourier-example-calc} completes the proof of~\Cref{prop:fourier-basis-example}.

\subsection{Proof of \Cref{prop:occumsr-example}}\label{subsec:proof-occumsr-example}
By triangle inequality, we have
\begin{align*}
  \vecnorm{\nabla \log \occumsr (\state)}{2} = \vecnorm{\nabla \log \int_0^{+ \infty} e^{- \discount t} \semigroup_t^* \pi_0 (\state) dt}{2} \leq \frac{\int_0^{+ \infty} e^{- \discount t} \vecnorm{\nabla \semigroup_t^* \pi_0 (\state)}{2} dt }{\int_0^{+ \infty} e^{- \discount t} \semigroup_t^* \pi_0 (\state) dt } \leq \sup_{t \geq 0} \frac{\vecnorm{\nabla \semigroup_t^* \pi_0 (\state)}{2}}{\semigroup_t^* \pi_0 (\state)}.
\end{align*}
For the heat semigroup, we have $\semigroup_t^* \pi_0 (\state) = (\pi_0 * \varphi_t) (\state)$, where $\varphi_t (x) = \frac{1}{\sqrt{2 \pi t}} \exp \big( - \frac{x^2}{2t} \big)$ is the probability density function of the normal distribution $\mathcal{N} (0, t I_\usedim)$. Therefore, for any $t \geq 0$, we have
\begin{align*}
  \frac{\vecnorm{\nabla \semigroup_t^* \pi_0 (\state)}{2}}{\semigroup_t^* \pi_0 (\state)} = \frac{\vecnorm{\nabla (\pi_0 * \varphi_t) (\state)}{2}}{(\pi_0 * \varphi_t) (\state)} \leq \frac{\int_{\real^\usedim} \vecnorm{\nabla_\state \pi_0 (\state - y)}{2} \phi_t (y)  dy}{\int_{\real^\usedim} \pi_0 (\state - y) \phi_t (y)  dy} \leq \sup_{\state \in \real^\usedim} \frac{\vecnorm{\nabla \pi_0 (\state)}{2}}{\pi_0 (\state)} \leq \smoothness.
\end{align*}
Substituting back completes the proof of Proposition~\ref{prop:occumsr-example}.

\section{Additional simulation results}\label{app:sec-additional-simulation}

For the Ornstein--Uhlenbeck processes, we perform additional simulation studies with a quadratic reward function $\reward (\state) = \state^2$. In such a case, the process is a linear quadratic system, which is canonical in stochastic control literature. We choose three basis functions, with $\psi_1 (\state) = 1$, $\psi_2 (\state) = \state$, and $\psi_3 (\state) = \state^2$. In Figures~\ref{fig:simulation-stochastic-time-lq} and~\ref{fig:simulation-stochastic-stepsize-lq}, we present the simulation results under quadratic rewards, by replicating exactly the same simulation setups in Figures~\ref{fig:simulation-stochastic-time} and~\ref{fig:simulation-stochastic-stepsize}. From Figures~\ref{fig:simulation-stochastic-time-lq} and~\ref{fig:simulation-stochastic-stepsize-lq}, we observe qualitative phenomena similar to the periodic case in Section~\ref{sec:simulation}, with the second-order methods significantly out-performing first-order counterparts. Note that since the true value function can be represented exactly by the basis functions, the approximation error does not exist, and the errors decay faster compared to the simulation results in Section~\ref{sec:simulation}. Yet, we still observe stabilization of numerical error in the fixed stepsize setting, and the elbow effect in the fixed horizon setting.

\begin{figure}[ht!]
\centering
\begin{tabular}{ccc}
  \widgraph{0.37\textwidth}{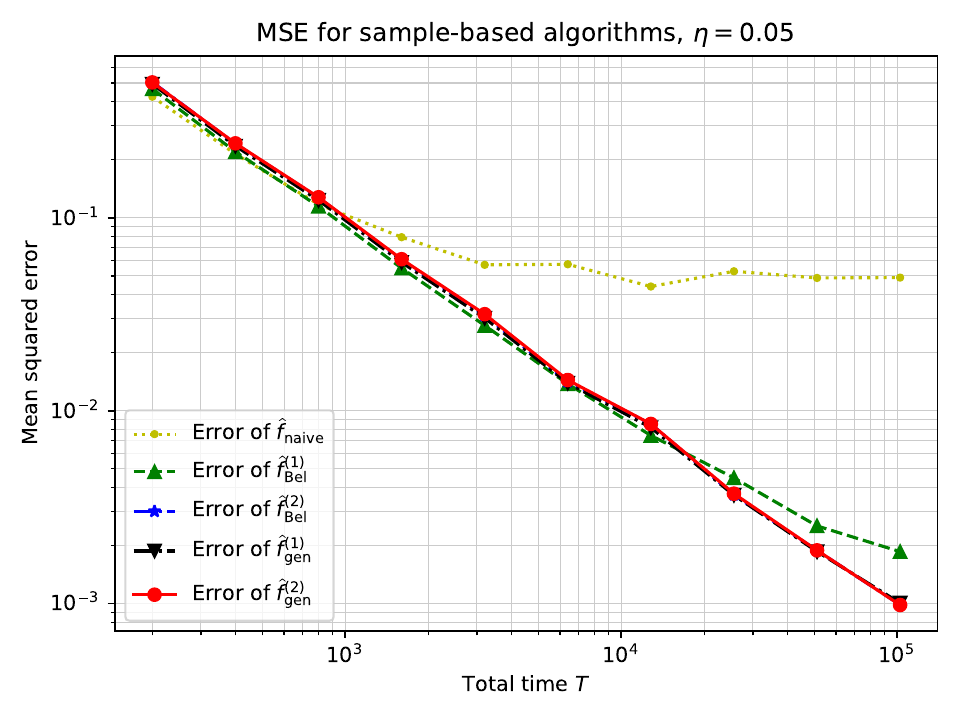} &&
  \widgraph{0.37\textwidth}{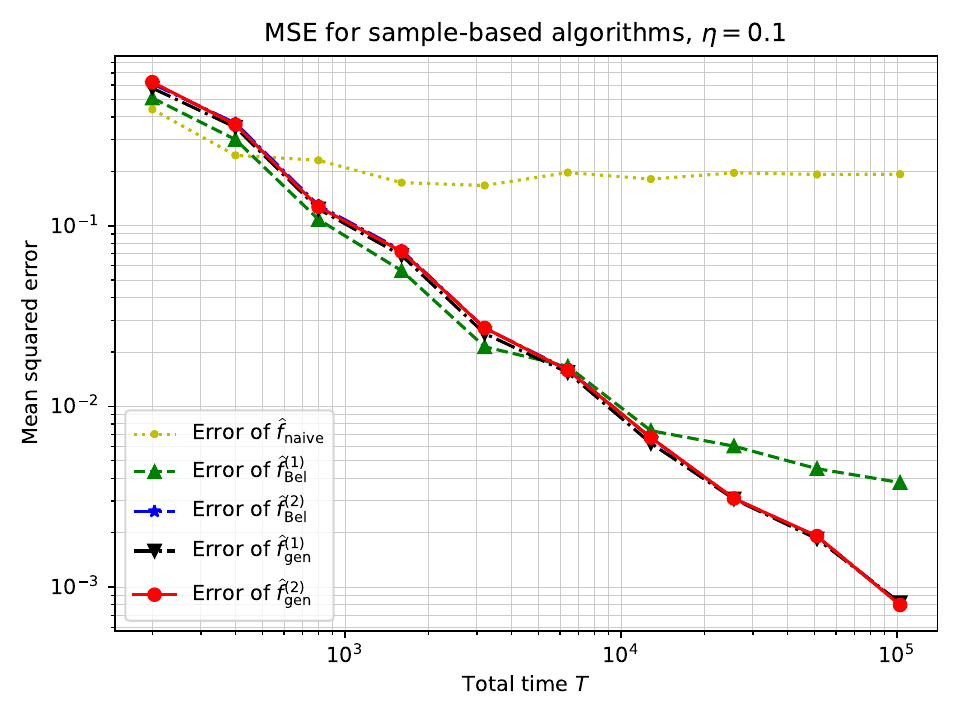} \\ (a) $\stepsize = 0.05$ && (b)  $\stepsize = 0.1$ \\
  \widgraph{0.37\textwidth}{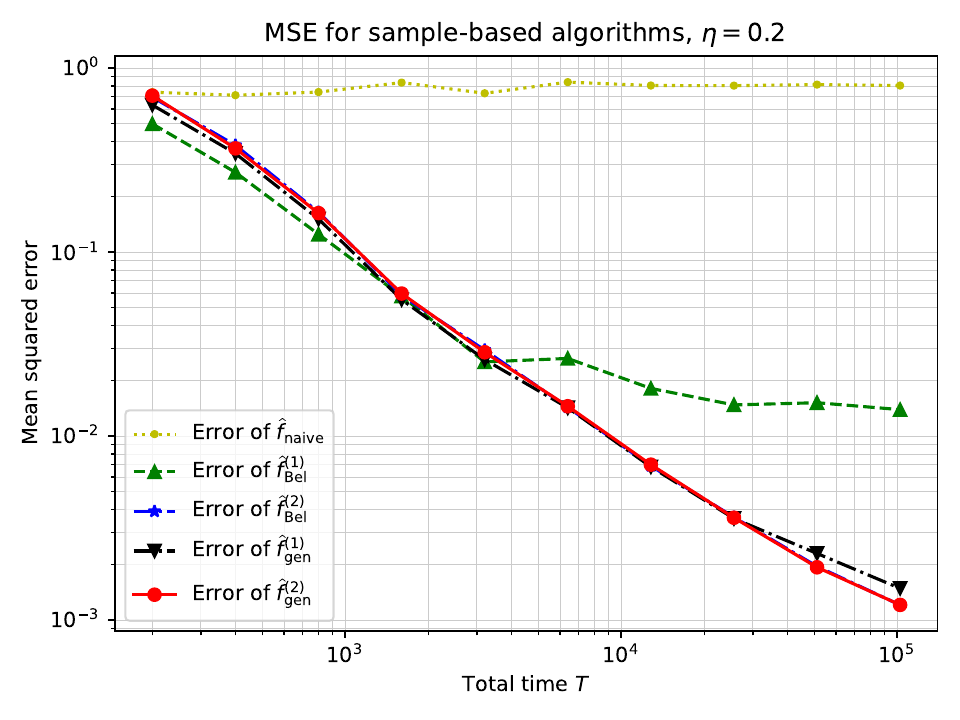} &&
  \widgraph{0.37\textwidth}{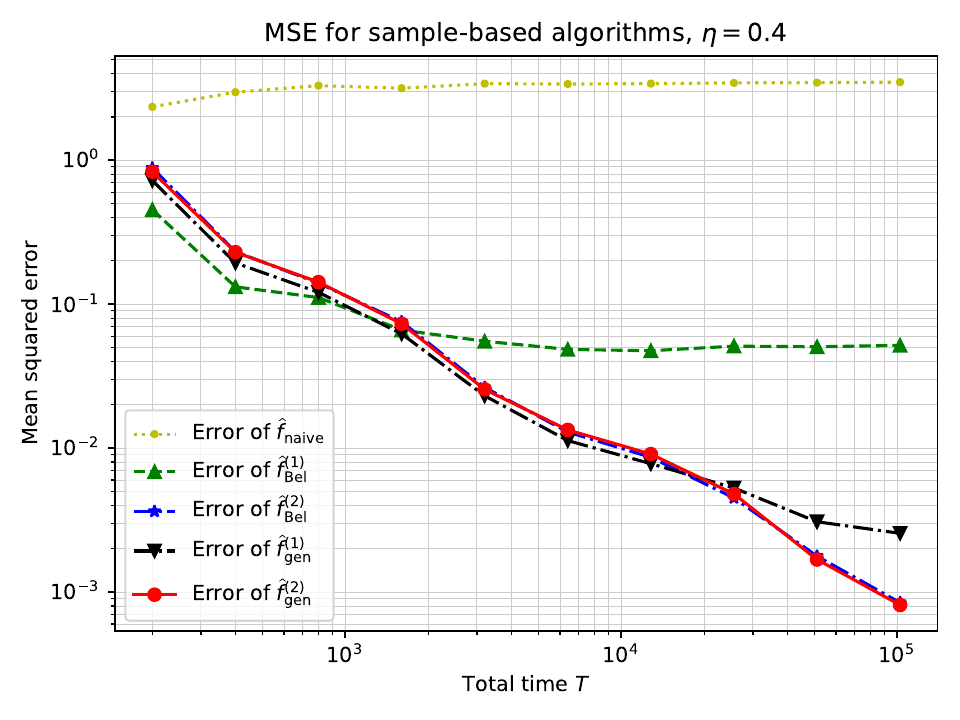} \\
(c)  $\stepsize = 0.2$ && (d)  $\stepsize = 0.4$
  \\
\end{tabular}
\caption{Plots of the mean-squared error $\Exs \big[ \statnorm{\valuehat - \valuestar}^2 \big]$ versus trajectory length
  $T$ in the linear-quadratic settings. The simulation setups are the same as Fig~\ref{fig:simulation-stochastic-time}.}
\label{fig:simulation-stochastic-time-lq}
\end{figure}

\begin{figure}[ht!]
\centering
\begin{tabular}{ccc}
  \widgraph{0.37\textwidth}{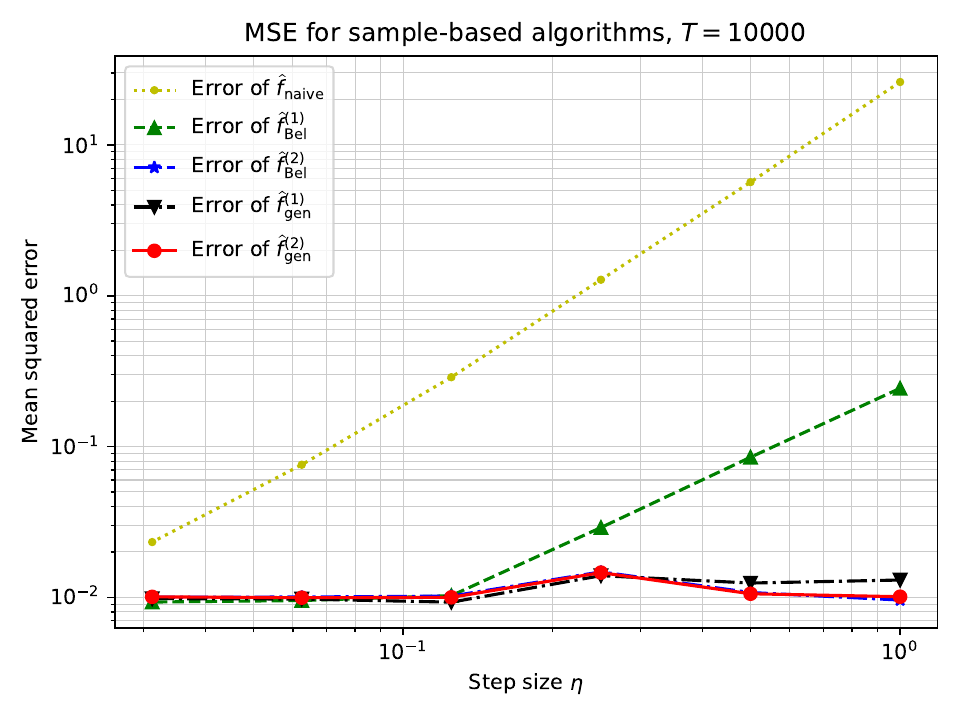} &&
  \widgraph{0.37\textwidth}{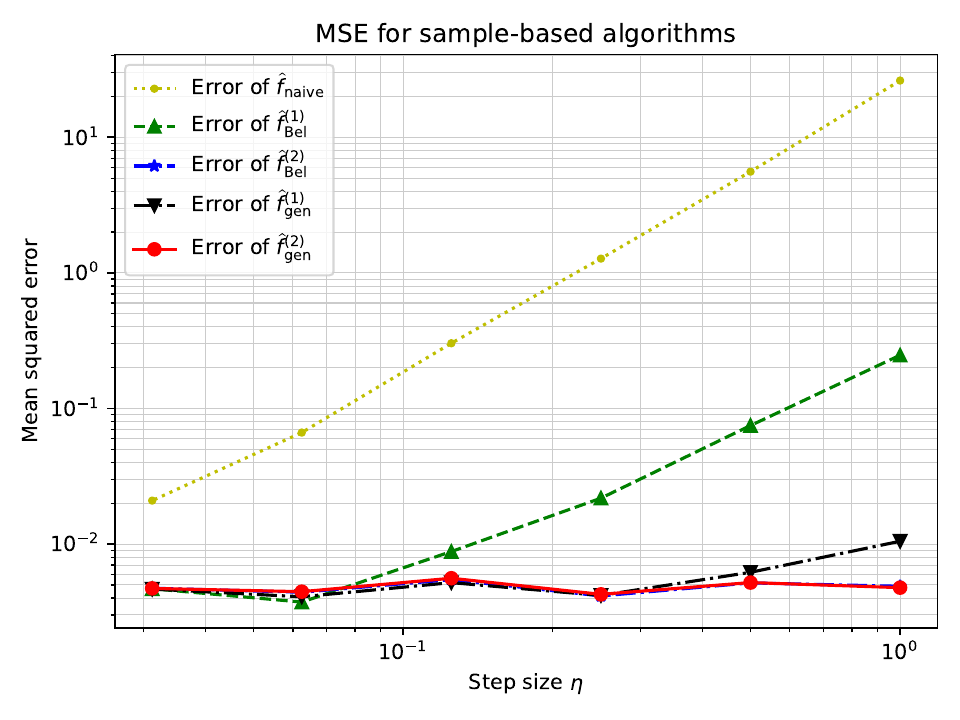} \\ (a) $T = 10000$ && (b) $T = 20000$ \\
  \widgraph{0.37\textwidth}{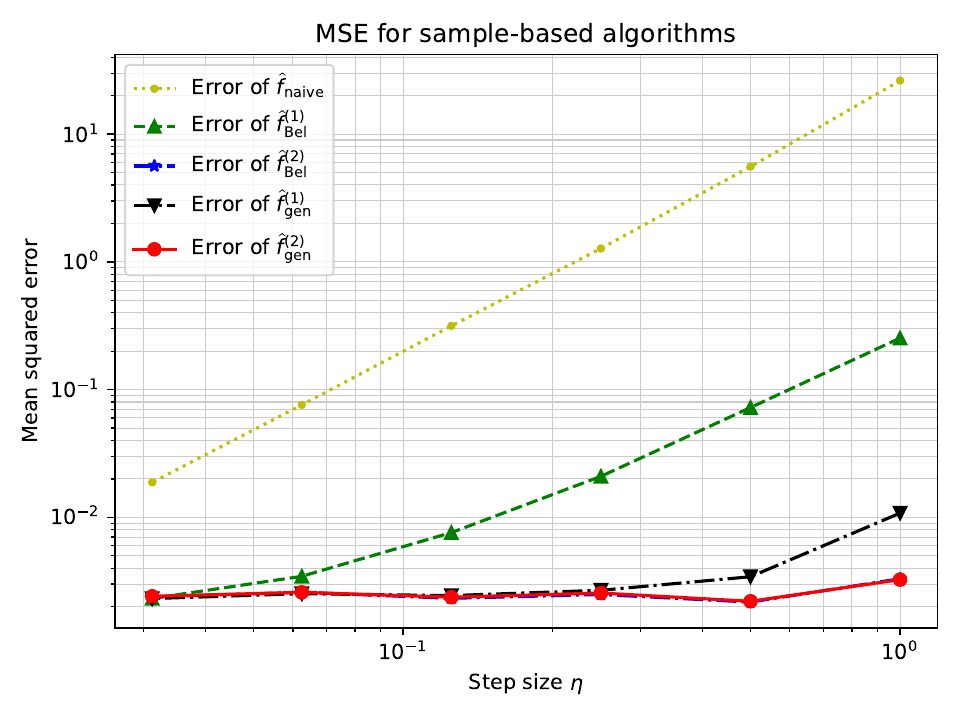} &&
  \widgraph{0.37\textwidth}{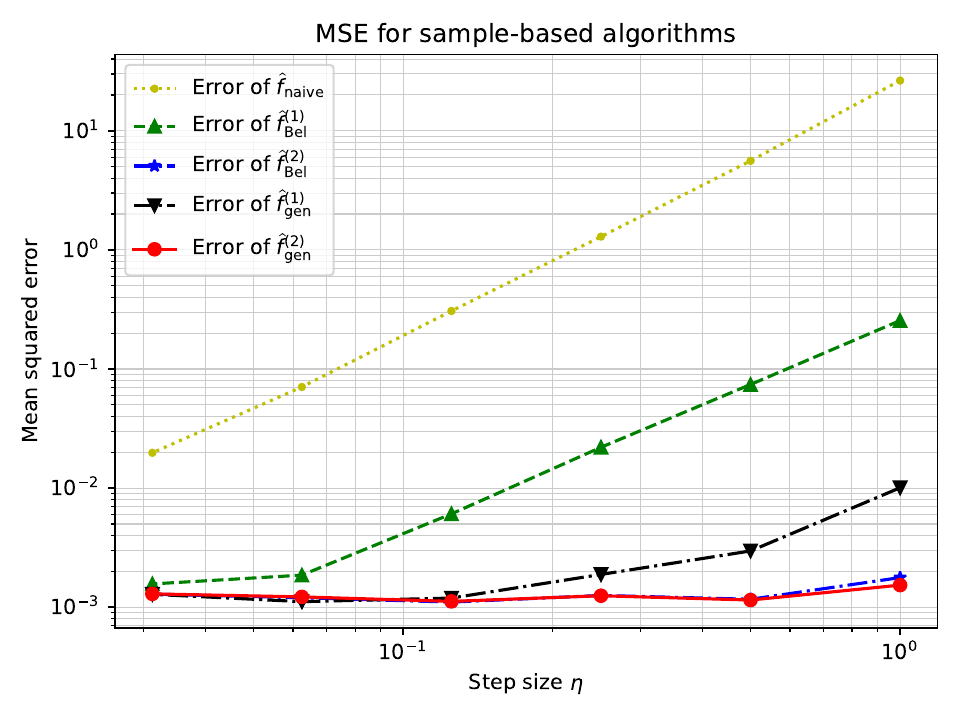} \\
(c) $T = 40000$ && (d) $T = 80000$
  \\
\end{tabular}
\caption{Plots of the mean-squared error $\Exs \big[ \statnorm{\valuehat - \valuestar}^2 \big]$ versus stepsize $\stepsize$ in the linear-quadratic settings. The simulation setups are the same as Fig~\ref{fig:simulation-stochastic-stepsize}.}
\label{fig:simulation-stochastic-stepsize-lq}
\end{figure}

\end{document}